%% file: neurips_2026.tex
\documentclass{article}

\PassOptionsToPackage{numbers, compress}{natbib}
\usepackage[main, final]{neurips_2026} 

\usepackage[utf8]{inputenc} 
\usepackage[T1]{fontenc}    
\usepackage{hyperref}       
\usepackage{url}            
\usepackage{booktabs}       
\usepackage{amsfonts}       
\usepackage{nicefrac}       
\usepackage{microtype}      
\usepackage{multirow}
\usepackage[table]{xcolor}
\usepackage{colortbl}
\usepackage{pifont}
\usepackage{amsmath, amssymb, amsthm, mathtools, bm}
\definecolor{LightGray}{gray}{0.93}
\definecolor{LightBlue}{RGB}{228,238,250}
\usepackage{graphicx}
\newcommand{\best}[1]{\textbf{#1}}
\newcommand{\second}[1]{\underline{#1}}
\usepackage{enumitem}
\usepackage{wrapfig}
\usepackage{subcaption}

\newtheorem{proposition}{Proposition}
\newtheorem{lemma}{Lemma}

\usepackage{titletoc}
\newcommand{\Cd}{\mathcal{C}_d}
\newcommand{\LE}{\mathrm{LE}}
\newcommand{\val}{\mathrm{val}}

\newcommand{\diagop}{\mathrm{diag}}

\newcommand{\ISR}{\mathrm{ISR}}
\newcommand{\MAE}{\mathrm{MAE}}
\newcommand{\Corr}{\mathrm{Corr}}

\newcommand{\Prob}{\mathbb{P}}
\newcommand{\E}{\mathbb{E}}
\newcommand{\R}{\mathbb{R}}

\newcommand{\cmark}{\textcolor{green!60!black}{\ding{51}}}
\newcommand{\xmark}{\textcolor{red!70!black}{\ding{55}}}

\title{MAGNET: Manifold-Aware Graph Diffusion Network for Connectome Generation}

\author{%
\makebox[\textwidth][c]{%
Protyay Dey$^{1}$\thanks{Equal contribution.} \ 
\thanks{Corresponding authors:
\texttt{protyayd@buffalo.edu} and
\texttt{vishnulo@buffalo.edu}.}%
\quad
Ayush Roy$^{1}$\footnotemark[1]
\quad
Hyuna Cho$^{2}$
\quad
Won Hwa Kim$^{3}$
\quad
Vishnu Suresh Lokhande$^{1}$
}\\[6pt]
\makebox[\textwidth][c]{%
\small
$^{1}$University at Buffalo, SUNY, USA
\quad
$^{2}$National Institutes of Health, USA
\quad
$^{3}$POSTECH, South Korea
}
}

\begin{document}

\maketitle

\begin{abstract}

Functional brain connectome represents neural connectivity as a matrix of pairwise interactions between brain regions. Generation of functional connectomes is not only a question of validity; rather, having satisfied the constraints on the correlation matrix, the next step is to recover the class-conditional geometry buried under coarse labels. We propose MAGNET, a Manifold-Aware Graph Diffusion Network which uses a normalized-Cholesky representation of the manifold of correlation matrices that guarantees validity. MAGNET lifts noisy latent states into ROI-level region tokens and performs denoising with a relational inductive bias over brain atlas regions. To deal with structural problems induced by coarse labels, MAGNET employs class-anchored conditioning, amortized structural bridge, and relevance-preserving corruption. Across ABIDE, ADNI, and OASIS-3, MAGNET consistently achieves favorable results compared to previous manifold-aware approaches, demonstrating improvements of $7-21\%$ in class-conditional fidelity ($\alpha,\beta$-F1) across three cohorts and better sampling efficiency. Moreover, while training only with strict binary labels, MAGNET is capable of maintaining clinical heterogeneity through fine substructure of the connectomes in a zero-shot setting, improving subclass covariance alignment ($\lambda$-MSE) by over 30\%. These results suggest that geometric validity is a necessary but insufficient condition for clinical utility in connectome synthesis. Moreover, efforts in making the diffusion denoising class-conditional manifold aware finds utility beyond the highly curved brain connectome generation as this is a critical problem in various general settings. The code repository can be found at \url{https://github.com/protyayofficial/MAGNET}.

\end{abstract}

\vspace{-5pt}
\section{Introduction}\label{sec:intro}
\vspace{-5pt}

A brain connectome is a comprehensive map of neural connectivity~\cite{sporns2005human}, formalized as a relational object where each entry encodes interactions between pairs of distinct brain regions. Resting-state functional connectomes play a central role in biomarker discovery, cohort analysis, and personalized prognosis in modern neuroimaging~\cite{smith2013functional, castellanos2013clinical}. Conditional generative model for connectomes enables key applications in data-scarce settings, including augmentation~\cite{dadi2019benchmarking, ahn2025spectral}, synthesis of underrepresented populations~\cite{navidi2025adaptive}, and construction of synthetic cohorts~\cite{varol2017hydra}. Related conditional diffusion approaches have also been
explored for longitudinal neurodegenerative data
synthesis~\cite{cho2025conditional}. Such a model must satisfy two criterias: (i) samples must be valid correlation matrices (symmetric, positive definite, with unit diagonals)~\cite{pennec2006riemannian, archakov2021new}; and (ii) must lie within the valid class-conditional regions of the connectome manifold, ensuring both geometric validity and clinical consistency~\cite{benny2021evaluation, regenwetter2023beyond,collas2025riemannian}. While manifold-aware generative models address validity, accurately recovering class-conditional support remains a key challenge.

Riemannian score models approximate distributions with manifold support~\cite{de2022riemannian}, while Riemannian flow matching and pullback-based pipelines enable efficient transport without leaving the manifold~\cite{huang2022riemannian,collas2025riemannian}. Recent work on manifold-preserving diffusion, guidance, and Lie group methods further emphasizes that geometry is intrinsic to the learning problem~\cite{he2023manifold,chung2024cfg++,bertolini2025diffusion,he2025position}. For connectomes, the normalized-Cholesky pullback~\cite{archakov2021new,collas2025riemannian} guarantees validity by construction, ensuring every decoded sample is a proper correlation matrix. As illustrated in Fig.~\ref{fig:intro_1d}, valid structures are clinically meaningful, and the primary manifold direction smoothly \textit{stratifies} diagnostic groups across cohorts.

\begin{wrapfigure}{r}{0.53\linewidth}
    \centering
    \vspace{-1\baselineskip}
    \includegraphics[width=\linewidth]{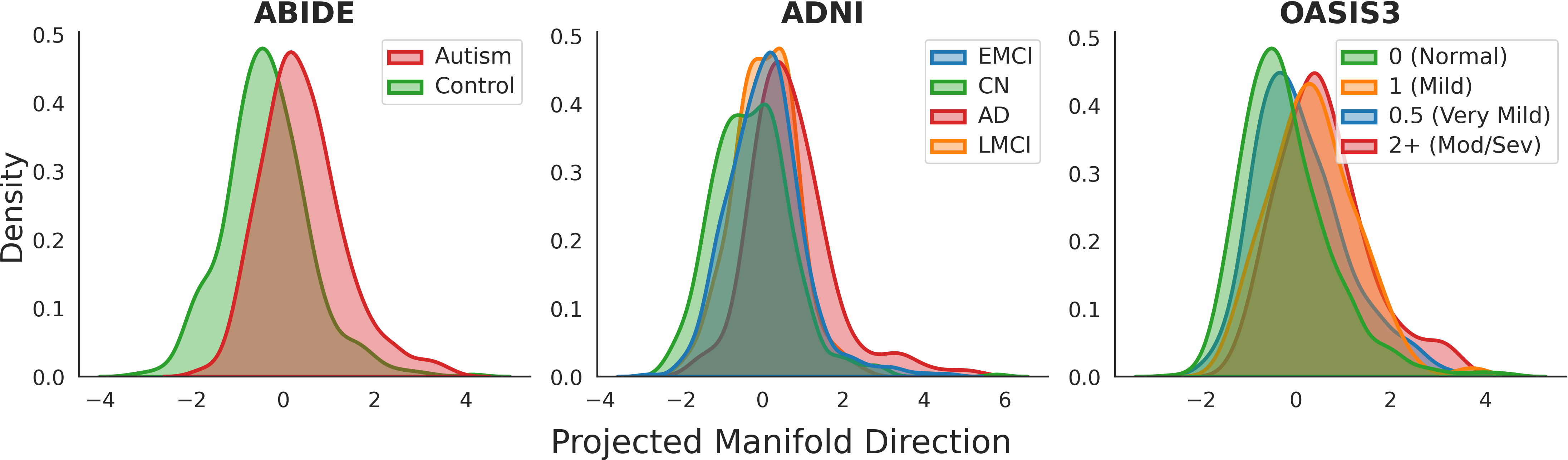}
    \vspace{-1.5em}
    \caption{\footnotesize \textbf{1D manifold projection.} Density estimates along the primary manifold direction across ABIDE, ADNI, and OASIS-3. A consistent diagnostic shift along the dominant geometric axis confirms that valid manifold structure carries clinical signal. Healthy controls are anchored in green.}
    \label{fig:intro_1d}
    \vspace{-1.5em}
\end{wrapfigure}
In neuroimaging, binary labels often collapse structurally diverse subpopulations into a single conditioning signal~\cite{varol2017hydra,wolfers2018mapping,wen2024dimensional}. Prior brain network classification
work has highlighted the importance of modeling
progressive diagnostic relationships and heterogeneous
graph structure~\cite{cho2024neurodegenerative}. This challenge is further exacerbated by inter-cohort distribution shifts and covariate imbalances, which can obscure clinically meaningful heterogeneity in neuroimaging datasets \cite{chytas2024pooling}. Fig.~\ref{fig:intro_2d} shows \emph{class-conditional distributions are not centered around a single prototype but instead span directions orthogonal to the main diagnostic axis, reflecting substantial variability that binary labels cannot capture.} Conditioning a generator on such coarse labels (e.g., labeling a scan simply as ``Alzheimer's'' without accounting for different subtypes or disease stages) forces the model to blur these distinct subgroups together into an unrealistic average~\cite{goodfellow2016nips}. Moreover, even when the generator produces samples in the correct geometric region, it may still fail to capture the underlying structure~\cite{benny2021evaluation,regenwetter2023beyond}. Thus, the challenge of support recovery, i.e., \emph{correctly assigning probability mass to the appropriate class-conditional regions on the manifold of valid connectomes}, remains unresolved (cannot be assessed by geometric validity alone).  

\begin{wrapfigure}{r}{0.55\linewidth}
    \centering
    \vspace{-1\baselineskip}
    \includegraphics[width=\linewidth]{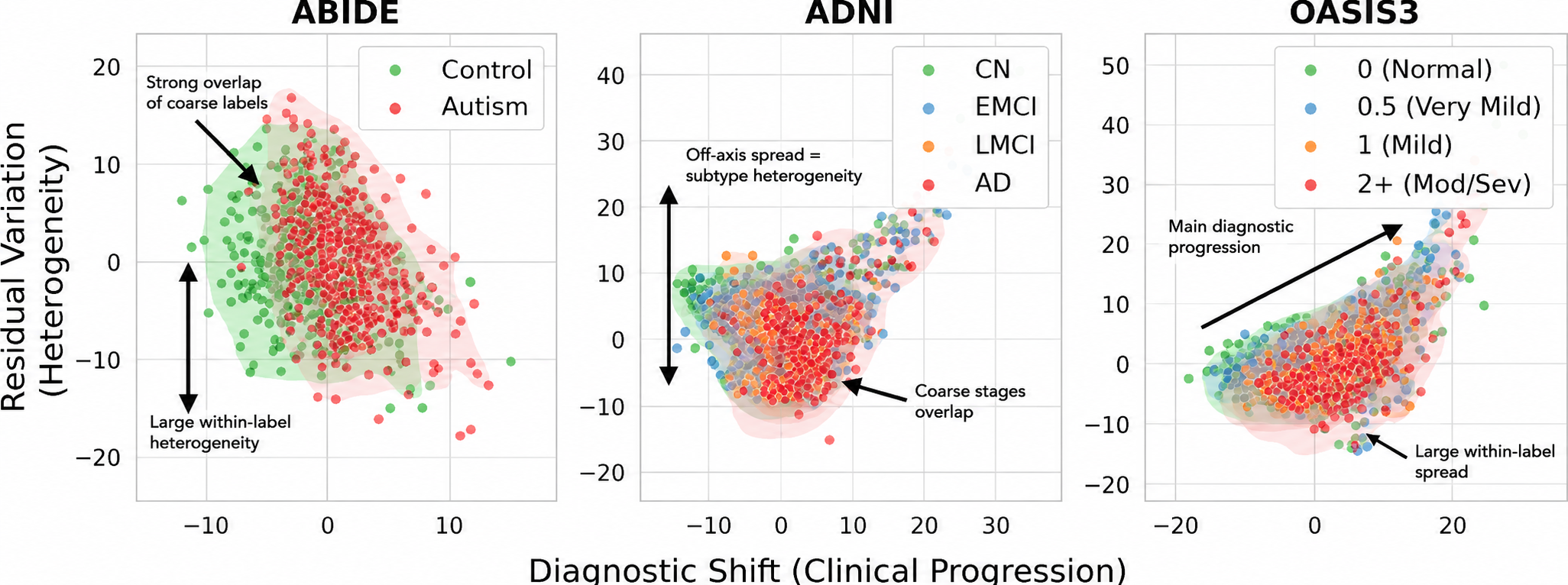}
    \vspace{-1.5em}
    \caption{\footnotesize \textbf{Latent structural heterogeneity.} 2D projections of connectomes across ABIDE, ADNI, and OASIS-3. \emph{Key observation:} The extensive vertical spread (residual variation) and severe overlap within individual clinical labels demonstrate that coarse categories collapse diverse subpopulations. This off-axis variation proves classes lack a single prototype, visually confirming the core challenge of conditional support recovery.}
    \label{fig:intro_2d}
    \vspace{-1em}
\end{wrapfigure}

\textbf{Core challenges in conditional support recovery.} While a valid chart is necessary, conditional connectome synthesis faces three inherent challenges. The \textit{label challenge} arises because coarse binary labels collapse structurally heterogeneous subpopulations, forcing feature blending~\cite{varol2017hydra, wolfers2018mapping, wen2024dimensional, japkowicz2002class, buda2018systematic, kang2019decoupling, byrd2019effect}. The \textit{global-context challenge} occurs because local noisy states lack sufficient information to recover global structure under limited reverse steps, while standard classifier-free guidance can overspecialize and reduce coverage~\cite{lipman2022flow, khan2026isokinetic, chung2024cfg++, ho2022classifier, kynkaanniemi2024applying}. The \textit{coordinate challenge} stems from isotropic corruption in a nonlinear chart, where equal-variance noise distorts informative coordinates and practically degrades class-conditional structure, making recovery increasingly difficult under constrained few-step sampling regimes~\cite{huang2022riemannian, bertolini2025diffusion, park2025temporal, su2025diffusion, wu2025optimal, song2020denoising}. 

All three challenges are compounded by a structural mismatch inherent to the chart: connectomes are relational objects~\cite{battaglia2018relational}, but the normalized-Cholesky parameterization imposes a strict ordering on brain regions, making latent coordinates explicitly not permutation-equivariant~\cite{zaheer2017deep, maron2018invariant}. Graph-based inductive biases are therefore motivated not by permutation invariance~\cite{bronstein2021geometric} but by the structured dependencies among ordered Cholesky coefficients. Each $z_{ij}$ encodes the orthogonalized partial correlation of region $i$ with respect to the preceding $j-1$ regions in the atlas ordering. Moreover, ROI-level attention with learned anatomical positional embeddings~\cite{he2025position, gao2025toward} consistently outperforms flat denoisers that ignore this relational structure~\cite{kan2022brain, bessadok2022graph, cui2022braingb, bronstein2017geometric}. Together, these shift the focus beyond geometric feasibility: \textit{the core difficulty is conditional support recovery}.

Keeping the aforementioned points in mind, we introduce \textbf{MAGNET}, a \textbf{M}anifold-\textbf{A}ware \textbf{G}raph \textbf{Net}work diffusion framework operating in a normalized-Cholesky chart~\cite{archakov2021new, collas2025riemannian}, where denoising is performed over region tokens rather than flattened coordinates. The central insight is that geometric validity is a necessary but insufficient condition and the binding constraint is conditional support recovery. It requires jointly resolving label ambiguity, missing global context, and coordinate corruption within the valid manifold. MAGNET addresses this through three coupled mechanisms: 

\begin{itemize}[leftmargin=*, nosep]
    \item We address the \emph{label challenge} via \textbf{class-anchored conditioning}: each class is modeled as a Gaussian mixture~\cite{mclachlan2019finite,fraley2002model,dempster1977maximum}, assigning subtype tokens to explicitly factor latent structural heterogeneity before conditioning reaches the denoiser~\cite{varol2017hydra,wolfers2018mapping,dhariwal2021diffusion,zhao2021improved}.
    
    \item For the \emph{global-context challenge}, we introduce an \textbf{amortized structural bridge}, a lightweight predictor that provides subtype-specific structural summaries derived from correlation statistics at each reverse step. This compensates for the inability of local noisy states to encode global connectivity under few-step regimes, while avoiding the diversity collapse typically induced by classifier-free guidance in limited-step inference~\cite{ho2022classifier,chung2024cfg++,kynkaanniemi2024applying,sadat2023cads}.
    
    \item Finally, the \emph{coordinate challenge} is handled via \textbf{relevance-preserving corruption}, a condition-adaptive forward process that uses the conditional–unconditional denoising gap to identify informative chart coordinates and reduce their corruption variance early. Unlike auxiliary-prior or conditioning-annealing approaches~\cite{su2025diffusion,sadat2023cads}, we directly modulate corruption covariance, ensuring class-relevant structure survives into the reverse trajectory where it is most recoverable~\cite{huang2022riemannian,park2025temporal}.
\end{itemize}

Sec.~\ref{sec:methodology} formalizes generation as a support-recovery problem, decoupling geometric validity from clinical fidelity. We establish theoretical guarantees and provide ablations (App.~\ref{app:frameworkanalysis}) validating our class-anchored conditioning, amortized bridge, and relevance-preserving corruption. Empirically (see Sec.~\ref{sec:exp_results}), MAGNET demonstrates robust conditional fidelity across ABIDE, ADNI, and OASIS-3, achieving superior zero-shot subclass recovery alongside accelerated sampling times.

\section{Methodology}
\label{sec:methodology}

We construct MAGNET using four coupled components, targeting the three challenges identified in Sec.~\ref{sec:intro} - label ambiguity from coarse supervision, missing global context under few-step inference, and coordinate-space corruption in the nonlinear chart. Sec.~\ref{subsec:post_validity}--\ref{subsec:graph_denoiser} establish the validity-preserving chart and ROI-level graph denoiser; Sec.~\ref{subsec:subtypes}--\ref{subsec:adaptive_corruption} introduce the three controllers that address label ambiguity, missing global context, and coordinate-space corruption respectively. Sec.~\ref{subsec:zero_shot} evaluates zero-shot subclass preservation. All proofs, ablations, and geometric analyses are in App.~\ref{app:proofs}--\ref{app:reversetime}.

\vspace{-1em}
\subsection{From validity to support recovery}
\label{subsec:post_validity}
\vspace{-0.5em}

Let $\mathcal{C}_d = \left\{C \in \mathbb{R}^{d \times d} : C = C^\top,\; C \succ 0,\; \mathrm{diag}(C) = \mathbf{1}\right\}$ denote the manifold of valid correlation connectomes~\cite{pennec2006riemannian}. We study conditional generation $C \sim p_\theta(C \mid y)$, where $y$ is a coarse label. Once validity is ensured, the central challenge shifts to \emph{support recovery}: generating samples from the correct class-conditional regions of the connectome manifold (Figure~\ref{fig:intro_2d}). This is difficult under coarse labels, which collapses structurally heterogeneous subpopulations into a single label~\cite{varol2017hydra,wolfers2018mapping,wen2024dimensional}. As this issue is orthogonal to geometric validity, validity-preserving transport alone is insufficient~\cite{collas2025riemannian}. Consequently, models may produce geometrically valid but clinically implausible samples by averaging across distinct modes~\cite{goodfellow2016nips,benny2021evaluation}, reducing diversity needed for downstream tasks~\cite{navidi2025adaptive,wolfers2018mapping}. MAGNET addresses this via three components: (i) a validity-preserving latent chart, (ii) an ROI-level graph denoiser, and (iii) controllers targeting coarse-label ambiguity, missing global context, and early corruption of condition-relevant coordinates.

\textbf{Geometry-preserving latent chart.} Instead of diffusing directly on $C$, we operate in the normalized-Cholesky chart~\cite{collas2025riemannian}, which bijectively parameterizes the free degrees of freedom of correlation matrices and guarantees valid outputs upon inversion~\cite{archakov2021new}. Let $C=LL^\top$ with $L$ lower triangular and positive diagonal, and define $\widetilde L=\mathrm{diag}(L)^{-1}L$. The chart coordinate is $z=\chi(C):=\operatorname{vech}_0(\widetilde L)\in\mathbb{R}^m$, where $m=\frac{d(d-1)}{2}$. Diffusion is performed on standardized coordinates $x=(z-\mu_z)/\sigma_z$, with statistics computed from training data~\cite{archakov2021new,pinkney2024short,collas2025riemannian}. Decoding reverses this process: we reconstruct $\widehat L$, form $G=\widehat L\widehat L^\top$, and renormalize to obtain $\widehat C=\mathrm{diag}(G)^{-1/2}\,G\,\mathrm{diag}(G)^{-1/2}$. By construction, $\widehat C$ is symmetric, positive definite, and unit-diagonal~\cite{archakov2021new,collas2025riemannian}, so validity is enforced intrinsically rather than via post-hoc projection. We reserve $\chi$ for this chart, while structural feature extractors used in the bridge and spectral loss are denoted separately by $\varphi(\cdot)$. 

\subsection{ROI-level graph denoising in the validity-preserving chart}
\label{subsec:graph_denoiser}
\vspace{-0.5em}

\begin{figure*}[!t]
    \centering
    \includegraphics[width=\linewidth]{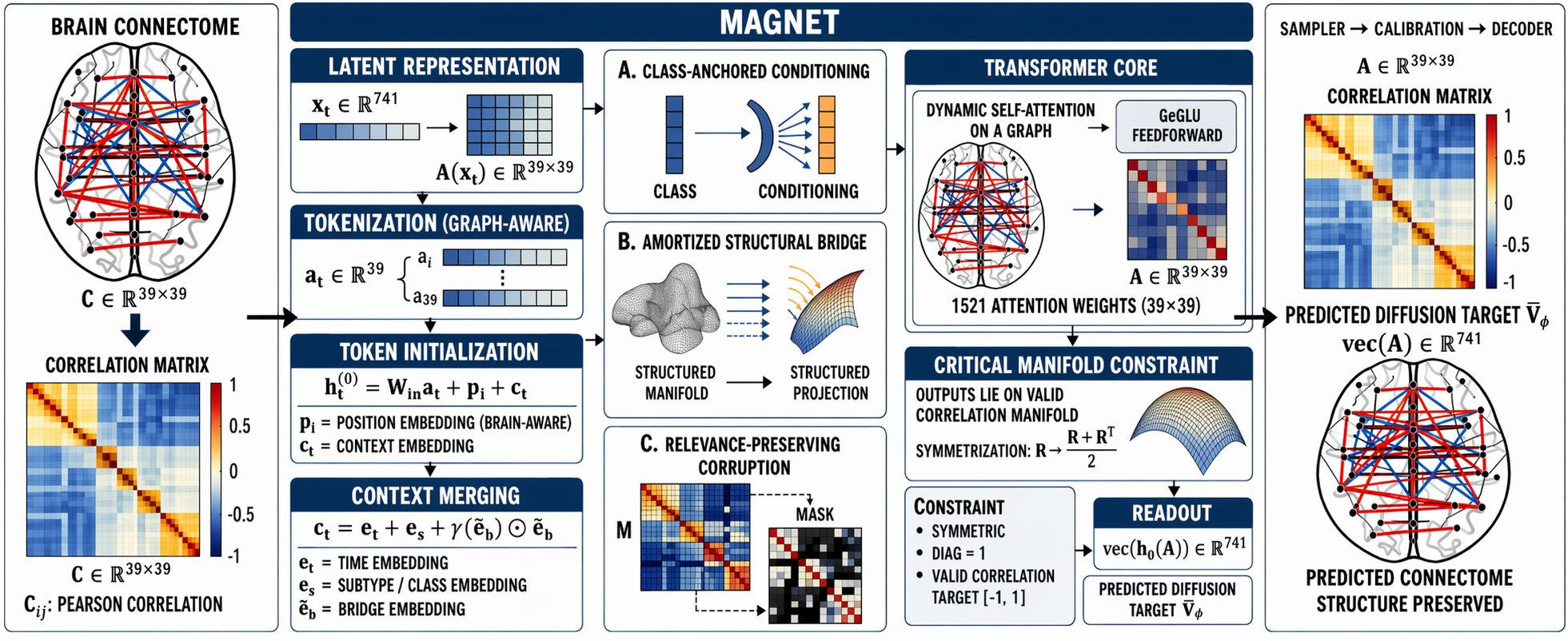}
    \vspace{-1.5em}\caption{\footnotesize\textbf{Illustration of MAGNET Architecture.} The pipeline maps raw correlation matrices to validity-preserving coordinates ($x_t$), which are tokenized and merged with additive context embeddings (time, subtype, and bridge). The Transformer core then performs graph-aware denoising, modulated by three dedicated controllers: \textbf{Class-Anchored Conditioning, Amortized Structural Bridge, and Relevance-Preserving Corruption}. Finally, row projection and symmetrization guarantee the predicted structure remains on the valid correlation manifold.}
    \label{fig:architecture}
    \vspace{-1em}
\end{figure*}

Validity-preserving charts guarantee geometrical feasibility but lack the relational inductive biases required to recover class-conditional interactions in flat latent spaces~\cite{battaglia2018relational,gao2025toward,he2025position,collas2025riemannian,bertolini2025diffusion}. Because functional connectomes inherently require representing nodes alongside relationship patterns~\cite{kan2022brain,bessadok2022graph,cui2022braingb}, MAGNET denoises ROI tokens directly. This successfully combines geometric feasibility with necessary structural biases~\cite{battaglia2018relational} (validated against flat charts in Appendix~\ref{subsec:backbone_controller_ablation}).

Given $x_t$, we reshape $x_t$ into a symmetric, zero-diagonal matrix $A(x_t) \in \mathbb{R}^{d \times d}$. Each row of this matrix is then used as an initial region token~\cite{kan2022brain,vaswani2017attention}. Concatenating orthogonal signals like timestep, class, and context often entangles their semantic roles. Instead, MAGNET uses additive token-level conditioning, ensuring each channel is independently optimizable~\cite{perez2018film,peebles2023scalable,nichol2021improved,ho2020denoising}. Defining $p_i$ as the learned anatomical positional embedding, $e_t$ as the timestep embedding~\cite{ho2020denoising}, $e_s$ as the subtype embedding, and $\widetilde{e}_b = \mathrm{Embed}_{\mathrm{spec}}(b_t)$ as the bridge embedding, where $b_t = B_\psi(x_t, t, s)$. Here, $x_t$ is the noisy normalized cholesky state, t is the diffusion timestep, s is the class-anchored subtype, and $B_\psi$ (.) is the amortized bridge predictor (see sec.~\ref{subsec:bridge}).  MAGNET forms a shared conditioning vector and initializes ROI tokens as:
\vspace{-0.5em}
\begin{equation}
c_t = e_t + e_s + \xi(\widetilde e_b)\odot \widetilde e_b,
\qquad
h_i^{(0)} = W_{\mathrm{in}}a_i(x_t) + p_i + c_t,
\label{eq:node_init}
\end{equation}
where $\xi(\cdot)$ is a learned gating map and $W_{\mathrm{in}}$ is a row projection. A pre-norm transformer stack~\cite{vaswani2017attention,xiong2020layer} with QK-normalized attention~\cite{henry2020query} and GeGLU feed-forward layers~\cite{shazeer2020glu} produces the final token matrix $H^{(L)}$. MAGNET then applies a row-wise output projection followed by symmetrization:
\begin{equation}
R = W_{\mathrm{out}}H^{(L)},
\qquad
E = \frac{1}{2}\left(R + R^\top\right),
\qquad
\widehat v_\theta = \operatorname{vech}_0(E).
\label{eq:readout}
\end{equation}
Eq.~\eqref{eq:readout} maps the region-level representation back to chart coordinates. Fig.~\ref{fig:architecture} summarizes this full architecture. The chart enforces valid correlation matrix output geometry~\cite{collas2025riemannian}, while Transformer backbone performing ROI-level graph denoising supplies the relational inductive bias~\cite{battaglia2018relational} needed for conditional support recovery. Detailed component implementations are provided in App.~\ref{app:archdetails}.


\vspace{-5pt}
\subsection{MAGNET Framework}
\vspace{-5pt}
\subsubsection{Class-anchored conditioning}
\vspace{-5pt}
\label{subsec:subtypes}

First problem arises due to coarse labels, where the application of one label to multiple latent structural subpopulations~\cite{varol2017hydra, wolfers2018mapping,wen2024dimensional} results in sample points located in different off-axis sub-regions rather than a single prototype (Fig.~\ref{fig:intro_2d}). Conditional on such labels, generator has no choice but to create a mean representation for each mode within the class~\cite{goodfellow2016nips}, generating geometrically valid yet clinically unrealistic samples~\cite{benny2021evaluation,regenwetter2023beyond,navidi2025adaptive}. Existing works on subpopulation identification and disease heterogeneity modeling shows that an explicit decomposition of the coarse labels into latent substructures significantly alleviates mode-averaging~\cite{varol2017hydra,wolfers2018mapping,wen2024dimensional}. Ignoring this sub-manifold structure harms generalization, analogous to label conflation in classification~\cite{japkowicz2002class,buda2018systematic}, and biases the denoiser toward majority modes, compressing diversity toward dominant structural patterns~\cite{byrd2019effect,kang2019decoupling}.

MAGNET refines each coarse class $y$ with a \emph{train-only subtype variable $s$}. For each class, we fit a diagonal-covariance Gaussian mixture in standardized chart coordinates via expectation maximization~\cite{dempster1977maximum}: $p(x \mid y)\approx \sum_{s=1}^{K_y}\pi_{y,s}\,p(x \mid y,s)$. We select $K_y$ by BIC~\cite{schwarz1978estimating,fraley2002model} and assign hard subtype labels $s_i=\arg\max_s p(s \mid x_i,y_i)$~\cite{mclachlan2019finite}. At generation, we sample $s\sim p(s \mid y)$ and condition the denoiser on the corresponding token. This variable $s$ captures latent within-class structure rather than a validated biological taxonomy~\cite{wolfers2018mapping,wen2024dimensional}.

\begin{proposition}[Subtype refinement monotonically reduces Bayes denoising risk]
\label{prop:anchored_risk}
Assume $\mathbb{E}\|X_0\|_2^2 < \infty$. Let $X_0$ be the clean latent variable, $X_t$ its noisy version, $Y$ the coarse class label, and $S$ a class-anchored subtype variable. Under squared loss, we define the Bayes denoising risk

\vspace{-0.5em}
\begin{equation}
\mathcal R_Y^\star
=
\inf_f \mathbb E\!\left[\|X_0-f(X_t,Y)\|_2^2\right],
\qquad
\mathcal R_{Y,S}^\star
=
\inf_g \mathbb E\!\left[\|X_0-g(X_t,Y,S)\|_2^2\right].
\label{eq:anchored_risks}
\end{equation}
\vspace{-0.5em}
Then
\vspace{-0.5em}
\begin{equation}
\mathcal R_Y^\star-\mathcal R_{Y,S}^\star
=
\mathbb E\!\left[
\left\|
\mathbb E[X_0\mid X_t,Y,S]-\mathbb E[X_0\mid X_t,Y]
\right\|_2^2
\right]
\ge 0.
\label{eq:anchored_identity}
\end{equation}
In particular, subtype refinement is beneficial exactly when it captures within-class structure unresolved by the coarse label.
\end{proposition}

\textbf{Remark.} Eq.~\ref{eq:anchored_identity} demonstrates that subtype refinement improves denoising exactly when $S$ captures within-class structure unresolved by the coarse label $Y$ (Fig.~\ref{fig:intro_2d}). Proofs are in App.~\ref{app:proof_anchored} and \ref{appsubsec:conditioning}.

\vspace{-5pt}
\subsubsection{Amortized structural bridge}
\vspace{-5pt}
\label{subsec:bridge}

Subtype conditioning is a solution for the problem of labeling, but it does not have the required global structural knowledge to use a few-step reverse sampling technique, hence creating a dilemma for the choice of either out-of-distribution averaging~\cite{benny2021evaluation} or constrained manifold learning~\cite{sadat2023cads,kynkaanniemi2024applying,chung2024cfg++}. In order to resolve this ambiguity without additional forward samples, we condense a structural feature of the subtype from a pre-trained dictionary of data into an amortized predictor~\cite{hinton2015distilling,chen2022analog}. Unlike self-conditioning, which recycles recursive denoiser outputs~\cite{chen2022analog}, this bridge provides an independently supervised, compact structural prior to offset the limitations of the noisy input $x_t$.

Let $\varphi(C)\in\mathbb{R}^{16}$ denote the proxy structural feature extractor, combining spectral properties (e.g., moments, rank, condition number) and off-diagonal connectivity statistics. For each subtype $s$, MAGNET fits a summary atom set $\mathcal A_s=\{a_{s,1},\dots,a_{s,K_s}\}\subset\mathbb{R}^{16}$ ($K_s\in\{1,2\}$) using either the subtype mean or two principal-split means, depending on training support. Given a clean latent $x_0$, the teacher bridge target is a distance-weighted convex combination of these atoms: $b^\star(x_0,s,t) = \sum_{k=1}^{K_s} w_k(x_0,s,t)\,a_{s,k}$. Because $x_0$ is unavailable at inference, MAGNET amortizes this process by predicting the target directly from the noisy state: $b_t = B_\psi(x_t,t,s) = f_\psi\!\left(\varphi_{\mathrm{cheap}}(x_t),e_t,e_s\right)$. Here, $\varphi_{\mathrm{cheap}}$ is a low-cost proxy of the noisy chart coordinates. Trained to match the teacher target, $B_\psi$ is injected via Eq.~\ref{eq:node_init}. By entirely replacing the teacher during sampling, $B_\psi$ bypasses repeated clean-state estimations, yielding a $24-33\times$ speedup (Table~\ref{tab:bridge_ablations}).

\vspace{-1.5em}
\begin{equation}
L_{\mathrm{bridge}} = \left\| B_\psi(x_t,t,s)-b^\star(x_0,s,t) \right\|_2^2. 
\label{eq:bridge_loss}
\end{equation}
At inference time, $B_\psi$ replaces the teacher entirely, yielding the approximately $25-32\times$ sampling speedup over teacher-bridge alternatives shown in Tab.~\ref{tab:bridge_ablations}.

\begin{proposition}[Bridge-assisted denoising risk and amortization error]
\label{thm:bridge_risk}
Let $U=(X_t,t,S)$, where $X_t$ is the noisy latent state, $t$ is diffusion time, and $S$ is the subtype token. Let $B^\star=b^\star(X_0,S,t)$ denote the teacher bridge target. Under squared loss, define
\begin{equation}
\mathcal R_U^\star
=
\inf_f \mathbb E\!\left[\|X_0-f(U)\|_2^2\right],
\qquad
\mathcal R_{U,B^\star}^\star
=
\inf_g \mathbb E\!\left[\|X_0-g(U,B^\star)\|_2^2\right].
\label{eq:bridge_risks}
\end{equation}
\vspace{-0.5em}
Then
\begin{equation}
\mathcal R_U^\star-\mathcal R_{U,B^\star}^\star
=
\mathbb E\!\left[
\left\|
\mathbb E[X_0\mid U,B^\star]-\mathbb E[X_0\mid U]
\right\|_2^2
\right]
\ge 0.
\label{eq:bridge_help}
\end{equation}
Moreover, if the Bayes regressor $\eta(u,b)=\mathbb E[X_0\mid U=u,B^\star=b]$ is $L$-Lipschitz in $b$, then the plug-in estimator $\widehat X_0(U)=\eta(U,B_\psi(U))$ satisfies
\begin{equation}
\mathbb E\!\left[\|X_0-\widehat X_0(U)\|_2^2\right]-\mathcal R_{U,B^\star}^\star
\le
L^2\,\mathbb E\!\left[\|B_\psi(U)-B^\star\|_2^2\right].
\label{eq:bridge_bound}
\end{equation}
\end{proposition}
\paragraph{Remark.} Prop.~\ref{thm:bridge_risk} shows that conditioning on the teacher bridge $B^\star$ reduces denoising risk~\cite{miyasawa1961empirical, efron2011tweedie} (Eq.~\ref{eq:bridge_help}), while minimizing $L_{\mathrm{bridge}}$ bounds the inference gap (Eq.~\ref{eq:bridge_loss}). As shown later in Tab.~\ref{tab:bridge_ablations} (App.~\ref{appsubsec:bridge}), lower distillation error directly improves $\alpha,\beta$-F1, emprirically validating the amortization bound (Eq.~\ref{eq:bridge_bound}). Furthermore, omitting $L_{\mathrm{bridge}}$ spikes tracking error by two orders of magnitude with no fidelity gain, confirming this active bridge mode is essential rather than optional (App.~\ref{appsubsec:bridge}, \ref{app:proof_bridge}).

\vspace{-5pt}
\subsubsection{Relevance-preserving corruption}
\vspace{-5pt}
\label{subsec:adaptive_corruption}

The third failure mode is geometric. Although the chart $\chi$ guarantees validity after decoding, isotropic corruption in chart coordinates is not geometry-neutral: because $\chi$ is nonlinear, equal-variance perturbations can move the noisy state away from the local class-conditional support that the reverse process must recover. This is especially harmful in the few-step regime where interpolation steps are larger. Motivated by geometry-aware diffusion and on-manifold correction methods \cite{huang2022riemannian,bertolini2025diffusion,he2023manifold,lee2025local,park2025temporal,roy2026manifoldgd}, we therefore replace isotropic corruption with a manifold-faithful, condition-adaptive noising rule as trajectory deviation at any timestep leads to accumulation of deviations away from the manifold \cite{roy2026manifoldgd,park2025temporal}.

Starting from a provisional isotropic sample $\tilde{x}_t = \sqrt{\bar{\alpha}_t}\,x_0
+ \sqrt{1-\bar{\alpha}_t}\,\epsilon$, $\epsilon \sim \mathcal{N}(0,I)$,
we estimate condition-relevant directions from the conditional-unconditional denoising gap
\begin{equation}
d_t
=
G_\theta(\tilde{x}_t,t,s,b_t)
-
G_\theta(\tilde{x}_t,t,\varnothing,\varnothing),
\quad
m_{t,j}
=
\frac{|d_{t,j}|}{\max_k |d_{t,k}|+\varepsilon},
\label{eq:relevance_score}
\end{equation}
where $m_t \in [0,1]^m$ scores chart coordinates by conditional relevance. Following the auxiliary-prior view of diffusion guidance \cite{su2025diffusion}, we define a task-specific manifold-faithfulness prior
\begin{equation}
\pi_t(x_t \mid x_0,\tilde{x}_t,s,b_t)
\propto
\exp\!\left(
-\lambda_t
\left\|
\sqrt{m_t}\odot
\bigl(x_t-\sqrt{\bar{\alpha}_t}x_0\bigr)
\right\|_2^2
\right),
\qquad
\lambda_t=\frac{\kappa \rho(t)}{2(1-\bar{\alpha}_t)},
\label{eq:mf_prior}
\end{equation}
where $\rho(t)\in[0,1]$ is a monotonically decreasing protection schedule. Equation~\eqref{eq:mf_prior} penalizes corruption away from the support-aligned coordinates indicated by $m_t$: directions judged condition-relevant are protected most strongly early in diffusion. \footnote{During training, $G_\theta$ in Equation \ref{eq:relevance_score} is evaluated with a stop-gradient, so the relevance mask $m_{t}$ does not back propagate through the denoizer and diffusion marginals remain tractable.}

Multiplying the isotropic Gaussian kernel by this prior yields a diagonal Gaussian corruption law, $q_{\mathrm{mf}}(x_t \mid x_0,\tilde{x}_t,s,b_t) = \mathcal{N}\!\left(
\sqrt{\bar{\alpha}_t}x_0,\,
(1-\bar{\alpha}_t)\,\mathrm{Diag}(\sigma_t^2)
\right)$, with coordinate-wise variance $\sigma_{t,j}^2 = \operatorname{clip}\!\bigl(
1-\kappa \rho(t)m_{t,j},
\ \sigma_{\min}^2,\ 1
\bigr)$. The implemented forward corruption is therefore $
x_t = \sqrt{\bar{\alpha}_t}\,x_0
+ \sqrt{1-\bar{\alpha}_t}\,
(\sigma_t \odot \epsilon),
\ \epsilon \sim \mathcal{N}(0,I).$

This rule is annealed in the spirit of \cite{sadat2023cads}: protection is strongest at early timesteps, where coarse corruption can erase condition-defining structure, and vanishes later so that the process reverts to standard isotropic noising near the clean end. Unlike \cite{sadat2023cads}, however, we anneal the corruption covariance rather than the conditioning signal itself. Thus MAGNET reduces off-support drift at its source: coordinates that matter most for conditional denoising retain higher signal-to-noise ratio early in the trajectory, while asymptotically recovering the standard diffusion kernel as $\rho(t)\to 0$. \textit{Intuitively, preserving condition-relevant structure for longer keeps the noisy state closer to the local class-conditional manifold at each timestep, making it easier for the reverse process to remain on, and return to, the correct support.} A Gaussian surrogate analysis formalizing this is in App.~\ref{app:proof_corruption} and corresponding ablations are reported in App.~\ref{appsubsec:corruption}.

\vspace{-5pt}
\subsection{Full Training Objective}
\vspace{-5pt}
\label{subsec:objective}

Our training objective is $L = L_{\mathrm{diff}} + \lambda_{\mathrm{bridge}} L_{\mathrm{bridge}} + \lambda_{\mathrm{spec}} L_{\mathrm{spec}}$. Here, $L_{\mathrm{diff}}$ corresponds to the diffusion loss, $L_{\mathrm{bridge}}$ is the bridge distillation error (Eq.~\ref{eq:bridge_loss}), and $L_{\mathrm{spec}}=\|\varphi(\widehat x_0)-\varphi(x_0)\|_2^2$ steers the predicted latent to retain structural features. $L_{\mathrm{diff}}$ uses min-SNR reweighting and inverse subject-frequency weights for longitudinal cohorts. $L_{\mathrm{bridge}}$ is additionally weighted by $\bar\alpha_t$ and masked to conditioned examples only and $L_{\mathrm{spec}}$ is mean-reduced over the minibatch.

We sample with DDIM~\cite{song2020denoising}, using six reverse steps as the primary setting. Generation begins by sampling a subtype $s \sim p(s \mid y)$ and initializing from Gaussian noise, followed by reverse denoising with subtype and bridge conditioning at every step. Instead of the standard linear CFG approach, we use a diversity-preserving variant that orthogonalizes the unconditional residual from the guidance direction (see App.~\ref{app:archdetails} for details), using $\gamma=1.5$  (App.~\ref{subsec:cfgscale}). This directly addresses the global-context failure discussed in Sec.~\ref{subsec:bridge}: plain linear CFG~\cite{ho2022classifier}  adds a fixed displacement along the guidance direction with a scale that can cause over-specialization toward narrow manifold regions~\cite{chung2024cfg++,kynkaanniemi2024applying}. Our formulation decomposes the unconditional prediction into components parallel and orthogonal to $g$, and \emph{replaces} the parallel variant with $\tanh$-saturated guidance instead of adding it. In standard CFG, the full unconditional prediction $\hat{v}_\emptyset$ is retained and $\gamma g$ is added which can keep the parallel component of $\hat{v}_\emptyset$ (might point in a conflicting direction to $g$) and can dilute directional control. By instead using $\hat{v}_\gamma = \hat{v}_\emptyset^\perp + \gamma \tanh(\|g\|_2)g$, the guidance direction is unambiguous and the $\tanh$ factor prevents over-steering for large $\|g\|_2$. The calibration is fit on the training split by matching class-conditional latent moments, and validity is preserved because decoding still proceeds through the normalized-Cholesky correlation chart (see App.~\ref{app:frameworkanalysis} for ablations and sampling sensitivity analyses).

\begin{table*}[!t]
\centering
\scriptsize
\setlength{\tabcolsep}{4pt}
\renewcommand{\arraystretch}{1.3}
\caption{\footnotesize
\textbf{Quantitative comparison of validity-preserving generative models under matched low-step budgets.}
We compare MAGNET against GDD-proj\cite{jo2022score}, TriangCFM, DiffeoGauss, and DiffeoCFM\cite{collas2025riemannian} on ABIDE~\cite{di2014autism}, ADNI~\cite{weiner2017alzheimer}, and OASIS-3~\cite{lamontagne2019oasis}. We report $\alpha$-Precision, $\beta$-Recall, and $\alpha,\beta$-F1~\cite{alaa2022faithful} to evaluate class-conditional support fidelity, ROC-AUC and F1 to evaluate downstream utility, and wall-clock sampling time per split. The {\setlength{\fboxsep}{1.2pt}\colorbox{LightGray}{Real Data}} row provides the empirical oracle from split-wise real-data comparisons, and the {\setlength{\fboxsep}{1.2pt}\colorbox{LightBlue}{MAGNET}} row highlights our method. Bold denotes the best generative model for each metric and underlining denotes the runner-up. Supporting geometric analyses are given in Appendix~\ref{app:geometricanalysis}--\ref{app:reversetime}.
}
\label{tab:main_brain_results}
\vspace{-1.0em}
\begin{tabular}{llcccccc}
\toprule
\multirow{2}{*}{\textbf{Dataset}} & \multirow{2}{*}{\textbf{Method}} 
& \multicolumn{3}{c}{\textbf{Quality Metrics}} 
& \multicolumn{2}{c}{\textbf{CAS Metrics}} 
& \multicolumn{1}{c}{\textbf{Time (s.)}} \\
\cmidrule(lr){3-5}\cmidrule(lr){6-7}\cmidrule(lr){8-8}
& & \textbf{$\alpha$-Precision $\uparrow$} & \textbf{$\beta$-Recall $\uparrow$} & \textbf{$\alpha,\beta$-F1 $\uparrow$}
& \textbf{ROC-AUC $\uparrow$} & \textbf{F1 $\uparrow$}
& \textbf{Sampling $\downarrow$} \\
\midrule

\multirow{6}{*}{\textbf{ABIDE}}
& \cellcolor{LightGray} Real Data 
& \cellcolor{LightGray} 0.80 $\pm$ 0.04 & \cellcolor{LightGray} 0.81 $\pm$ 0.04 & \cellcolor{LightGray} 0.80 $\pm$ 0.01
& \cellcolor{LightGray} 0.66 $\pm$ 0.04 & \cellcolor{LightGray} 0.58 $\pm$ 0.06 & \cellcolor{LightGray} N/A \\
& TriangCFM
& 0.21 $\pm$ 0.21 & 0.24 $\pm$ 0.35 & 0.20 $\pm$ 0.29
& 0.50 $\pm$ 0.05 & 0.40 $\pm$ 0.13 & 7.16 $\pm$ 1.17 \\
& TriangDDPM & 0.25 $\pm$ 0.04 & 0.22 $\pm$ 0.06 & 0.23 $\pm$ 0.04
& 0.52 $\pm$ 0.05 & 0.43 $\pm$ 0.08 & 5.45 $\pm$ 0.19 \\
& DiffeoGauss
& 0.51 $\pm$ 0.04 & 0.23 $\pm$ 0.03 & 0.32 $\pm$ 0.03
& \best{0.66 $\pm$ 0.05} & \best{0.56 $\pm$ 0.06} & \best{0.06 $\pm$ 0.01} \\
& DiffeoCFM
& \best{0.81 $\pm$ 0.05} & \second{0.50 $\pm$ 0.04} & \second{0.62 $\pm$ 0.05}
& 0.60 $\pm$ 0.04 & \underline{0.52 $\pm$ 0.07} & 9.04 $\pm$ 0.86 \\
& GDSS-proj
& 0.00 $\pm$ 0.00 & 0.28 $\pm$ 0.07 & 0.00 $\pm$ 0.00
& 0.47 $\pm$ 0.05 & 0.29 $\pm$ 0.31 & 1.05 $\pm$ 0.02 \\
& \cellcolor{LightBlue} \textbf{MAGNET (ours)}
& \cellcolor{LightBlue} \second{0.80 $\pm$ 0.04} & \cellcolor{LightBlue} \best{0.59 $\pm$ 0.04} & \cellcolor{LightBlue} \best{0.68 $\pm$ 0.03}
& \cellcolor{LightBlue} \second{0.64 $\pm$ 0.07} & \cellcolor{LightBlue} \best{0.56 $\pm$ 0.06} & \cellcolor{LightBlue} \second{0.38 $\pm$ 0.02} \\
\midrule

\multirow{6}{*}{\textbf{ADNI}}
& \cellcolor{LightGray} Real Data
& \cellcolor{LightGray} 0.87 $\pm$ 0.09 & \cellcolor{LightGray} 0.87 $\pm$ 0.02 & \cellcolor{LightGray} 0.87 $\pm$ 0.05
& \cellcolor{LightGray} 0.58 $\pm$ 0.03 & \cellcolor{LightGray} 0.59 $\pm$ 0.03 & \cellcolor{LightGray} N/A \\
& TriangCFM
& 0.23 $\pm$ 0.31 & 0.25 $\pm$ 0.35 & 0.23 $\pm$ 0.33
& 0.52 $\pm$ 0.05 & 0.37 $\pm$ 0.10 & 21.70 $\pm$ 3.35 \\
& TriangDDPM
& 0.27 $\pm$ 0.06
& 0.30 $\pm$ 0.02
& 0.28 $\pm$ 0.03
& 0.54 $\pm$ 0.08
& 0.43 $\pm$ 0.05
& 13.92 $\pm$ 0.93 \\
& DiffeoGauss
& 0.02 $\pm$ 0.01 & 0.50 $\pm$ 0.08 & 0.04 $\pm$ 0.02
& \best{0.58 $\pm$ 0.03} & 0.28 $\pm$ 0.13 & \second{1.42 $\pm$ 0.40} \\
& DiffeoCFM
& \second{0.68 $\pm$ 0.12} & \second{0.75 $\pm$ 0.04} & \second{0.71 $\pm$ 0.05}
& \second{0.57 $\pm$ 0.04} & 0.40 $\pm$ 0.09 & 36.08 $\pm$ 5.84 \\
& GDSS-proj
& 0.57 $\pm$ 0.05 & 0.38 $\pm$ 0.07 & 0.45 $\pm$ 0.05
& 0.48 $\pm$ 0.07 & \underline{0.53 $\pm$ 0.36} & 1.70 $\pm$ 0.01 \\
& \cellcolor{LightBlue} \textbf{MAGNET (ours)}
& \cellcolor{LightBlue} \best{0.71 $\pm$ 0.13} & \cellcolor{LightBlue} \best{0.84 $\pm$ 0.06} & \cellcolor{LightBlue} \best{0.76 $\pm$ 0.06}
& \cellcolor{LightBlue} \second{0.57 $\pm$ 0.07} & \cellcolor{LightBlue} \best{0.61 $\pm$ 0.05} & \cellcolor{LightBlue} \best{0.86 $\pm$ 0.28} \\
\midrule

\multirow{6}{*}{\textbf{OASIS-3}}
& \cellcolor{LightGray} Real Data
& \cellcolor{LightGray} 0.89 $\pm$ 0.04 & \cellcolor{LightGray} 0.88 $\pm$ 0.05 & \cellcolor{LightGray} 0.88 $\pm$ 0.01
& \cellcolor{LightGray} 0.71 $\pm$ 0.04 & \cellcolor{LightGray} 0.61 $\pm$ \cellcolor{LightGray} 0.06 & \cellcolor{LightGray} N/A \\
& TriangCFM
& 0.22 $\pm$ 0.27 & 0.26 $\pm$ 0.37 & 0.23 $\pm$ 0.32
& 0.51 $\pm$ 0.05 & 0.45 $\pm$ 0.17 & 15.29 $\pm$ 2.51 \\
& TriangDDPM
& 0.26 $\pm$ 0.08
& 0.31 $\pm$ 0.05
& 0.28 $\pm$ 0.05
& 0.55 $\pm$ 0.07
& 0.48 $\pm$ 0.03
& 8.79 $\pm$ 1.31 \\
& DiffeoGauss
& 0.52 $\pm$ 0.06 & 0.31 $\pm$ 0.02 & 0.39 $\pm$ 0.03
& \best{0.68 $\pm$ 0.03} & 0.40 $\pm$ 0.09 & 1.88 $\pm$ 1.10 \\
& DiffeoCFM
& \second{0.63 $\pm$ 0.06} & \second{0.38 $\pm$ 0.04} & \second{0.47 $\pm$ 0.05}
& \best{0.68 $\pm$ 0.03} & \underline{0.53 $\pm$ 0.04} & 23.24 $\pm$ 2.55 \\
& GDSS-proj
& 0.01 $\pm$ 0.02 & 0.04 $\pm$ 0.01 & 0.01 $\pm$ 0.01
& 0.51 $\pm$ 0.03 & 0.30 $\pm$ 0.32 & 1.44 $\pm$ 0.01 \\
& \cellcolor{LightBlue} \textbf{MAGNET (ours)}
& \cellcolor{LightBlue} \best{0.75 $\pm$ 0.07} & \cellcolor{LightBlue} \best{0.47 $\pm$ 0.06} & \cellcolor{LightBlue} \best{0.57 $\pm$ 0.06}
& \cellcolor{LightBlue} \best{0.68 $\pm$ 0.03} & \cellcolor{LightBlue} \best{0.56 $\pm$ 0.04} & \cellcolor{LightBlue} \best{0.63 $\pm$ 0.00} \\
\bottomrule
\end{tabular}
\end{table*}

\begin{figure}[!t]
    \centering
    \includegraphics[width=\linewidth]{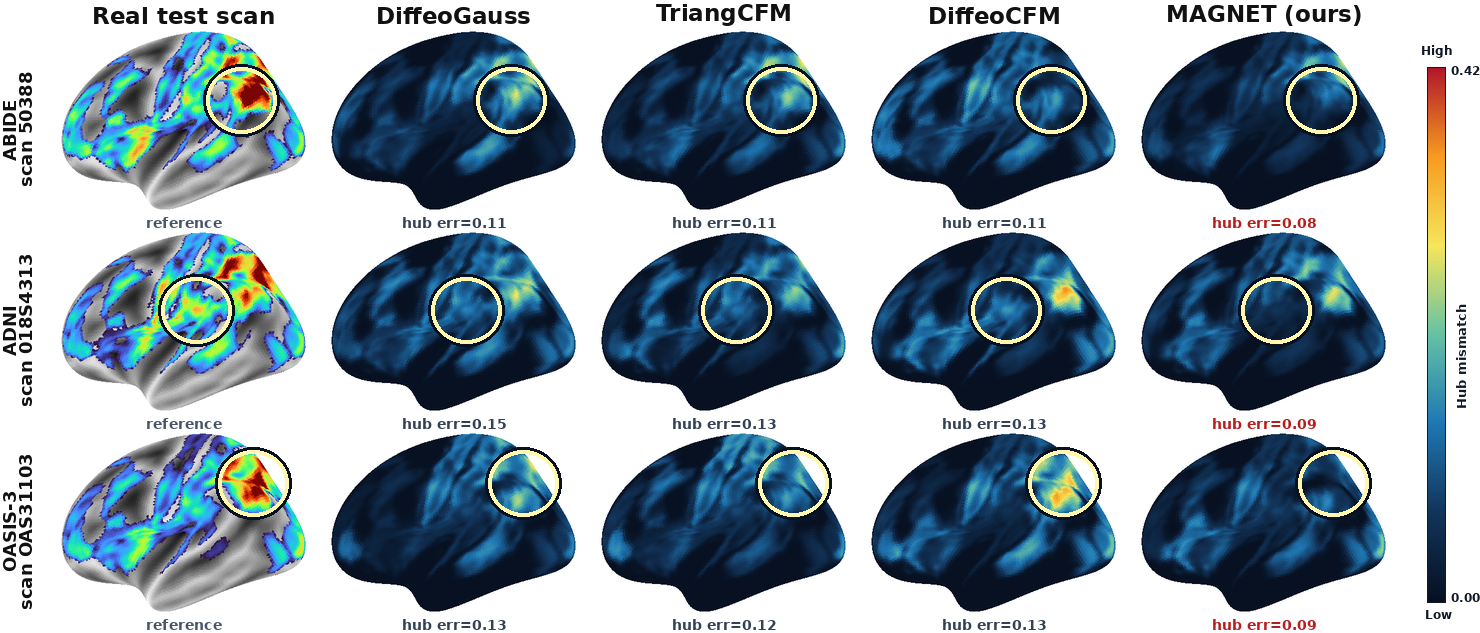}
    \vspace{-1.5em}
    \caption{\footnotesize
      \textbf{Spatial mismatch of cortical hubs.} Cortical surface maps illustrating the structural deviation of generated samples relative to a held-out real test scan (left). Brighter regions indicate a larger spatial mismatch compared to the real reference, while darker regions signify higher fidelity. MAGNET consistently produces the darkest surfaces across all cohorts, visually confirming superior support recovery and yielding the lowest quantitative hub errors (shown below each panel).}
    \label{fig:connectome_summary}
    \vspace{-1.0em}
\end{figure}

\vspace{-1em}
\section{Experimental Results}
\vspace{-1em}
\label{sec:exp_results}

We evaluate MAGNET on three neuroimaging cohorts: ABIDE~\cite{di2014autism}, ADNI~\cite{weiner2017alzheimer}, and OASIS-3~\cite{lamontagne2019oasis}. Regional time series are extracted using Nilearn~\cite{Nilearn} with the MSDL atlas~\cite{varoquaux2011multi} ($d=39$), z-scored, and converted to correlation matrices via OAS-estimated covariances~\cite{dadi2019benchmarking}. The dataset comprises 4989 scans with strict subject-level splits: ABIDE (846 scans; 455 control, 391 positive), ADNI (2343 scans; 980 control, 1363 positive), and OASIS-3 (1800 scans; 1047 control, 753 positive). 

In the primary benchmark, we use binary labels (control = 0, all diagnoses = 1), intentionally imposing coarse labels to test recovery of within-class structure (Sec.~\ref{subsec:quantitative}). A separate zero-shot multiclass evaluation assesses whether finer clinical subtypes are preserved despite binary training (Sec.~\ref{subsec:zero_shot}). Baseline details and experimental setup are provided in App.~\ref{appsec:baselines} and App.~\ref{subsec:exp_setup} respectively.
\vspace{-1em}
\subsection{Main Benchmark}
\vspace{-0.5em}
\label{subsec:quantitative}
\begin{table*}[!t]
\centering
\scriptsize
\caption{\footnotesize
\textbf{Zero-shot recovery of fine-grained clinical structure.}
Generators are trained using only binary labels, but evaluated on resolved four-class clinical spaces: CN/SMC/MCI/AD for ADNI and CDR $0/0.5/1/2+$ for OASIS-3. Real held-out samples retain their metadata subclass labels, while generated samples are mapped into the same subclass space using a classifier trained only on real training data. Macro $\alpha$-precision, $\beta$-recall, and $\alpha,\beta$-F1 measure support recovery on the resolved subclass manifolds; multiclass CAS macro-F1 measures downstream utility; class KL divergence and Bhattacharyya distance measure preservation of subclass balance; and $\lambda$-MSE measures alignment of subclass-specific covariance structure. Bold and underlining denote the best and second-best generative models for each metric.
}
\label{tab:zero_shot_metrics}
\setlength{\tabcolsep}{1.8pt}
\renewcommand{\arraystretch}{1}
\vspace{-1.0em}
\begin{tabular}{l l c c c c c c c}
\toprule
\multirow{2}{*}{} & \multirow{2}{*}{\textbf{Method}} 
& \multicolumn{3}{c}{\textbf{Quality (Macro-Averaged)}} 
& \multicolumn{1}{c}{\textbf{CAS}} 
& \multicolumn{3}{c}{\textbf{Distributional Preservation $\downarrow$}} \\
\cmidrule(lr){3-5} \cmidrule(lr){6-6} \cmidrule(l){7-9}
& & \textbf{$\alpha$-Prec. $\uparrow$} & \textbf{$\beta$-Recall $\uparrow$} & \textbf{$\alpha,\beta$-F1 $\uparrow$} & \textbf{Macro F1 $\uparrow$} & \textbf{Class $D_{KL}$} & \textbf{Class $D_B$} & \textbf{$\lambda$-MSE} \\
\midrule

\multirow{5}{*}{\rotatebox[origin=c]{90}{\textbf{ADNI}}}
& TriangCFM 
& $\mathbf{0.51 \pm 0.11}$ & $0.21 \pm 0.05$ & $0.25 \pm 0.05$ & $0.19 \pm 0.04$ & $0.19 \pm 0.07$ & $0.05 \pm 0.02$ & $\underline{9.16 \pm 0.16}$ \\
& DiffeoGauss 
& $0.00 \pm 0.00$ & $\underline{0.54 \pm 0.04}$ & $0.01 \pm 0.01$ & $0.26 \pm 0.03$ & $\underline{0.05 \pm 0.06}$ & $0.01 \pm 0.01$ & $11.37 \pm 0.64$ \\
& DiffeoCFM 
& $\underline{0.38 \pm 0.16}$ & $0.48 \pm 0.06$ & $\underline{0.33 \pm 0.07}$ & $\underline{0.27 \pm 0.03}$ & $0.06 \pm 0.08$ & $\underline{0.01 \pm 0.01}$ & $14.30 \pm 0.89$ \\
& GDSS-proj
& $0.27 \pm 0.10$ & $0.30 \pm 0.06$ & $0.20 \pm 0.04$
& $0.19 \pm 0.03$ & $0.34 \pm 0.16$ & $0.09 \pm 0.03$ & $13.47 \pm 10.59$ \\
& \cellcolor{LightBlue} \textbf{MAGNET (ours)} 
& \cellcolor{LightBlue} $0.34 \pm 0.07$ & \cellcolor{LightBlue} $\mathbf{0.56 \pm 0.14}$ & \cellcolor{LightBlue} $\mathbf{0.35 \pm 0.04}$ & \cellcolor{LightBlue} $\mathbf{0.28 \pm 0.02}$ & \cellcolor{LightBlue} $\mathbf{0.05 \pm 0.03}$ & \cellcolor{LightBlue} $\mathbf{0.01 \pm 0.01}$ & \cellcolor{LightBlue} $\mathbf{6.08 \pm 0.53}$ \\

\midrule

\multirow{5}{*}{\rotatebox[origin=c]{90}{\textbf{OASIS-3}}}
& TriangCFM 
& $0.59 \pm 0.05$ & $\underline{0.19 \pm 0.06}$ & $\underline{0.26 \pm 0.07}$ & $0.25 \pm 0.05$ & $0.03 \pm 0.03$ & $0.01 \pm 0.01$ & $\underline{1.75 \pm 0.83}$ \\
& DiffeoGauss 
& $\underline{0.61 \pm 0.06}$ & $0.15 \pm 0.04$ & $0.21 \pm 0.05$ & $0.34 \pm 0.05$ & $0.03 \pm 0.03$ & $0.01 \pm 0.01$ & $2.24 \pm 0.14$ \\
& DiffeoCFM 
& $0.60 \pm 0.06$ & $0.18 \pm 0.05$ & $0.26 \pm 0.06$ & $\underline{0.34 \pm 0.04}$ & $\underline{0.03 \pm 0.03}$ & $\underline{0.01 \pm 0.01}$ & $1.93 \pm 0.46$ \\
& GDSS-proj
& $0.00 \pm 0.00$ & $0.02 \pm 0.01$ & $0.00 \pm 0.00$
& $0.22 \pm 0.07$ & $0.11 \pm 0.09$ & $0.03 \pm 0.02$ & $4.00 \pm 3.04$ \\
& \cellcolor{LightBlue} \textbf{MAGNET (ours)} 
& \cellcolor{LightBlue} $\mathbf{0.65 \pm 0.14}$ & \cellcolor{LightBlue} $\mathbf{0.21 \pm 0.05}$ & \cellcolor{LightBlue} $\mathbf{0.28 \pm 0.05}$ & \cellcolor{LightBlue} $\mathbf{0.36 \pm 0.04}$ & \cellcolor{LightBlue} $\mathbf{0.02 \pm 0.02}$ & \cellcolor{LightBlue} $\mathbf{0.01 \pm 0.00}$ & \cellcolor{LightBlue} $\mathbf{1.61 \pm 0.97}$ \\

\bottomrule
\end{tabular}
\vspace{-2.0em}
\end{table*}

Under matched sampling budgets, MAGNET improves upon prior baselines, delivering up to a 21\% relative increase in class-conditional fidelity ($\alpha,\beta$-F1) over the strongest validity-preserving baselines (Tab.~\ref{tab:main_brain_results}). These gains are most pronounced on heterogeneous cohorts: $21\%$ on OASIS-3 and $7\%$ on ADNI. This suggests that MAGNET better captures the within-class clinical variation that confounds support recovery in these cohorts. Furthermore, this strict geometric fidelity directly translates to downstream utility: MAGNET achieves the highest or tied-best classifier-based F1 scores across all three cohorts. Crucially, MAGNET establishes this performance while simultaneously accelerating generation. Computational graph reparameterizations have similarly been explored to improve the efficiency of probabilistic deep learning \cite{nazarovs2021graph}. The primary speedup over flow-matching baselines arises from the graph-denoiser architecture operating in discrete diffusion steps rather than continuous ODE integration; the amortized structural bridge provides an additional approximately $25-32\times$ reduction relative to teacher-bridge alternatives (Tab.~\ref{tab:bridge_ablations}). 

Fig.~\ref{fig:connectome_summary} further confirms these gains topologically by visualizing the spatial mismatch of cortical hubs relative to real test scans. While baselines exhibit brighter surface regions indicative of larger spatial deviations, MAGNET generates noticeably darker surfaces that closely resemble the real distributions, yielding the lowest hub errors across all three cohorts. Finally, we note that generic discrete graph models (e.g., GDSS) fail to achieve this level of fidelity. Because these models are natively designed for sparse networks with discrete edge attributes, they cannot inherently enforce the strict positive-definite and unit-diagonal constraints of correlation matrices. Consequently, their reliance on post-hoc projection to force outputs into valid dense, continuous functional connectomes severely disrupts the generative trajectory and destroys subtle structural patterns, rendering them fundamentally insufficient for conditional support recovery (see App.~\ref{appsec:baselines} for detailed analysis).
\vspace{-1em}
\subsection{Zero-Shot Recovery of Fine-Grained Clinical Subclasses}
\label{subsec:zero_shot}
\vspace{-0.5em}
To test if binary-trained generators natively preserve unseen fine-grained clinical structure, we map generated samples into these subclasses using a classifier trained exclusively on real data. Because no generator is exposed to these labels during training, this zero-shot setting isolates the preservation of classifier-resolved clinical organization rather than validating biological taxonomy. This confound is controlled across methods by applying the same real-data-trained resolver to every generator.

\textbf{Results.} As detailed in Tab.~\ref{tab:zero_shot_metrics}, MAGNET achieves the strongest subclass covariance alignment in this zero-shot setting, reducing $\lambda$-MSE on ADNI by $33-57\%$ relative to baselines. Also, TriangCFM (a continuous euclidean baseline) over-indexes on precision at the severe expense of coverage. MAGNET also attains the best macro $\beta$-recall, $\alpha,\beta$-F1, and CAS macro-F1 on ADNI, while sweeping all support-recovery, CAS, and $\lambda$-MSE metrics on OASIS-3. Across both datasets, MAGNET preserves subclass balance with the lowest KL divergence and tied Bhattacharyya distance. Distributional alignment has also been studied through optimal transport regularization \cite{mehta2023efficient}. Results indicate MAGNET better preserves fine-grained clinical organization under coarse binary supervision. Crucially, after generated samples are assigned to the resolved subclass space, $\lambda$-MSE directly compares subclass covariance spectra rather than classifier accuracy, providing complementary evidence of subclass preservation. In stark contrast, the generic discrete graph baseline (GDSS-proj) suffers catastrophic degradation in this fine-grained setting (e.g., dropping to zero F1 on OASIS-3), further emphasizing that post-hoc projection severely disrupts the subtle topological variations necessary for resolving unseen clinical subclasses (see App.~\ref{appsec:baselines} for details).

\vspace{-1em}
\subsection{Graph Backbone and Controller Necessity}
\vspace{-0.5em}
\label{subsec:backbone_controller_ablation}

\begin{table*}[!t]
\centering
\scriptsize
\setlength{\tabcolsep}{3.0pt}
\renewcommand{\arraystretch}{1.15}
\caption{\footnotesize
\textbf{Backbone and controller ablation.}
All variants use the same normalized-Cholesky validity chart and a 6-step sampling budget. Controller columns denote class-anchored conditioning (Subtype), amortized bridge conditioning (Bridge), relevance-preserving noise (Noise), and diversity-preserving classifier-free guidance (Guidance). The vector row tests whether the validity chart alone is sufficient with a flat denoiser. The graph row restores the region-level graph denoiser but disables the controller. The individual controller rows activate one mechanism at a time on top of the graph denoiser. The penultimate row uses the integrated controller without diversity preserving classifier-free guidance, and the final row is full MAGNET.  
}
\label{tab:backbone_controller_ablation}
\begin{tabular}{lccccccccc}
\toprule
\multirow{2}{*}{\textbf{Dataset}}
& \multirow{2}{*}{\textbf{Backbone}}
& \multicolumn{4}{c}{\textbf{Controller}}
& \multicolumn{3}{c}{\textbf{Quality}}
& \multirow{2}{*}{\textbf{Sampling (s.)} $\downarrow$} \\
\cmidrule(lr){3-6}
\cmidrule(lr){7-9}
& & \textbf{Subtype} & \textbf{Bridge} & \textbf{Noise} & \textbf{Guidance}
& \textbf{$\alpha$-Prec.} $\uparrow$
& \textbf{$\beta$-Rec.} $\uparrow$
& \textbf{$\alpha,\beta$-F1} $\uparrow$
& \\
\midrule

\multirow{7}{*}{\textbf{ABIDE}}
& Vector & \xmark & \xmark & \xmark & \xmark
& $0.52 \pm 0.05$ & $0.18 \pm 0.03$ & $0.26 \pm 0.03$ & $\mathbf{0.04 \pm 0.00}$ \\
& Graph & \xmark & \xmark & \xmark & \xmark
& $0.70 \pm 0.05$ & $0.46 \pm 0.06$ & $0.56 \pm 0.05$ & $0.36 \pm 0.01$ \\
& Graph & \cmark & \xmark & \xmark & \xmark
& $0.72 \pm 0.05$ & $0.50 \pm 0.06$ & $0.59 \pm 0.05$ & $0.39 \pm 0.01$ \\
& Graph & \xmark & \cmark & \xmark & \xmark
& $0.73 \pm 0.06$ & $0.48 \pm 0.06$ & $0.58 \pm 0.06$ & $0.38 \pm 0.01$ \\
& Graph & \xmark & \xmark & \cmark & \xmark
& $0.71 \pm 0.05$ & $0.50 \pm 0.04$ & $0.59 \pm 0.03$ & $0.38 \pm 0.00$ \\
& Graph & \cmark & \cmark & \cmark & \xmark
& $0.77 \pm 0.06$ & $0.54 \pm 0.05$ & $0.63 \pm 0.04$ & $0.38 \pm 0.06$ \\
& \cellcolor{LightBlue} Graph & \cellcolor{LightBlue}\cmark & \cellcolor{LightBlue}\cmark & \cellcolor{LightBlue}\cmark & \cellcolor{LightBlue}\cmark
& \cellcolor{LightBlue}$\mathbf{0.80 \pm 0.04}$ & \cellcolor{LightBlue}$\mathbf{0.59 \pm 0.04}$ & \cellcolor{LightBlue}$\mathbf{0.68 \pm 0.03}$ &
\cellcolor{LightBlue}$0.38 \pm 0.02$ \\
\midrule

\multirow{7}{*}{\textbf{ADNI}}
& Vector & \xmark & \xmark & \xmark & \xmark
& $0.52 \pm 0.01$ & $0.25 \pm 0.04$ & $0.34 \pm 0.04$ & $\mathbf{0.07 \pm 0.00}$ \\
& Graph & \xmark & \xmark & \xmark & \xmark
& $0.61 \pm 0.08$ & $0.54 \pm 0.08$ & $0.57 \pm 0.06$ & $0.72 \pm 0.01$ \\
& Graph & \cmark & \xmark & \xmark & \xmark
& $0.63 \pm 0.06$ & $0.72 \pm 0.10$ & $0.67 \pm 0.06$ & $0.74 \pm 0.04$ \\
& Graph & \xmark & \cmark & \xmark & \xmark
& $0.62 \pm 0.14$ & $0.78 \pm 0.07$ & $0.69 \pm 0.11$ & $0.76 \pm 0.02$ \\
& Graph & \xmark & \xmark & \cmark & \xmark
& $0.64 \pm 0.10$ & $0.76 \pm 0.08$ & $0.69 \pm 0.07$ & $0.75 \pm 0.02$ \\
& Graph & \cmark & \cmark & \cmark & \xmark
& $0.68 \pm 0.12$ & $0.80 \pm 0.04$ & $0.74 \pm 0.09$ & $0.72 \pm 0.01$ \\
& \cellcolor{LightBlue} Graph & \cellcolor{LightBlue}\cmark & \cellcolor{LightBlue}\cmark & \cellcolor{LightBlue}\cmark & \cellcolor{LightBlue}\cmark
& \cellcolor{LightBlue}$\mathbf{0.71 \pm 0.13}$ & \cellcolor{LightBlue}$\mathbf{0.84 \pm 0.06}$ & \cellcolor{LightBlue}$\mathbf{0.76 \pm 0.06}$ &
\cellcolor{LightBlue}$0.86 \pm 0.28$ \\
\midrule

\multirow{7}{*}{\textbf{OASIS-3}}
& Vector & \xmark & \xmark & \xmark & \xmark
& $0.58 \pm 0.09$ & $0.26 \pm 0.04$ & $0.36 \pm 0.05$ & $\mathbf{0.05 \pm 0.00}$ \\
& Graph & \xmark & \xmark & \xmark & \xmark
& $0.67 \pm 0.08$ & $0.44 \pm 0.07$ & $0.54 \pm 0.08$ & $0.60 \pm 0.01$ \\
& Graph & \cmark & \xmark & \xmark & \xmark
& $0.68 \pm 0.08$ & $0.44\pm 0.07$ & $0.54 \pm 0.06$ & $0.61 \pm 0.02$ \\
& Graph & \xmark & \cmark & \xmark & \xmark
& $0.70 \pm 0.06$ & $0.46 \pm 0.06$ & $0.55 \pm 0.06$ & $0.60 \pm 0.02$ \\
& Graph & \xmark & \xmark & \cmark & \xmark
& $0.70 \pm 0.08$ & $0.47 \pm 0.05$ & $0.56 \pm 0.06$ & $0.61 \pm 0.02$ \\
& Graph & \cmark & \cmark & \cmark & \xmark
& $0.72 \pm 0.07$ & $0.46 \pm 0.05$ & $0.56 \pm 0.06$ & $0.59 \pm 0.00$ \\
& \cellcolor{LightBlue} Graph & \cellcolor{LightBlue}\cmark & \cellcolor{LightBlue}\cmark & \cellcolor{LightBlue}\cmark & \cellcolor{LightBlue}\cmark
& \cellcolor{LightBlue}$\mathbf{0.75 \pm 0.07}$ & \cellcolor{LightBlue}$\mathbf{0.47 \pm 0.06}$ & \cellcolor{LightBlue}$\mathbf{0.57 \pm 0.06}$ &
\cellcolor{LightBlue}$0.63 \pm 0.00$ \\
\bottomrule
\vspace{-2em}
\end{tabular}
\end{table*}

MAGNET relies on two core pillars: a validity-preserving chart for mathematical feasibility~\cite{collas2025riemannian} and a graph-structured denoiser for relational inductive biases~\cite{dadi2019benchmarking, cui2022braingb, mohammadi2024graph}. Tab.~\ref{tab:backbone_controller_ablation} isolates these components under a fixed six-step generation budget~\cite{song2020denoising} to evaluate backbone necessity and controller synergy in resolving coarse-label deficits.

\textbf{The structural gains of the graph backbone.} Switching from the vector denoiser to the graph backbone yields substantial gains across all datasets. While the pullback chart guarantees mathematical validity, it lacks a mechanism for modeling ROI relationships; the graph denoiser supplies this crucial inductive bias, bridging mathematical soundness with anatomical relevance~\cite{dadi2019benchmarking, cui2022braingb, mohammadi2024graph}.

\textbf{Coupled controller dynamics.} The control components: subtype conditioning, amortized bridge, relevance-preserving noise, and classifier-free guidance~\cite{ho2022classifier} act as interdependent latent transforms rather than strictly additive modules. Individually, the amortized bridge provides the largest improvement, particularly on the highly diverse ADNI dataset ($\alpha,\beta$-F1 improves from 0.57 to 0.69). 

\textbf{Generalization across atlas resolutions.} To test whether MAGNET's gains depend on the compact parcellation, we repeat the ABIDE benchmark on two higher-resolution atlases: AAL116, an anatomically defined 116-region parcellation, and CC200, a functionally derived parcellation with approximately 200 regions. Both experiments use a common quality-controlled ABIDE cohort of 831 scans with CPAC preprocessing, band-pass filtering, and no global signal regression, together with the same 10-split group-aware evaluation protocol. These settings increase the dimensionality of the generation problem while varying the construction of the atlas. MAGNET remains the strongest generator at higher atlas resolutions, achieving the best $\alpha,\beta$-F1 on AAL116 and CC200, with relative gains of $17\%$ and $14\%$ over DiffeoCFM, respectively (Tab.~\ref{tab:rebuttal_high_res}).

\begin{table}[!t]
\centering
\scriptsize
\setlength{\tabcolsep}{3pt}
\renewcommand{\arraystretch}{1.15}
\caption{\footnotesize
\textbf{Generalization to higher-resolution ABIDE parcellations.} We evaluate MAGNET and the competing generators on AAL116, an anatomically defined 116-region atlas, and CC200, a functionally derived parcellation with approximately 200 regions. Both settings use the same quality-controlled 831-scan ABIDE cohort and 10-split group-aware evaluation protocol. We report class-conditional support fidelity, downstream classifier-based utility, and wall-clock sampling time. Results are mean $\pm$ standard deviation across splits. The Real Data row provides the empirical reference, and bold denotes the best generative model within each atlas.
}
\label{tab:rebuttal_high_res}
\begin{tabular}{llcccccc}
\toprule
\textbf{Atlas} & \textbf{Method}
& \textbf{$\alpha$-Prec.$\uparrow$}
& \textbf{$\beta$-Rec.$\uparrow$}
& \textbf{$\alpha,\beta$-F1$\uparrow$}
& \textbf{ROC-AUC$\uparrow$}
& \textbf{F1$\uparrow$}
& \textbf{Time (s)$\downarrow$} \\
\midrule
\multirow{7}{*}{\textbf{AAL116}}
& \cellcolor{LightGray} Real Data       & \cellcolor{LightGray} 0.78 $\pm$ 0.06 & \cellcolor{LightGray} 0.77 $\pm$ 0.09 & \cellcolor{LightGray} 0.77 $\pm$ 0.05 & \cellcolor{LightGray} 0.72 $\pm$ 0.09 & \cellcolor{LightGray} 0.63 $\pm$ 0.07 & \cellcolor{LightGray} N/A \\
& TriangCFM       & 0.34 $\pm$ 0.09 & 0.02 $\pm$ 0.01 & 0.04 $\pm$ 0.02 & 0.49 $\pm$ 0.05 & 0.42 $\pm$ 0.19 & 3.80 $\pm$ 0.21 \\
& TriangDDPM      & 0.41 $\pm$ 0.05 & 0.12 $\pm$ 0.07 & 0.37 $\pm$ 0.08 & 0.53 $\pm$ 0.08 & 0.48 $\pm$ 0.12 & 3.21 $\pm$ 0.32 \\
& DiffeoGauss     & 0.55 $\pm$ 0.05 & 0.35 $\pm$ 0.07 & 0.43 $\pm$ 0.05 & 0.53 $\pm$ 0.07 & 0.32 $\pm$ 0.12 & 1.94 $\pm$ 0.39 \\
& DiffeoCFM       & 0.58 $\pm$ 0.12 & 0.62 $\pm$ 0.08 & 0.59 $\pm$ 0.05 & 0.65 $\pm$ 0.08 & 0.54 $\pm$ 0.13 & 4.16 $\pm$ 0.24 \\
& GDSS-proj       & 0.05 $\pm$ 0.02 & 0.21 $\pm$ 0.01 & 0.08 $\pm$ 0.03 & 0.51 $\pm$ 0.03 & 0.34 $\pm$ 0.09 & 1.43 $\pm$ 0.12 \\
& \cellcolor{LightBlue} \textbf{MAGNET} & \cellcolor{LightBlue} \textbf{0.78 $\pm$ 0.05} & \cellcolor{LightBlue} \textbf{0.62 $\pm$ 0.08} & \cellcolor{LightBlue} \textbf{0.69 $\pm$ 0.05} & \cellcolor{LightBlue} \textbf{0.68 $\pm$ 0.06} & \cellcolor{LightBlue} \textbf{0.58 $\pm$ 0.07} & \cellcolor{LightBlue} \textbf{1.12 $\pm$ 0.07} \\
\midrule
\multirow{7}{*}{\textbf{CC200}}
& \cellcolor{LightGray} Real Data       & \cellcolor{LightGray} 0.79 $\pm$ 0.06 & \cellcolor{LightGray} 0.66 $\pm$ 0.06 & \cellcolor{LightGray} 0.71 $\pm$ 0.02 & \cellcolor{LightGray} 0.76 $\pm$ 0.02 & \cellcolor{LightGray} 0.68 $\pm$ 0.04 & \cellcolor{LightGray} N/A \\
& TriangCFM       & 0.39 $\pm$ 0.02 & 0.12 $\pm$ 0.05 & 0.18 $\pm$ 0.06 & 0.48 $\pm$ 0.03 & 0.47 $\pm$ 0.09 & 4.21 $\pm$ 0.13 \\
& TriangDDPM      & 0.43 $\pm$ 0.02 & 0.21 $\pm$ 0.12 & 0.28 $\pm$ 0.11 & 0.55 $\pm$ 0.02 & 0.52 $\pm$ 0.11 & 3.71 $\pm$ 0.16 \\
& DiffeoGauss     & 0.57 $\pm$ 0.03 & 0.41 $\pm$ 0.02 & 0.48 $\pm$ 0.02 & 0.58 $\pm$ 0.01 & 0.41 $\pm$ 0.08 & 1.73 $\pm$ 0.19 \\
& DiffeoCFM       & 0.65 $\pm$ 0.05 & 0.53 $\pm$ 0.03 & 0.58 $\pm$ 0.03 & 0.63 $\pm$ 0.04 & 0.54 $\pm$ 0.07 & 3.78 $\pm$ 0.13 \\
& GDSS-proj       & 0.08 $\pm$ 0.03 & 0.26 $\pm$ 0.11 & 0.12 $\pm$ 0.04 & 0.55 $\pm$ 0.06 & 0.41 $\pm$ 0.04 & 1.56 $\pm$ 0.07 \\
& \cellcolor{LightBlue} \textbf{MAGNET} & \cellcolor{LightBlue} \textbf{0.77 $\pm$ 0.09} & \cellcolor{LightBlue} \textbf{0.58 $\pm$ 0.09} & \cellcolor{LightBlue} \textbf{0.66 $\pm$ 0.06} & \cellcolor{LightBlue} \textbf{0.68 $\pm$ 0.05} & \cellcolor{LightBlue} \textbf{0.59 $\pm$ 0.05} & \cellcolor{LightBlue} \textbf{1.42 $\pm$ 0.05} \\
\bottomrule
\end{tabular}
\vspace{-2em}
\end{table}

\textbf{Biological plausibility and clinical structure.}
Beyond distributional fidelity, we evaluate neuroscientific plausibility through cortical hub preservation and network topology, established descriptors of biologically meaningful brain organization~\cite{van2013network,rubinov2010complex}. Fig.~\ref{fig:connectome_summary} measures hub preservation, while Figs.~\ref{fig:connectome_class0}-\ref{fig:connectome_class1} compare connectivity over fixed, real edge patterns. MAGNET further achieves the lowest subclass covariance error on ADNI and OASIS-3, including a $33$--$57\%$ reduction in $\lambda$-MSE on ADNI (Tab.~\ref{tab:zero_shot_metrics}), while its gains persist on the higher-resolution AAL116 and CC200 atlases (Tab.~\ref{tab:rebuttal_high_res}).

\vspace{-1.5em}
\section{Conclusion}\label{sec:conclusion}
\vspace{-1.5em}
Geometric validity is necessary for conditional connectome synthesis, but it is not sufficient for structural recovery. In this paper, we introduced MAGNET, a structure-aware generator that uses the normalized-Cholesky chart for validity and focuses the modeling problem on recovering class-conditional organization inside the valid correlation manifold. MAGNET combines an ROI-level graph denoiser with class-anchored subtype conditioning, an amortized structural bridge, and relevance-preserving corruption to reduce ambiguity from coarse labels, global structural context, and condition-relevant coordinates. Across three cohorts, MAGNET improves distributional fidelity, support recovery, zero-shot clinical alignment, and few-step sampling efficiency over strong baselines. MAGNET combines an ROI-level graph denoiser with class-anchored subtype conditioning, an amortized structural bridge, and relevance-preserving corruption to address ambiguity from coarse labels, missing global context, and corruption of condition-relevant coordinates. Future directions include multi-modal imaging inputs, longitudinal trajectory generation, and evaluation of synthetic cohort utility in downstream clinical prediction. More broadly, the amortized bridge principle of distilling structural summaries into fast inference-time predictors could be extended to other constrained generative settings beyond connectomes.

\noindent
\textbf{Acknowledgment}
Prof. Lokhande acknowledges support from University at Buffalo startup funds, an Adobe Research Gift, an NVIDIA Academic Grant, and the National Center for Advancing Translational Sciences of the NIH (award UM1TR005296 to the University at Buffalo). Prof. Kim's contribution to this project is supported by RS-2019-II191906 (Artificial Intelligence Graduate School, POSTECH).

\bibliographystyle{plainnat}
\bibliography{neurips_2026}

\newpage
\appendix


\section{Notations}\label{app:notations}
We summarize the main notation used.
\begin{center}
\renewcommand{\arraystretch}{1.3}
\begin{tabular}{ll}
\toprule
\textbf{Symbol} & \textbf{Meaning} \\
\midrule

$\chi(C)$ & Normalized-Cholesky chart coordinate of $C$ \\

$x_t$ & Noisy latent state at diffusion time $t$ \\

$s$ & Train-only subtype token within class $y$ \\

$b_t$ & Amortized bridge state at time $t$ \\

$\varphi(\cdot)$ & Structural feature extractor used by the bridge and spectral loss \\

$B_\psi(x_t,t,s)$ & Amortized bridge predictor \\

$b^\star(x_0,s,t)$ & Teacher bridge target \\

$L_{\mathrm{bridge}}$ & Bridge distillation loss \\

$\alpha,\beta\text{-F1}$ & Harmonic mean of $\alpha$-precision and $\beta$-recall \\

$\mathcal{C}_d$ & Set of $d \times d$ valid correlation connectomes \\

$C \in \mathcal{C}_d$ & A correlation connectome \\

$y \in \{0,1\}$ & Class label \\

$P_{\mathrm{val}}^y$ & Held-out real distribution for class $y$ \\

$Q_m^y$ & Generated distribution for method $m$, class $y$ \\

$\Phi(C) = \log C$ & Log-Euclidean map to $\mathrm{Sym}(d)$ \\

$d_{\mathrm{LE}}(C_1, C_2)$ & Log-Euclidean distance (Eq.~\ref{eq:le-distance}) \\

$\mu^{\log,\mathrm{val}}_y$ & Mean of $\log C$ under $P_{\mathrm{val}}^y$ \\

$\bar{C}_{\mathrm{LE},y}^{\mathrm{val}}$ & Log-Euclidean Fréchet mean \\

$r_y(C)$ & Radial distance from real mean (Eq.~\ref{eq:radial-coordinate}) \\

$\nu_y^{\mathrm{val}}, \nu_y^m$ & Radial pushforward distributions \\

$W_{\mathrm{rad}}^{m,y}$ & Radial Wasserstein distance (Eq.~\ref{eq:radial-w1}) \\

$E_{\mathrm{rad}}^{m,y}$ & Radial quantile calibration error (Eq.~\ref{eq:radial-logerr}) \\

$z_i^{\mathrm{val},y}$ & Log-centered vectorized connectome \\

$W_y$ & PCA basis (top 2 directions from real data) \\

$u_y(C)$ & 2D chart coordinate (Eq.~\ref{eq:tangent-chart}) \\

$\mathrm{ISR}_{m,y}$ & Inside-support rate (Eq.~\ref{eq:isr}) \\

$C_{m,t}^y$ & Generated sample at reverse-time $t$ \\

$\mathrm{ISR}_{m,y}(t)$ & Time-dependent inside-support rate (Eq.~\ref{eq:isr-time}) \\

$\bar{C}_{m,t}^y$ & Generated mean connectome at time $t$ \\

$\bar{C}_{\mathrm{val}}^y$ & Real mean connectome \\

$M_y$ & Real-defined edge mask \\

$\mathrm{MAE}_{m,t}^y$ & Mean absolute error on mask (Eq.~\ref{eq:masked-mae}) \\

$\mathrm{Corr}_{m,t}^y$ & Correlation on masked edges (Eq.~\ref{eq:masked-corr}) \\

\bottomrule
\end{tabular}
\end{center}

\section{Proofs of Formal Results}
\label{app:proofs}

This appendix collects the proofs of the formal results stated in Section~\ref{sec:methodology}. The main text uses these statements to justify the design of class-anchored conditioning, the amortized structural bridge, and relevance-preserving corruption while keeping the method section visually compact.

\subsection{Proof of Proposition~\ref{prop:anchored_risk}}
\label{app:proof_anchored}

The result follows from the $L^2$-projection property of conditional expectation which is a generalization of the Rao-Blackwell theorem to the squared-loss denoising setting. We include the derivation for completeness. The key non-trivial empirical content is that the gain is \emph{strictly positive} whenever $S$ encodes structure that $Y$ collapses. Fig.~\ref{fig:intro_2d} also confirms visually for all three cohorts.

\begin{proof}
For squared loss, the Bayes estimators are the conditional means: $f^\star(X_t,Y) = \mathbb{E}[X_0 \mid X_t, Y]$ and 
$g^\star(X_t,Y,S) = \mathbb{E}[X_0 \mid X_t, Y, S]$. We decompose the coarser residual as
\begin{align*}
X_0 - \mathbb{E}[X_0 \mid X_t,Y]
= \underbrace{X_0 - \mathbb{E}[X_0 \mid X_t,Y,S]}_{A}
+ \underbrace{\mathbb{E}[X_0 \mid X_t,Y,S] - \mathbb{E}[X_0 \mid X_t,Y]}_{B}.
\end{align*}
Since $A$ is the residual from the finer predictor, $\mathbb{E}[A \mid X_t,Y,S]=0$. Because $B$ is  $\sigma(X_t,Y,S)$-measurable, the cross term $\mathbb{E}[\langle A,B\rangle] = 0$. Taking squared norms and expectations yields Eq.~\eqref{eq:anchored_identity} directly. Strict positivity holds whenever $B \not\equiv 0$ with positive probability, i.e., whenever $S$ resolves within-class structure unresolved by $Y$.
\end{proof}

\textbf{Remark (empirical content).} The identity converts the entire visual motivation of Fig.~\ref{fig:intro_2d} into a concrete optimality criterion. The gain is zero iff the subtype provides no additional information beyond the coarse label. The heterogeneity entropy $\mathbf{H}$ analysis in App.~\ref{appsubsec:conditioning} confirms this is non-zero and scales with dataset complexity.

\subsection{Proof of Proposition~\ref{thm:bridge_risk}}
\label{app:proof_bridge}

\begin{proof}
The risk improvement identity Eq.~\eqref{eq:bridge_help} follows by the same $L^2$-projection argument as Prop.~\ref{prop:anchored_risk}, with $(U)$ in place of $(X_t,Y)$ and $(U,B^\star)$ in place of $(X_t,Y,S)$; we omit the repetition. The result is the amortization bound of Eq.~\eqref{eq:bridge_bound}, which is proved below.
\begin{align*}
f^\star(U) &= \mathbb E[X_0\mid U],\\
g^\star(U,B^\star) &= \mathbb E[X_0\mid U,B^\star].
\end{align*}
Substituting these into Eq.~\ref{eq:bridge_risks} we get
\begin{align*}
\mathcal R_U^\star
&=
\mathbb E\!\left[\|X_0-\mathbb E[X_0\mid U]\|_2^2\right],\\
\mathcal R_{U,B^\star}^\star
&=
\mathbb E\!\left[\|X_0-\mathbb E[X_0\mid U,B^\star]\|_2^2\right].
\end{align*}

Now we decompose the coarser residual as
\begin{align*}
X_0-\mathbb E[X_0\mid U]
&=
\underbrace{X_0-\mathbb E[X_0\mid U,B^\star]}_{A}
+
\underbrace{\mathbb E[X_0\mid U,B^\star]-\mathbb E[X_0\mid U]}_{B}.
\end{align*}
Now, upon squaring we get
\begin{align*}
\|X_0-\mathbb E[X_0\mid U]\|_2^2
=
\|A\|_2^2+\|B\|_2^2+2\langle A,B\rangle.
\end{align*}
We now condition on $(U,B^\star)$. Since $A = X_0-\mathbb E[X_0\mid U,B^\star]$, we have
\begin{align*}
\mathbb E[A\mid U,B^\star]=0.
\end{align*}
Because $B$ is measurable with respect to $(U,B^\star)$, we get
\begin{align*}
\mathbb E[\langle A,B\rangle \mid U,B^\star]
=
\left\langle \mathbb E[A\mid U,B^\star],\,B \right\rangle
=
0.
\end{align*}
Now, we take the conditional expectation followed by an expectation again, we get
\begin{align*}
\mathcal R_U^\star
=
\mathcal R_{U,B^\star}^\star
+
\mathbb E\!\left[
\left\|
\mathbb E[X_0\mid U,B^\star]-\mathbb E[X_0\mid U]
\right\|_2^2
\right],
\end{align*}
which proves Eq.~\ref{eq:bridge_help}.

We now prove the amortization bound in Eq.~\ref{eq:bridge_bound}. Let $\eta(u,b)=\mathbb E[X_0\mid U=u,B^\star=b]$ and we define the plug-in predictor $\widehat X_0(U)=\eta(U,B_\psi(U))$. Inserting and subtracting the bridge-aware Bayes predictor, we get:
\begin{align*}
X_0-\widehat X_0(U)
=
\underbrace{X_0-\eta(U,B^\star)}_{A}
+
\underbrace{\eta(U,B^\star)-\eta(U,B_\psi(U))}_{B}.
\end{align*}
Now, we take the square and the expectation to yield
\begin{align*}
\mathbb E\!\left[\|X_0-\widehat X_0(U)\|_2^2\right]
&=
\mathbb E\!\left[\|A\|_2^2\right]
+
\mathbb E\!\left[\|B\|_2^2\right]
+
2\,\mathbb E\!\left[\langle A,B\rangle\right].
\end{align*}
Further, we condition on $(U,B^\star,B_\psi(U))$. Since $\eta(U,B^\star)=\mathbb E[X_0\mid U,B^\star]$, the residual $A$ has zero conditional mean given $(U,B^\star)$, and therefore also given $(U,B^\star,B_\psi(U))$. Hence
\begin{align*}
\mathbb E[\langle A,B\rangle \mid U,B^\star,B_\psi(U)]
=
\left\langle
\mathbb E[A\mid U,B^\star,B_\psi(U)],
\,B
\right\rangle
=
0.
\end{align*}
Thus the cross term vanishes, and the first term is exactly $\mathcal R_{U,B^\star}^\star$, so
\begin{align*}
\mathbb E\!\left[\|X_0-\widehat X_0(U)\|_2^2\right]-\mathcal R_{U,B^\star}^\star
=
\mathbb E\!\left[\|\eta(U,B^\star)-\eta(U,B_\psi(U))\|_2^2\right].
\end{align*}
By the assumed Lipschitz property of $\eta$ in its bridge argument, we get
\begin{align*}
\|\eta(U,B^\star)-\eta(U,B_\psi(U))\|_2
\le
L\,\|B^\star-B_\psi(U)\|_2.
\end{align*}
Again, upon squaring and taking the expectation we get,
\begin{align*}
\mathbb E\!\left[\|\eta(U,B^\star)-\eta(U,B_\psi(U))\|_2^2\right]
\le
L^2\,\mathbb E\!\left[\|B_\psi(U)-B^\star\|_2^2\right].
\end{align*}
By finally combining the last two equations, we prove Eq.~\ref{eq:bridge_bound}.
\end{proof}

\textbf{Lipschitz condition on the Bayes regressor.} The $L$-Lipschitz condition on $\eta(\cdot,b)=\mathbb{E}[X_0\mid U,B^\star=b]$ is a regularity assumption on the underlying data distribution: it requires that the population-level conditional expectation varies smoothly with the bridge input. This is a standard smoothness assumption in nonparametric estimation~\cite{miyasawa1961empirical}. Separately, to confirm that the error bound in Eq.~\eqref{eq:bridge_bound} remains finite in practice, we note that the deployed bridge predictor $B_\psi$, parameterized as a $D$-layer MLP with weight matrices $\{W_d\}$ and $\sigma$-Lipschitz activations, satisfies $\mathrm{Lip}(B_\psi) \le \sigma^{D-1}\prod_{d=1}^D\|W_d\|_{\mathrm{op}}$, which is finite under $\ell_2$ weight regularization~\cite{bartlett2017spectrally,miyato2018spectral}. This confirms numerical tractability but does not substitute for the distributional regularity assumption on $\eta$. Because $b^\star$ is a convex combination of atoms from a compact set $\mathcal A_s$ with diameter $\delta_s = \max_{j,k}\|a_{s,j}-a_{s,k}\|_2 < \infty$, the bridge input is bounded: $\|b^\star\|_2 \le \max_k \|a_{s,k}\|_2$. Because $b^\star$ is a convex combination of atoms from the finite set $\mathcal A_s$, the bridge input is restricted to the compact domain $\operatorname{conv}(\mathcal A_s)$ and is therefore bounded. Compactness alone does not imply Lipschitz continuity of the Bayes regressor; accordingly, Eq.~\eqref{eq:bridge_bound} explicitly assumes that $\eta(u,b)=\mathbb{E}[X_0\mid U=u,B^\star=b]$ is $L$-Lipschitz in $b$ on the relevant bridge domain. A sufficient condition is that $\eta$ be continuously differentiable with respect to $b$ on a neighborhood of this domain with uniformly bounded Jacobian. The bounded bridge domain therefore controls where this regularity assumption is required, rather than implying the assumption itself. This is a distributional property and cannot be inferred from the Lipschitz constant of the deployed approximator $B_\psi$. 

Although our amortized bridge is trained against a teacher target \textit{(drawing directly on the principles of knowledge distillation~\cite{hinton2015distilling})}, our structural objective is entirely different. In traditional distillation, the capacity is shrunk in the same output space, while we use the structure for addressing a lack of context, to provide a macroscopic overview of the conditional probability distribution which cannot be captured by an isolated state $X_t$ via few-step sampling~\cite{ho2022classifier,chen2022analog}. Equation~\ref{eq:bridge_bound} represents this difference by constraining the amortization tracking error rather than shrinking the capacity. As shown in Appendix~\ref{appsubsec:bridge}, tracking the teacher target effectively leads to increased generative accuracy, allowing us to maintain structural properties of the teacher bridge with up to approximately $25-32\times$ savings in sampling cost.

\textbf{Note.} i) Prop. \ref{prop:adaptive_surrogate} and \ref{prop:anchored_risk} are \emph{design-justification} results, not end-to-end generalization guarantees with the purpose of showing the information-theoretic direction of MAGNET: adding a subtype variable or a structural bridge can only reduce Bayes denoising risk, and cannot hurt in the oracle limit. This rules out the possibility that the architecture is fundamentally misspecified even under ideal conditions, which is a nontrivial baseline to establish before empirical validation. ii) For the learned regime, \textit{we deliberately rely on empirical evidence}: Tab.~\ref{tab:bridge_ablations} shows that lower amortization tracking error directly correlates
with higher $\alpha,\beta$-F1, empirically validating the amortization bound (Eq.~\ref{eq:bridge_bound}) in the regime where it matters. Appendix~\ref{appsubsec:conditioning} shows that the class-anchored conditioning gain scales with dataset heterogeneity
entropy $\mathbf{H}$, empirically validating that the learned subtypes capture within-class structure rather than being a regularization artifact. App.~\ref{app:proof_bridge} discusses why finiteness of $L$ is guaranteed on the relevant domain by compactness of the atom set $\mathcal{A}_s$, and we clarify there that \textit{this is a distributional regularity condition rather than a claim about the deployed network}.

\subsection{A surrogate result for relevance-preserving corruption}
\label{app:proof_corruption}

The adaptive corruption rule in Section~\ref{subsec:adaptive_corruption} is data-driven: the relevant directions are estimated from the conditional--unconditional guidance difference introduced in Section~\ref{subsec:adaptive_corruption}. To formalize the underlying design principle, we analyze a stylized Gaussian surrogate.

\begin{lemma}[Adaptive corruption preserves class separation under a Gaussian surrogate]
\label{prop:adaptive_surrogate}
Let the clean latent decompose as $X_0=(U,V)$, where $U\in\mathbb R^r$ is the class-relevant subspace and $V\in\mathbb R^q$ is the nuisance subspace. Assume that conditioned on $Y=y$, $U\mid Y=y \sim \mathcal N(\mu_y,\Sigma_U)$ and $V\mid Y=y \sim \mathcal N(0,\Sigma_V)$, with $\Sigma_U \succ 0$, $\Sigma_V \succ 0$, and $U \perp V \mid Y$. For fixed diffusion time $t$, consider the anisotropic corruption model
\begin{equation}
X_t
=
\sqrt{\bar\alpha_t}\,X_0
+
\sqrt{1-\bar\alpha_t}\,\xi,
\qquad
\xi
\sim
\mathcal N\!\left(
0,
\begin{bmatrix}
\sigma_{\mathrm{rel}}^2 I_r & 0\\
0 & \sigma_{\mathrm{nui}}^2 I_q
\end{bmatrix}
\right).
\label{eq:block_corruption_appendix}
\end{equation}
Let $\Delta=\mu_1-\mu_0$, and define the class-separation functional
\begin{equation}
D_t(\sigma_{\mathrm{rel}}^2)
=
(\mu_{t,1}-\mu_{t,0})^\top \Sigma_t^{-1}(\mu_{t,1}-\mu_{t,0}),
\label{eq:Dt_appendix}
\end{equation}
where $\mu_{t,y}$ and $\Sigma_t$ are the class-conditional mean and covariance of $X_t$. Then
\begin{equation}
D_t(\sigma_{\mathrm{rel}}^2)
=
\bar\alpha_t\,
\Delta^\top
\left(
\bar\alpha_t\Sigma_U+(1-\bar\alpha_t)\sigma_{\mathrm{rel}}^2 I_r
\right)^{-1}
\Delta,
\label{eq:Dt_closed_appendix}
\end{equation}
which is strictly decreasing in $\sigma_{\mathrm{rel}}^2$ whenever $\Delta\neq 0$.
\end{lemma}

\begin{proof}
Under Eq.~\ref{eq:block_corruption_appendix}, the conditional distribution of $X_t$ given $Y=y$ is Gaussian with mean
\begin{align*}
\mu_{t,y}
=
\sqrt{\bar\alpha_t}
\begin{bmatrix}
\mu_y\\
0
\end{bmatrix}
\end{align*}
and covariance
\begin{align*}
\Sigma_t
=
\begin{bmatrix}
\bar\alpha_t\Sigma_U+(1-\bar\alpha_t)\sigma_{\mathrm{rel}}^2 I_r & 0\\
0 & \bar\alpha_t\Sigma_V+(1-\bar\alpha_t)\sigma_{\mathrm{nui}}^2 I_q
\end{bmatrix}.
\end{align*}
Since only the relevant block carries class-dependent means,
\begin{align*}
\mu_{t,1}-\mu_{t,0}
=
\sqrt{\bar\alpha_t}
\begin{bmatrix}
\Delta\\
0
\end{bmatrix}.
\end{align*}
Substituting these expressions into Eq.~\ref{eq:Dt_appendix} yields Eq.~\ref{eq:Dt_closed_appendix}.

To prove monotonicity, write $A=\bar\alpha_t\Sigma_U$, $s=\sigma_{\mathrm{rel}}^2$, and $\beta=1-\bar\alpha_t$. Then
\begin{align*}
D_t(s)=\bar\alpha_t\,\Delta^\top(A+\beta s I_r)^{-1}\Delta.
\end{align*}
Differentiate with respect to $s$. Using the matrix derivative identity
\begin{align*}
\frac{d}{ds}(A+\beta s I_r)^{-1}
=
-\beta(A+\beta s I_r)^{-2},
\end{align*}
we obtain
\begin{align*}
\frac{d}{ds}D_t(s)
=
-\bar\alpha_t\beta\,
\Delta^\top(A+\beta s I_r)^{-2}\Delta.
\end{align*}
Because $A+\beta s I_r$ is positive definite, $(A+\beta s I_r)^{-2}$ is positive definite as well. Therefore the quadratic form above is strictly positive whenever $\Delta\neq 0$, and hence
\begin{align*}
\frac{d}{ds}D_t(s)<0
\qquad \text{for } \Delta\neq 0.
\end{align*}
This proves that $D_t(\sigma_{\mathrm{rel}}^2)$ is strictly decreasing in $\sigma_{\mathrm{rel}}^2$.
\end{proof}

\textbf{Remark.}
Lemma~\ref{prop:adaptive_surrogate} is a population-level surrogate statement that explains why protecting a sparse set of class-relevant directions is preferable to isotropic corruption when the reverse budget is small. In the full model, those directions are estimated from the guidance signal rather than assumed known.


\section{Baselines}
\label{appsec:baselines}

We compare MAGNET against baselines designed to isolate the roles of validity-preserving geometry, learned transport, and structure-aware conditional recovery. All baselines are evaluated on the same splits, preprocessing pipeline, and sampling budgets as MAGNET, so differences reflect model design rather than data handling or integration depth.

\textbf{DiffeoCFM.}
DiffeoCFM~\cite{collas2025riemannian} is the strongest prior validity-preserving baseline in our setting. It performs conditional flow matching in a global pullback chart of the correlation manifold induced by the normalized-Cholesky parameterization. This baseline tests whether validity-preserving transport alone is sufficient for conditional connectome generation.

\textbf{DiffeoGauss.}
DiffeoGauss is a structural ablation of DiffeoCFM in which the learned transport dynamics are replaced by a class-conditional Gaussian model fit directly in the same pullback chart. It therefore preserves the geometry of DiffeoCFM while removing learned flow dynamics. Comparing against DiffeoGauss isolates the contribution of learned transport beyond the validity-preserving coordinate system itself.

\textbf{TriangCFM.}
TriangCFM is a Euclidean flow-matching baseline defined directly on the strict lower-triangular coordinates, without a validity-preserving correlation-manifold chart. Generated samples must therefore be mapped back to the correlation manifold through post-hoc projection. This baseline isolates the importance of validity-preserving geometry by retaining continuous transport while removing the pullback construction.

\textbf{Real Data.}
The real-data row is not a generative model, but an empirical oracle computed by comparing disjoint real-data splits under the same evaluation pipeline. It quantifies the finite-sample ceiling induced by split-wise estimation and provides a reference point for interpreting generative fidelity and downstream utility.

\textbf{Why not generic graph generative baselines?} We do not include graph generative models such as GDSS~\cite{jo2022score} or DiGress~\cite{vignac2022digress} because they are designed for sparse graphs with discrete edge labels, while our task requires the generation of dense, continuous-valued functional connectomes. In functional connectomes, every ROI pair carries a real-valued dependency, and the output must satisfy symmetry and unit-diagonal structure. These requirements fall outside the design assumptions of such baselines. Adapting them to this setting would require a nontrivial redesign of their representation, denoising objective, and output parameterization, along with additional procedures to force a valid connectome structure. This would make the comparison hard to interpret, since any performance degradation could come either from the underlying model or from the adaptation stack needed to fit a fundamentally different data type. We therefore focus on baselines that are inherently designed for continuous connectome generation under comparable structural assumptions.

To make this limitation concrete, we also implemented an adaptation of GDSS. We trained the official GDSS adjacency score network on dense signed ROI adjacency matrices and applied post-hoc projection to force each generated output back to the correlation-matrix domain. We denote this adapted baseline as GDSS-proj. This is not a native connectome generator: the projection step is required because generic graph diffusion does not enforce positive definiteness or unit-diagonal correlation structure. The results in Tables~\ref{tab:main_brain_results}--\ref{tab:zero_shot_metrics} support the design choice behind MAGNET. Although GDSS-proj is fast, it substantially underperforms validity-aware connectome generators on support fidelity and zero-shot subclass recovery, indicating that generic graph diffusion plus post-hoc projection is insufficient for dense functional connectome synthesis.

\section{Experimental Setup}\label{subsec:exp_setup}

\textbf{Evaluation metrics.} We assess class-conditional support recovery using $\alpha$-precision (fidelity to held-out support) and $\beta$-recall (coverage of that support), summarizing the trade-off with their harmonic mean, $\alpha,\beta\text{-F1}$~\cite{alaa2022faithful}. Because generative fidelity alone does not guarantee the preservation of task-relevant structures, we also evaluate downstream utility by training a logistic regression classifier on Riemannian tangent-space features extracted from generated samples, reporting ROC-AUC and F1 on real validation connectomes.

\paragraph{Zero-shot subclass covariance metric.} To evaluate whether generated samples reproduce the dominant covariance spectrum of resolved clinical subclasses (e.g., in ADNI~\cite{weiner2017alzheimer} or OASIS-3~\cite{lamontagne2019oasis}), we report $\lambda$-MSE. For each subclass $\ell\in\mathcal{L}$, we collect held-out real connectomes $\mathcal{R}_\ell$ and generated connectomes $\mathcal{G}_\ell$ assigned to $\ell$ by the real-data-trained resolver. We project each matrix to a valid correlation manifold and vectorize its strictly upper-triangular entries via $\psi(C)=\operatorname{vech}_0(\Pi_{\mathcal{C}_d}(C))\in\mathbb{R}^{p}$, where $p=d(d-1)/2$. After mean-centering these features to form $X_\ell$ and $\widehat{X}_\ell$, we compute their top covariance eigenvalues from singular values $s_{\ell,k}$ and $\widehat{s}_{\ell,k}$ as $\lambda^{\mathrm{real}}_{\ell,k}=s_{\ell,k}^2/(n_\ell-1)$ and $\lambda^{\mathrm{gen}}_{\ell,k}=\widehat{s}_{\ell,k}^2/(\widehat{n}_\ell-1)$. For $K_\ell=\min(5,n_\ell-1,\widehat{n}_\ell-1,p)$, the metric is defined as $\lambda\text{-MSE}=|\mathcal{L}_{\mathrm{valid}}|^{-1}\sum_{\ell\in\mathcal{L}_{\mathrm{valid}}}K_\ell^{-1}\sum_{k=1}^{K_\ell}(\lambda^{\mathrm{real}}_{\ell,k}-\lambda^{\mathrm{gen}}_{\ell,k})^2$, where $\mathcal{L}_{\mathrm{valid}}$ contains subclasses with $\geq 5$ real and generated samples. Unlike CAS macro-F1, this strictly isolates the alignment of subclass-specific covariance spread rather than evaluating classifier accuracy.

\textbf{Training and sampling protocol.} All methods are evaluated across $10$ disjoint train/validation splits (10\% validation). Experiments were conducted on a single NVIDIA A6000 GPU. Wall-clock sampling times in Table~\ref{tab:main_brain_results} are measured as the mean time to generate one full split's worth of samples, averaged over 10 splits on the same hardware.  For ADNI and OASIS-3, train/test disjointness is strictly enforced at the subject level~\cite{collas2025riemannian}. MAGNET utilizes balanced class sampling, inverse subject-frequency weighting on repeated-subject cohorts, and min-SNR loss reweighting. To isolate model design rather than integration depth, we fix the computational budget: MAGNET samples via a $6$-step DDIM chain~\cite{song2020denoising}, while DiffeoCFM and its ablations use matched $6$-point ODE grids.


\section{Architectural Details}  \label{app:archdetails}

MAGNET couples the validity-preserving normalized-Cholesky chart with a region-aware denoising backbone. In the MSDL setting, each connectome has $39$ regions, giving $741$ chart coordinates. Rather than processing the noisy latent state $x_t$ as a flat vector, the denoiser first lifts it into a symmetric ROI-level matrix $A(x_t)$. The chart coordinates are written into the strict lower triangle, reflected across the diagonal, and the diagonal is set to zero. The $i$th row, denoted $a_i(x_t)$, is then interpreted as the current noisy
connectivity profile of region $i$.

Each region is represented as one token. The initial token combines the row-wise connectivity profile with anatomical identity and global conditioning: $h_i^{(0)}=W_{\mathrm{in}}a_i(x_t)+p_i+c_t.$ Here, $W_{\mathrm{in}}$ is a learned row projection, $p_i$ is a learned ROI positional embedding, and $c_t$ is a shared conditioning broadcast to all regions. The conditioning vector combines the diffusion timestep embedding, the class-anchored subtype embedding, and the amortized structural bridge embedding. In the implementation, the bridge embedding is passed through a learned gate before being injected, allowing the denoiser to modulate how strongly the global structural summary affects each reverse step.

The token sequence is processed by a pre-norm Transformer encoder with hidden size $128$, four layers, four QK-normalized attention heads, GeGLU feed-forward blocks of width $512$, and dropout $0.1$. Because attention is applied across all ROI tokens, each region updates its representation using information from the full current connectome. This gives the denoiser a relational inductive bias without hard-coding an anatomical edge graph. The learned positional embeddings preserve ROI identity, so the model remains tied to the atlas-defined ordering of brain regions.

After the transformer stack, MAGNET maps the final token matrix $H^{(L)}$ back to chart coordinates. A row-wise output projection produces a square matrix $R$, which is symmetrized and vectorized over the strict lower triangle: $E=\frac{1}{2}(R+R^\top),\ \widehat v_\theta=\operatorname{vech}_0(E).$ Thus, the denoiser predicts a diffusion target in the same normalized-Cholesky chart used by the sampler. The chart itself, rather than the denoiser output alone, is what guarantees that decoded samples are valid correlation matrices.

The same graph denoiser is used for conditional and unconditional predictions during classifier-free guidance. Let $\widehat v_c$ and $\widehat v_\varnothing$ denote the conditional and unconditional denoiser outputs, and let $g=\widehat v_c-\widehat v_\varnothing$ be the guidance direction. MAGNET retains the component of the unconditional prediction that is
orthogonal to $g$: $\widehat v_\varnothing^\perp=\widehat v_\varnothing-\frac{\langle \widehat v_\varnothing,g\rangle}{\|g\|_2^2+\varepsilon}g.$
The guided prediction is then $\widehat v_\gamma=\widehat v_\varnothing^\perp
+\gamma\,\tanh(\|g\|_2)\,g.$ The $\tanh(\|g\|_2)$ factor limits over-steering when the conditional and unconditional predictions are close, while the orthogonal residual preserves variation not aligned with the class-discriminative direction.

Finally, MAGNET applies class-wise coordinate-affine calibration to the final sampled latent before decoding. This calibration is fitted on the training split by matching class-conditional latent means and variances between uncalibrated generated samples and real training latents. The calibrated latent is clipped and then decoded through the normalized-Cholesky inverse map, so validity is still enforced by the chart-based decoder.


\begin{table*}[!t]
\centering
\scriptsize
\caption{\footnotesize\textbf{Bridge ablations across ABIDE, ADNI, and OASIS-3.} We compare six structural-conditioning variants: no bridge, self-conditioned bridge, teacher bridge at inference, amortized bridge without $L_{\mathrm{bridge}}$, amortized bridge with $L_{\mathrm{bridge}}$, and the full MAGNET objective. We report $\alpha,\beta$-F1, ROC-AUC, F1, bridge distillation error where applicable, and sampling time. Lower bridge error tracks stronger structural fidelity, while removing $L_{\mathrm{bridge}}$ causes severe bridge-tracking failure. Gray rows denote the real-data reference. Bold highlights the best generative model for each metric.}
\setlength{\tabcolsep}{5pt}
\renewcommand{\arraystretch}{1.12}
\begin{tabular}{l l c c c c c}
\toprule
\multirow{2}{*}{\textbf{Dataset}} & \multirow{2}{*}{\textbf{Method}} 
& \multicolumn{1}{c}{\textbf{Quality}} 
& \multicolumn{2}{c}{\textbf{CAS Metrics}} 
& \multirow{2}{*}{\begin{tabular}[c]{@{}c@{}}\textbf{Bridge Distillation}\\ \textbf{Error $\downarrow$}\end{tabular}} 
& \multicolumn{1}{c}{\textbf{Time (s.)}} \\
\cmidrule(lr){3-3} \cmidrule(lr){4-5} \cmidrule(l){7-7}
& & $\alpha,\beta$-\textbf{F1} $\uparrow$ & \textbf{ROC-AUC} $\uparrow$ & \textbf{F1} $\uparrow$ & & \textbf{Sampling} $\downarrow$ \\
\midrule

\multirow{7}{*}{\textbf{ABIDE}}
& \cellcolor{LightGray} Real Data 
& \cellcolor{LightGray} 0.80 $\pm$ 0.01 & \cellcolor{LightGray} 0.66 $\pm$ 0.04 & \cellcolor{LightGray} 0.58 $\pm$ 0.06 & \cellcolor{LightGray} N/A & \cellcolor{LightGray} N/A \\
& No bridge 
& 0.58 $\pm$ 0.05 & 0.66 $\pm$ 0.06 & 0.57 $\pm$ 0.09 & - & \textbf{0.10 $\pm$ 0.00} \\
& Self-conditioning 
& 0.51 $\pm$ 0.06 & 0.66 $\pm$ 0.06 & 0.56 $\pm$ 0.11 & - & 3.01 $\pm$ 0.01 \\
& Teacher bridge 
& 0.64 $\pm$ 0.06 & 0.66 $\pm$ 0.06 & 0.57 $\pm$ 0.09 & - & 3.00 $\pm$ 0.01 \\
& Amortized bridge ($L_{bridge}$) 
& 0.65 $\pm$ 0.05 & 0.66 $\pm$ 0.05 & 0.57 $\pm$ 0.08 & 0.48 $\pm$ 0.04 & 0.12 $\pm$ 0.00 \\
& Amortized bridge w/o bridge loss 
& 0.59 $\pm$ 0.07 & 0.66 $\pm$ 0.06 & 0.56 $\pm$ 0.09 & 14.81 $\pm$ 0.53 & 0.13 $\pm$ 0.00 \\
& \cellcolor{LightBlue} Proposed ($L_{bridge} + L_{spec}$) 
& \cellcolor{LightBlue} \textbf{0.68 $\pm$ 0.04} & \cellcolor{LightBlue} \textbf{0.67 $\pm$ 0.06} & \cellcolor{LightBlue} \textbf{0.58 $\pm$ 0.09} & \cellcolor{LightBlue} \textbf{0.46 $\pm$ 0.04} & \cellcolor{LightBlue} 0.12 $\pm$ 0.00 \\
\midrule

\multirow{7}{*}{\textbf{ADNI}}
& \cellcolor{LightGray} Real Data 
& \cellcolor{LightGray} 0.87 $\pm$ 0.05 & \cellcolor{LightGray} 0.58 $\pm$ 0.03 & \cellcolor{LightGray} 0.59 $\pm$ 0.03 & \cellcolor{LightGray} N/A & \cellcolor{LightGray} N/A \\
& No bridge 
& 0.59 $\pm$ 0.12 & 0.58 $\pm$ 0.06 & 0.66 $\pm$ 0.05 & - & \textbf{0.22 $\pm$ 0.00} \\
& Self-conditioning 
& 0.41 $\pm$ 0.15 & 0.58 $\pm$ 0.04 & 0.55 $\pm$ 0.22 & - & 8.99 $\pm$ 0.04 \\
& Teacher bridge 
& 0.67 $\pm$ 0.11 & 0.58 $\pm$ 0.05 & 0.65 $\pm$ 0.05 & - & 8.09 $\pm$ 0.08 \\
& Amortized bridge ($L_{bridge}$) 
& 0.68 $\pm$ 0.12 & 0.58 $\pm$ 0.06 & 0.65 $\pm$ 0.06 & 0.05 $\pm$ 0.00 & 0.25 $\pm$ 0.00 \\
& Amortized bridge w/o bridge loss 
& 0.48 $\pm$ 0.11 & 0.58 $\pm$ 0.05 & \textbf{0.67 $\pm$ 0.05} & 14.67 $\pm$ 1.16 & 0.25 $\pm$ 0.00 \\
& \cellcolor{LightBlue} Proposed ($L_{bridge} + L_{spec}$)
& \cellcolor{LightBlue} \textbf{0.74 $\pm$ 0.08} & \cellcolor{LightBlue} \textbf{0.58 $\pm$ 0.07} & \cellcolor{LightBlue} \textbf{0.67 $\pm$ 0.06} & \cellcolor{LightBlue} \textbf{0.05 $\pm$ 0.00} & \cellcolor{LightBlue} 0.25 $\pm$ 0.00 \\
\midrule

\multirow{7}{*}{\textbf{OASIS-3}}
& \cellcolor{LightGray} Real Data 
& \cellcolor{LightGray} 0.88 $\pm$ 0.01 & \cellcolor{LightGray} 0.71 $\pm$ 0.04 & \cellcolor{LightGray} 0.61 $\pm$ 0.06 & \cellcolor{LightGray} N/A & \cellcolor{LightGray} N/A \\
& No bridge 
& 0.49 $\pm$ 0.07 & 0.69 $\pm$ 0.05 & 0.58 $\pm$ 0.05 & - & \textbf{0.18 $\pm$ 0.00} \\
& Self-conditioning 
& 0.46 $\pm$ 0.06 & 0.69 $\pm$ 0.04 & 0.57 $\pm$ 0.04 & - & 6.38 $\pm$ 0.05 \\
& Teacher bridge 
& 0.54 $\pm$ 0.06 & 0.69 $\pm$ 0.04 & \textbf{0.58 $\pm$ 0.04} & - & 6.28 $\pm$ 0.02 \\
& Amortized bridge ($L_{bridge}$)
& 0.55 $\pm$ 0.05 & 0.69 $\pm$ 0.05 & 0.57 $\pm$ 0.05 & 0.09 $\pm$ 0.00 & 0.20 $\pm$ 0.00 \\
& Amortized bridge w/o bridge loss 
& 0.49 $\pm$ 0.05 & \textbf{0.69 $\pm$ 0.04} & 0.57 $\pm$ 0.05 & 14.90 $\pm$ 1.06 & 0.21 $\pm$ 0.00 \\
& \cellcolor{LightBlue} Proposed ($L_{bridge} + L_{spec}$) 
& \cellcolor{LightBlue} \textbf{0.57 $\pm$ 0.05} & \cellcolor{LightBlue} 0.68 $\pm$ 0.04 & \cellcolor{LightBlue} \textbf{0.58 $\pm$ 0.05} & \cellcolor{LightBlue} \textbf{0.08 $\pm$ 0.00} & \cellcolor{LightBlue} 0.20 $\pm$ 0.00 \\
\bottomrule
\vspace{-3em}
\end{tabular}
\label{tab:bridge_ablations}
\end{table*}

\section{Framework Analysis}\label{app:frameworkanalysis}

\subsection{Amortized Structural Bridge}
\label{appsubsec:bridge}

To evaluate the role of the amortized structural bridge, we compare six variants: no bridge, self-conditioned bridge, teacher bridge at inference, amortized bridge without $L_{\mathrm{bridge}}$, amortized bridge with $L_{\mathrm{bridge}}$, and the full MAGNET objective. For bridge-enabled variants, we report the Bridge Distillation Error $\|B_\psi(x_t,t,s)-b^\star(x_0,s,t)\|_2^2$, which measures how well the amortized bridge tracks the time-dependent teacher target.

Table~\ref{tab:bridge_ablations} shows a clear bridge-quality versus structural-fidelity pattern. Lower distillation error is consistently associated with stronger $\alpha,\beta$-F1, while removing $L_{\mathrm{bridge}}$ causes a severe bridge-tracking failure without yielding a corresponding fidelity gain. This is the regime the bridge is designed for, and the ablation matches the Remark of proposition~\ref{thm:bridge_risk}: the amortized bridge is useful to the extent that it tracks the teacher target well.

The efficiency gap is equally important. Teacher-bridge and self-conditioned alternatives are substantially slower because they require a second structural pass at inference, whereas the amortized bridge preserves the few-step sampling regime. In practice, the bridge-enabled MAGNET variants retain roughly $30 -60\times$ lower sampling cost than the teacher and self-conditioned alternatives while delivering the strongest overall fidelity profile. Taken together, these results justify the amortized bridge as the primary structural-conditioning mechanism in MAGNET. It provides the gains of explicit structural guidance without sacrificing the speed advantages of few-step generation.

\subsection{Relevance-Preserving Corruption Ablation}\label{appsubsec:corruption}

\begin{table*}[!t]
\centering
\scriptsize
\caption{
\footnotesize
\textbf{Relevance-preserving corruption ablation across ABIDE, ADNI, and OASIS-3.}
We compare standard isotropic diffusion against three temporal variations of our adaptive protection schedule (Low, Adaptive, and High density decay). Performance remains stable across all adaptive regimes, demonstrating that structural recovery is invariant to the exact decay curve provided early-stage topological protection is active.
}
\label{tab:corruption_ablation}

\begin{tabular}{llccccc}
\toprule
\multirow{2}{*}{\textbf{Dataset}} 
& \multirow{2}{*}{\textbf{Configuration}} 
& \multirow{2}{*}{\textbf{Decay ($p$)}} 
& \textbf{Quality} 
& \multicolumn{2}{c}{\textbf{CAS Metrics}} 
& \multirow{2}{*}{\textbf{Sampling Time (s) $\downarrow$}} \\
\cmidrule(lr){4-4} 
\cmidrule(lr){5-6}
& & 
& \textbf{$\alpha,\beta$-F1 $\uparrow$} 
& \textbf{ROC-AUC $\uparrow$} 
& \textbf{F1 $\uparrow$} 
& \\

\midrule
\multirow{4}{*}{\textbf{ABIDE}}
& Isotropic (Baseline) & N/A & $0.62 \pm 0.03$ & $0.60 \pm 0.07$ & $0.50 \pm 0.09$ & $0.36 \pm 0.00$ \\
& Low Density & 1 & $0.68 \pm 0.03$ & $0.64 \pm 0.07$ & $0.55 \pm 0.06$ & $0.44 \pm 0.12$ \\
& \cellcolor{LightBlue} \textbf{Adaptive (Ours)} & \cellcolor{LightBlue} \textbf{3} & \cellcolor{LightBlue} $\mathbf{0.68 \pm 0.03}$ & \cellcolor{LightBlue} $\mathbf{0.64 \pm 0.07}$ & \cellcolor{LightBlue} $\mathbf{0.56 \pm 0.06}$ & \cellcolor{LightBlue} $\mathbf{0.38 \pm 0.02}$ \\
& High Density & 5 & $0.69 \pm 0.03$ & $0.64 \pm 0.07$ & $0.55 \pm 0.05$ & $0.39 \pm 0.03$ \\

\midrule
\multirow{4}{*}{\textbf{ADNI}}
& Isotropic (Baseline) & N/A & $0.69 \pm 0.09$ & $0.52 \pm 0.07$ & $0.55 \pm 0.05$ & $0.73 \pm 0.01$ \\
& Low Density & 1 & $0.75 \pm 0.07$ & $0.57 \pm 0.06$ & $0.61 \pm 0.06$ & $0.76 \pm 0.02$ \\
& \cellcolor{LightBlue} \textbf{Adaptive (Ours)} & \cellcolor{LightBlue} \textbf{3} & \cellcolor{LightBlue} $\mathbf{0.76 \pm 0.06}$ & \cellcolor{LightBlue} $\mathbf{0.57 \pm 0.07}$ & \cellcolor{LightBlue} $\mathbf{0.61 \pm 0.05}$ & \cellcolor{LightBlue} $\mathbf{0.86 \pm 0.28}$ \\
& High Density & 5 & $0.76 \pm 0.06$ & $0.56 \pm 0.06$ & $0.61 \pm 0.05$ & $0.76 \pm 0.02$ \\

\midrule
\multirow{4}{*}{\textbf{OASIS-3}}
& Isotropic (Baseline) & N/A & $0.53 \pm 0.06$ & $0.64 \pm 0.04$ & $0.49 \pm 0.05$ & $0.60 \pm 0.00$ \\
& Low Density & 1 & $0.58 \pm 0.06$ & $0.68 \pm 0.03$ & $0.55 \pm 0.04$ & $0.63 \pm 0.01$ \\
& \cellcolor{LightBlue} \textbf{Adaptive (Ours)} & \cellcolor{LightBlue} \textbf{3} & \cellcolor{LightBlue} $\mathbf{0.59 \pm 0.06}$ & \cellcolor{LightBlue} $\mathbf{0.68 \pm 0.03}$ & \cellcolor{LightBlue} $\mathbf{0.56 \pm 0.04}$ & \cellcolor{LightBlue} $\mathbf{0.63 \pm 0.00}$ \\
& High Density & 5 & $0.57 \pm 0.06$ & $0.68 \pm 0.03$ & $0.56 \pm 0.04$ & $0.63 \pm 0.01$ \\

\bottomrule
\end{tabular}
\vspace{-1.5em}
\end{table*}

Our relevance-preserving corruption relies on a monotonically decreasing protection schedule, $\rho(t)$, to shield condition-defining coordinates during early forward diffusion. To evaluate the framework's sensitivity to this mechanism, we parameterize the decay curve as $\rho(t) = (1 - \tau_t)^p$, where $\tau_t \in [0,1]$ is the normalized continuous diffusion time (i.e., $\tau_t = t/T$). The hyperparameter $p$ strictly controls the temporal density of the protection, dictating how quickly structural coordinates yield to Gaussian noise, rather than the total magnitude of protection applied.

We ablate this schedule across four distinct regimes: \textbf{(i) Isotropic (Baseline):} No adaptive manifold preservation; \textbf{(ii) Low Density ($p = 1$):} Gradual, linear relaxation of constraints; \textbf{(iii) Adaptive ($p = 3$, Ours):} Our primary, balanced decay schedule; and \textbf{(iv) High Density ($p = 5$):} Aggressive, early release of structural constraints.

Table~\ref{tab:corruption_ablation} highlights a clear trend across ABIDE, ADNI, and OASIS-3: MAGNET is robust to the exact shape of the decay curve. Generative fidelity ($\alpha,\beta$-F1) and downstream utility (CAS) remain consistently stable across the Low, Adaptive, and High density regimes. The adaptive schedule provides consistent improvement over isotropic diffusion across all three cohorts. Crucially, the robustness to the exact decay shape ($p=1,3,5$) confirms that the mechanism is not sensitive to hyperparameter tuning — any protection schedule that shields condition-relevant coordinates during the high-noise phase suffices. The primary contribution of relevance-preserving corruption is providing a cleaner latent trajectory for the bridge and guidance components to operate on, as evidenced by the full-system results in Table~\ref{tab:backbone_controller_ablation}.

\subsection{Classifier-Free Guidance Scale}
\label{subsec:cfgscale}

\begin{figure*}[!t]
    \centering
    \includegraphics[width=\linewidth]{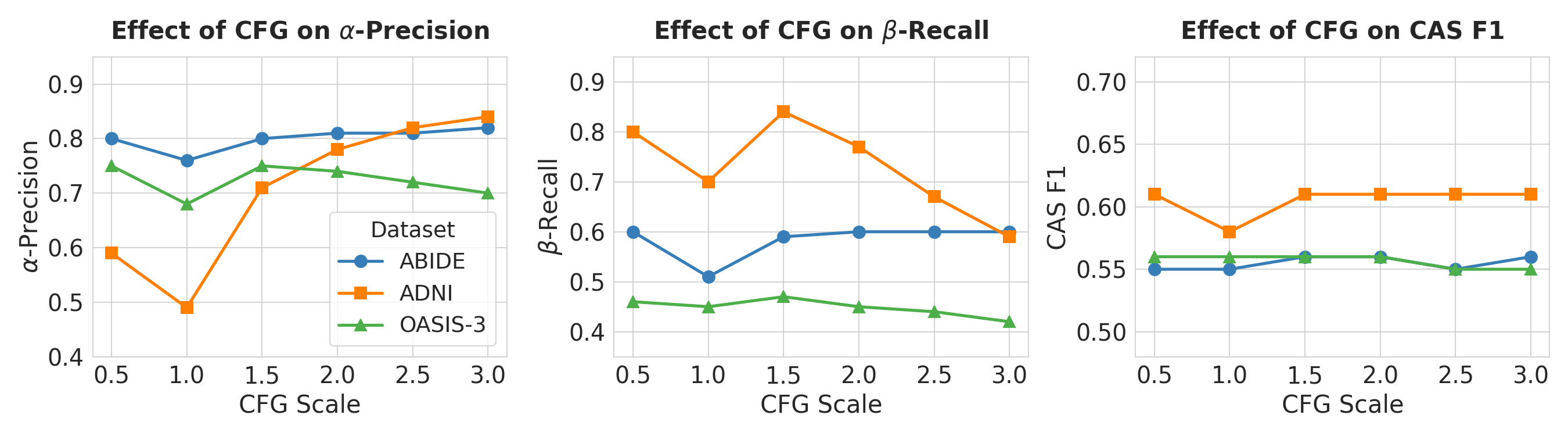} 
    \caption{\footnotesize\textbf{Classifier-free guidance scale ablation for MAGNET.} We vary the guidance scale $\gamma$ from $0.5$ to $3.0$ and plot class-conditional fidelity ($\alpha$-Precision, $\beta$-Recall) together with downstream utility (CAS F1) on ABIDE, ADNI, and OASIS-3. Across datasets, performance is strongest in the moderate-guidance regime, with $\gamma=1.5$ giving the most stable fidelity--utility trade-off. Larger guidance scales tend to increase precision at the expense of recall, consistent with overspecialization to a narrower region of the class manifold.}
    \label{fig:cfg_ablation}
    \vspace{-2em}
\end{figure*}

The classifier-free guidance scale $\gamma$ controls how strongly reverse-time sampling is steered toward the conditional branch. In MAGNET, guidance is not implemented as plain linear extrapolation. Instead, the sampler scales the conditional steering direction while retaining an orthogonal residual from the unconditional prediction. Figure~\ref{fig:cfg_ablation} shows that performance is strongest in the moderate-guidance regime.

On ADNI, $\gamma=1.5$ gives the best overall fidelity--utility trade-off, improving clearly over both weak guidance ($\gamma=0.5$) and overly strong guidance ($\gamma\geq 2.5$). The main failure mode at high guidance is not precision but coverage: $\alpha$-precision increases, but $\beta$-recall drops sharply, indicating overspecialization to a narrower subset of the class manifold. OASIS-3 follows the same pattern, though less sharply, with $\gamma=1.5$ again near-optimal. ABIDE is less sensitive once guidance is enabled, consistent with its simpler binary structure.

A useful implementation detail is that $\gamma=1.0$ in Fig.~\ref{fig:cfg_ablation} is implemented as a discrete branch skip (the unconditional forward pass is omitted entirely), rather than as a continuous evaluation of the guidance formula at that scale. Consequently, the transition from $\gamma=1.0$ to $\gamma=0.5$ represents a qualitative mode change, not a smooth interpolation, which accounts for the performance gap observed at that boundary. This explains both its lower runtime and its weaker fidelity relative to the moderately guided regime. We therefore use $\gamma=1.5$ throughout the paper.

\subsection{Class-Anchored Conditioning}\label{appsubsec:conditioning}

To evaluate the necessity of Class-Anchored Conditioning, we quantify using the complexity of the datasets. The heterogeneity entropy $\mathbf{H}$ of a dataset is the Shannon Entropy of the discovered structural subtypes $S$ conditioned on coarse labels $C$ as $\mathbf{H} = - \sum_{c \in \mathcal{C}} P(c) \sum_{k=1}^{K_c} P(s_k | c) \log P(s_k | c)$ where $s_k$ represents the latent sub-manifolds within a given class $c$. As $\mathbf{H}$ increases, the variance in a single label expands, leading to denoising ambiguity.

\begin{wrapfigure}{r}{0.55\textwidth}
    \centering
    \vspace{-1.5em}
    \includegraphics[width=\linewidth]{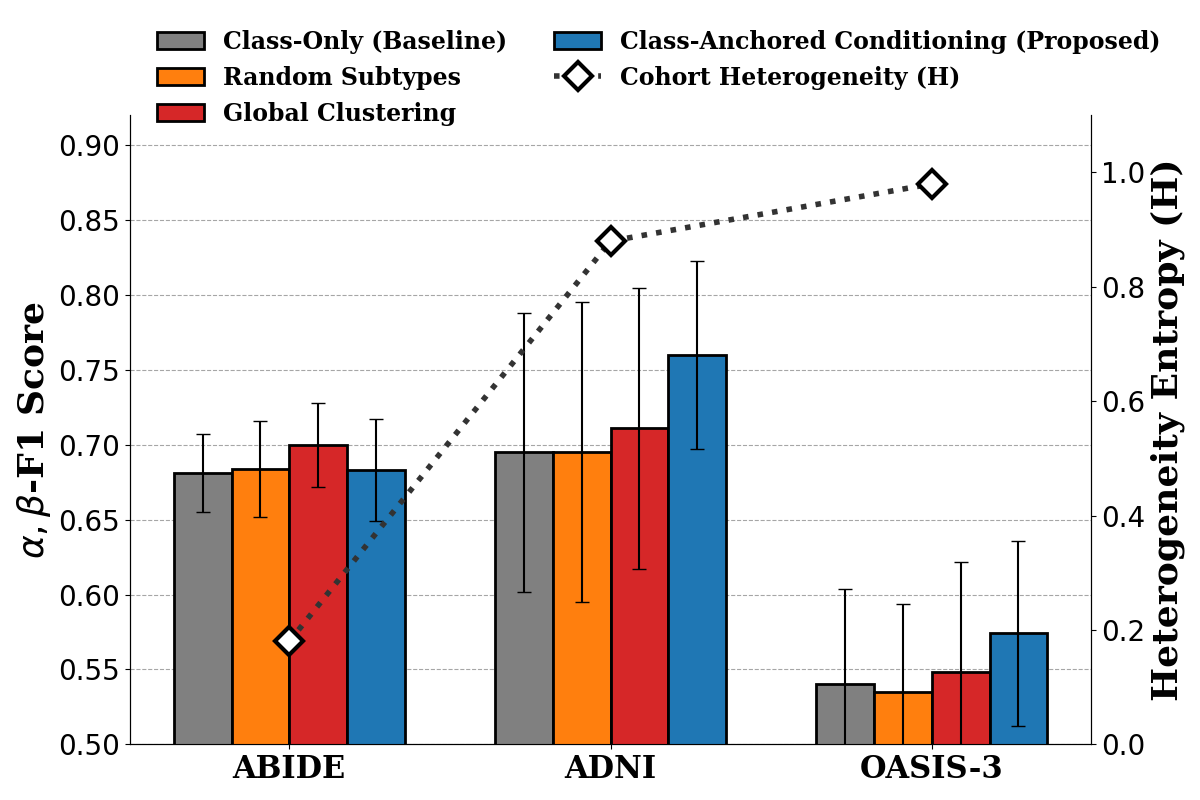}
    \caption{\footnotesize\textbf{Efficacy of Class-Anchored Conditioning} against class-only, random-subtype, and global clustering baselines. The performance gain of Class-Anchored Conditioning strictly scales with structural heterogeneity $\mathbf{H}$.}
    \label{fig:cac_ablation}
    \vspace{-1em}
\end{wrapfigure} 

We ablate against three paradigms (Fig.~\ref{fig:cac_ablation}) to isolate performance gains. \textbf{Class-only} utilizes the class coarse labels, \textbf{Random Subtypes} reassigns tokens randomly within a class to measure if gains are a regularizing artifact of the noise injection, and \textbf{Global Clustering} to discover $M$ latent structural clusters across the dataset. On highly heterogeneous cohort like ADNI and OASIS-3, the global clustering fails to differentiate between structurally similar but clinically distinct states. Our proposed Class-Anchored Conditioning enforces a stricter hierarchy to find $K$ distinct clusters only within the bounds of a specific clinical label ensuring the diffusion process is guided by latent substructures, not violating the broader class condition. The performance gain in generative fidelity confirms high-quality connectome generation in complex clinical datasets requires sub-manifold conditioning that remains strictly anchored to the coarse class constraint.

\subsection{Self-conditioning probability}
\label{subsec:selfconditioningprobability}

We also tested the fallback self-conditioning probability $p_{\mathrm{sc}}$, which in earlier non-amortized variants controls how often a first-pass denoiser estimate is reused as auxiliary structural context during training. In the final MAGNET configuration, however, structural context is supplied primarily by the amortized bridge, so this fallback mechanism is no longer central to the active model.

Empirically, varying $p_{\mathrm{sc}}$ from $0.2$ to $0.8$ produces no meaningful change in generative fidelity or downstream performance across ABIDE, ADNI, and OASIS-3: all three datasets remain effectively invariant at the precision reported in the main benchmark, with only minor variation in runtime. We therefore treat $p_{\mathrm{sc}}$ as a secondary implementation parameter and use $p_{\mathrm{sc}}=0.5$ throughout the paper. This confirms that, once the amortized bridge is active, structural recovery is largely insensitive to the scheduling of the fallback self-conditioning branch.

\subsection{Training and Sampling Efficiency}
\label{appsubsec:time_complexity}

MAGNET trades additional training cost for substantially faster repeated sampling. On a single NVIDIA A6000 GPU, training requires $532.08 \pm 16.66$~s per split on ABIDE, $1593.63 \pm 27.99$~s on ADNI, and $1242.53 \pm 18.49$~s on OASIS-3, corresponding to approximately $9$--$27$ minutes per trained split. This is roughly an order of magnitude more training time than the DiffeoCFM and TriangCFM baselines under the same hardware setting.

The inference regime is substantially different. MAGNET uses a six-step DDIM sampler and requires $0.38$~s, $0.86$~s, and $0.63$~s per generated split on ABIDE, ADNI, and OASIS-3, respectively, compared with $9.04$~s, $36.08$~s, and $23.24$~s for DiffeoCFM (Table~\ref{tab:main_brain_results}). This corresponds to approximately $24\times$, $42\times$, and $37\times$ faster sampling, respectively. The additional training cost is therefore incurred once when fitting a checkpoint, whereas the sampling advantage is retained whenever that checkpoint is subsequently used for synthetic cohort generation.

Importantly, this efficiency does not arise from sacrificing conditional fidelity: MAGNET simultaneously improves $\alpha,\beta$-F1 by $7$--$21\%$ over the strongest validity-preserving baseline across the primary cohorts and achieves the lowest zero-shot subclass covariance error in Table~\ref{tab:zero_shot_metrics}. These results characterize the computational trade-off of MAGNET as higher one-time training cost in exchange for improved support recovery and substantially lower repeated inference cost.

\subsection{Sensitivity to $\lambda_{\mathrm{bridge}}$ and $\lambda_{\mathrm{spec}}$}
\label{appsubsec:loss_ablation}

The full MAGNET objective combines the diffusion loss with bridge supervision and structural regularization,
$
L = L_{\mathrm{diff}}
+ \lambda_{\mathrm{bridge}}L_{\mathrm{bridge}}
+ \lambda_{\mathrm{spec}}L_{\mathrm{spec}}.
$
We therefore evaluate whether the gains of the proposed framework depend on narrowly tuned auxiliary-loss weights. For each loss, we vary its coefficient over
$\{0,0.025,0.05,0.1,0.2,0.4\}$ while holding the other coefficient fixed at its default value of $0.1$. The nonzero sweep spans a $16\times$ range around the default operating regime. All remaining training, sampling, and evaluation settings are unchanged.

Table~\ref{tab:loss_weight_sensitivity} reveals two complementary properties. First, both auxiliary objectives make a material contribution to support recovery. Removing the spectral term decreases $\alpha,\beta$-F1 from $0.68$ to $0.51$ on ABIDE, from $0.76$ to $0.68$ on ADNI, and from $0.57$ to $0.49$ on OASIS-3, corresponding to relative reductions of approximately $11$--$25\%$. Removing bridge supervision produces similarly consistent degradation, reducing $\alpha,\beta$-F1 to $0.53$, $0.65$, and $0.42$, respectively. These ablations are consistent with the intended roles of the two losses: $L_{\mathrm{bridge}}$ constrains the amortized predictor to track its structural target, while $L_{\mathrm{spec}}$ encourages the denoised latent to preserve global structural statistics.

Second, once either loss is enabled, performance is insensitive to its precise weighting. Across the nonzero $\lambda_{\mathrm{spec}}$ sweep, $\alpha,\beta$-F1 varies by at most $0.02$ on each cohort; across the nonzero $\lambda_{\mathrm{bridge}}$ sweep, the maximum variation is $0.03$. These changes are smaller than, or comparable to, the split-to-split variability reported in the main benchmark. ROC-AUC and downstream F1 exhibit the same qualitative stability. This behavior is also consistent with the construction of the objective: $L_{\mathrm{bridge}}$ is restricted to conditioned examples and weighted by $\bar{\alpha}_t$, while $L_{\mathrm{spec}}$ is mean-reduced over the minibatch.

Thus, the auxiliary objectives are important through their presence rather than through a narrowly selected coefficient. The default choice
$\lambda_{\mathrm{bridge}}=\lambda_{\mathrm{spec}}=0.1$
lies within a broad stable region across all three cohorts, indicating that MAGNET's performance is not dependent on fine-grained tuning of either loss weight.

\begin{table*}[!t]
\centering
\scriptsize
\caption{\footnotesize
\textbf{Sensitivity to bridge-supervision and spectral-regularization weights.}
We vary $\lambda_{\mathrm{spec}}$ and $\lambda_{\mathrm{bridge}}$ independently while fixing the other coefficient at its default value of $0.1$. All remaining training and evaluation settings are unchanged. Results are mean $\pm$ standard deviation across the same splits used in the main benchmark. Setting either coefficient to zero measures removal of the corresponding objective, while the nonzero values test sensitivity over a $16\times$ range.}
\label{tab:loss_weight_sensitivity}
\setlength{\tabcolsep}{3.8pt}
\renewcommand{\arraystretch}{1.12}
\begin{tabular}{lcccc|cccc}
\toprule
\multirow{2}{*}{Dataset}
& \multicolumn{4}{c|}{$\lambda_{\mathrm{spec}}$ sensitivity}
& \multicolumn{4}{c}{$\lambda_{\mathrm{bridge}}$ sensitivity} \\
\cmidrule(lr){2-5} \cmidrule(lr){6-9}
& $\lambda_{\mathrm{spec}}$ & $\alpha,\beta$-F1 $\uparrow$ & ROC-AUC $\uparrow$ & F1 $\uparrow$
& $\lambda_{\mathrm{bridge}}$ & $\alpha,\beta$-F1 $\uparrow$ & ROC-AUC $\uparrow$ & F1 $\uparrow$ \\
\midrule
\multirow{6}{*}{ABIDE}
& 0     & 0.51 $\pm$ 0.04 & 0.52 $\pm$ 0.06 & 0.43 $\pm$ 0.01 & 0     & 0.53 $\pm$ 0.03 & 0.56 $\pm$ 0.06 & 0.41 $\pm$ 0.09 \\
& 0.025 & 0.66 $\pm$ 0.06 & 0.62 $\pm$ 0.09 & 0.55 $\pm$ 0.03 & 0.025 & 0.66 $\pm$ 0.05 & 0.63 $\pm$ 0.05 & 0.55 $\pm$ 0.06 \\
& 0.05  & 0.67 $\pm$ 0.01 & 0.62 $\pm$ 0.06 & 0.55 $\pm$ 0.09 & 0.05  & 0.66 $\pm$ 0.03 & 0.64 $\pm$ 0.08 & 0.55 $\pm$ 0.05 \\
& 0.1 (default) & 0.68 $\pm$ 0.03 & 0.64 $\pm$ 0.07 & 0.56 $\pm$ 0.06 & 0.1 (default) & 0.68 $\pm$ 0.03 & 0.64 $\pm$ 0.07 & 0.56 $\pm$ 0.06 \\
& 0.2   & 0.67 $\pm$ 0.02 & 0.63 $\pm$ 0.02 & 0.56 $\pm$ 0.12 & 0.2   & 0.66 $\pm$ 0.03 & 0.64 $\pm$ 0.02 & 0.54 $\pm$ 0.01 \\
& 0.4   & 0.68 $\pm$ 0.07 & 0.64 $\pm$ 0.01 & 0.55 $\pm$ 0.07 & 0.4   & 0.66 $\pm$ 0.04 & 0.65 $\pm$ 0.01 & 0.55 $\pm$ 0.04 \\
\midrule
\multirow{6}{*}{ADNI}
& 0     & 0.68 $\pm$ 0.03 & 0.49 $\pm$ 0.01 & 0.53 $\pm$ 0.08 & 0     & 0.65 $\pm$ 0.06 & 0.45 $\pm$ 0.03 & 0.47 $\pm$ 0.07 \\
& 0.025 & 0.75 $\pm$ 0.09 & 0.57 $\pm$ 0.03 & 0.61 $\pm$ 0.09 & 0.025 & 0.73 $\pm$ 0.01 & 0.56 $\pm$ 0.01 & 0.61 $\pm$ 0.03 \\
& 0.05  & 0.75 $\pm$ 0.02 & 0.55 $\pm$ 0.08 & 0.61 $\pm$ 0.03 & 0.05  & 0.75 $\pm$ 0.09 & 0.57 $\pm$ 0.02 & 0.60 $\pm$ 0.02 \\
& 0.1 (default) & 0.76 $\pm$ 0.06 & 0.57 $\pm$ 0.07 & 0.61 $\pm$ 0.05 & 0.1 (default) & 0.76 $\pm$ 0.06 & 0.57 $\pm$ 0.07 & 0.61 $\pm$ 0.05 \\
& 0.2   & 0.75 $\pm$ 0.03 & 0.56 $\pm$ 0.01 & 0.60 $\pm$ 0.08 & 0.2   & 0.75 $\pm$ 0.04 & 0.57 $\pm$ 0.06 & 0.60 $\pm$ 0.09 \\
& 0.4   & 0.76 $\pm$ 0.09 & 0.56 $\pm$ 0.03 & 0.60 $\pm$ 0.01 & 0.4   & 0.74 $\pm$ 0.02 & 0.55 $\pm$ 0.04 & 0.59 $\pm$ 0.04 \\
\midrule
\multirow{6}{*}{OASIS-3}
& 0     & 0.49 $\pm$ 0.09 & 0.59 $\pm$ 0.07 & 0.47 $\pm$ 0.01 & 0     & 0.42 $\pm$ 0.01 & 0.53 $\pm$ 0.04 & 0.42 $\pm$ 0.06 \\
& 0.025 & 0.57 $\pm$ 0.08 & 0.67 $\pm$ 0.05 & 0.56 $\pm$ 0.09 & 0.025 & 0.55 $\pm$ 0.04 & 0.68 $\pm$ 0.08 & 0.56 $\pm$ 0.03 \\
& 0.05  & 0.56 $\pm$ 0.06 & 0.68 $\pm$ 0.08 & 0.55 $\pm$ 0.04 & 0.05  & 0.55 $\pm$ 0.09 & 0.65 $\pm$ 0.06 & 0.57 $\pm$ 0.01 \\
& 0.1 (default) & 0.57 $\pm$ 0.06 & 0.68 $\pm$ 0.03 & 0.56 $\pm$ 0.04 & 0.1 (default) & 0.57 $\pm$ 0.06 & 0.68 $\pm$ 0.03 & 0.56 $\pm$ 0.04 \\
& 0.2   & 0.57 $\pm$ 0.07 & 0.67 $\pm$ 0.01 & 0.55 $\pm$ 0.08 & 0.2   & 0.56 $\pm$ 0.03 & 0.67 $\pm$ 0.09 & 0.56 $\pm$ 0.09 \\
& 0.4   & 0.58 $\pm$ 0.01 & 0.65 $\pm$ 0.09 & 0.54 $\pm$ 0.07 & 0.4   & 0.57 $\pm$ 0.06 & 0.66 $\pm$ 0.01 & 0.55 $\pm$ 0.08 \\
\bottomrule
\end{tabular}
\end{table*}



\section{Geometric Analysis of Conditional Connectome Generation}\label{app:geometricanalysis}
In this section, we analyze whether the generated connectomes are in the same class-conditional geometric space through a series of experiments to compare against the held-out real validation data. In order to keep \textit{occupancy, overlap, and support recovery} non self-referential, we fix all the coordinate system, support contours, and local scales using the real data and ensure valid comparison between $Q_m^y$ and $P_{\val}^y$.

\subsection{Visualization Setup}

Let $\Cd=\left\{C\in\R^{d\times d}\;:\;C=C^\top,\;C\succ 0,\;\diagop(C)=\mathbf{1}\right\}$ denote the symmetric positive definite correlation matrix manifold on the MSDL atlas~\cite{varoquaux2011multi}. Let $P^y_{\val}$ denote the held-out real validation distribution and $Q^y_m$ denote the class-conditional distribution generated by method $m$ where $y\in\{0,1\}$ is the dataset specific class label. We compare $Q_m(\cdot\mid Y=y)$ against $P_{\val}(\cdot\mid Y=y)$ to make class-conditional comparisons. This is essential because \textit{a conditional generator can appear globally plausible while still failing to reproduce the support or geometry of a particular class}~\cite{benny2021evaluation,regenwetter2023beyond}.

\subsection{Geometry, Assumptions, and why this analysis}

\paragraph{Log-Euclidean representation.} Since $\mathcal{C_d}$ is not a vector space, euclidean operations in matrix space do not capture the geometry of connectomes faithfully. Standard operations on $C_1,C_2\in\mathcal{C_d}$ treat the objects as points in a flat vector space, ignoring the intrinsic geometric structure of the set. Hence, we adopt the log-euclidean representation~\cite{arsigny2006log, arsigny2007geometric} and define $\Phi:\mathcal{C_d}\to\mathrm{Sym}(d),\quad \Phi(C)=\log C$ where $\mathrm{Sym}(d)$ is the vector space of real symmetric matrices and $\log C$ is the matrix logarithm. The log-euclidean distance is 
\begin{equation}
d_{\LE}(C_1,C_2) = \|\Phi(C_1)-\Phi(C_2)\|_F.
\label{eq:le-distance}
\end{equation}

Moreover, every $C\in\mathcal{C}_d$ has strictly positive eigenvalues and $\log C$ is well defined via the spectral decomposition. If, $C = U \operatorname{diag}(\lambda_1,\dots,\lambda_d) U^\top$ with $\lambda_i>0$, then $\log C = U \operatorname{diag}(\log \lambda_1,\dots,\log \lambda_d) U^\top$. Because $\log C$ is symmetric, the image of $\Phi$ lies in $\mathrm{Sym}(d)$. Thus, instead of comparing the connectomes in euclidean matrix space, we first map them to $\mathrm{Sym}(d)$ and then utilize euclidean geometry via the Frobenius norm.

\paragraph{Assumptions.} All the visual analysis is based on the following assumptions.

\begin{enumerate}[label=\textbf{A\arabic*.}, leftmargin=2.6em]
    \item \textbf{SPD Validity.} The basic structural assumption required by the log-euclidean model is that all connectomes lie in $\mathcal{C}_d$, making $\Phi(C)$ well defined.
    \item \textbf{Class-conditional evaluation.} We evaluate $Q_m^y$ against $P_{\text{val}}^y$ for each class $y$. Enforcing this allows us to differentiate between models that mixes structures while still appearing globally close to the pooled real distribution.
    \item \textbf{Real-derived reference geometry.} All references for evaluations are derived from the real data only to prevent the reference from adapting to the generated samples.
    \item \textbf{Local regularity of held-out support.} The density of $\Phi_{\#}P_{\text{val}}^y$ is assumed to be locally bounded away from $0$ and $\infty$ to support interior of $P_{\text{val}}$. We adapted the regime from~\cite{cover1967nearest, devroye2013probabilistic} where nearest-neighbor radii consistently estimate local spacing. 
\end{enumerate}

\paragraph{Why these assumptions enter the analysis.} \textbf{A1, A2, and A3} ensure that every visual object remains fixed across methods and is well defined in a valid geometry. \textbf{A4} is the standard condition to estimate negative local spacing in the nearest neighbor radii. Each assumption allows to interpret the visualization in regards to the class-conditional distributional agreement in addition to summary statistics.


\subsection{Global geometric fidelity through radial calibration}\label{app:radialanalysis}

We want to identify if the generated samples have the correct global geometric understanding as per the held-out real center. For class $y$,  the log-domain mean is $\mu_{\log,\val}^y= \E_{C\sim P_{\val}^y}[\Phi(C)]$, and the corresponding log-Euclidean Fr\'echet mean is $\bar C^y_{\LE,\val} = \exp(\mu_{\log,\val}^y)$. The radial coordinate of a connectome is 

\begin{equation}
r_y(C)=d_{\LE}(C,\bar C^y_{\LE,\val})=\|\Phi(C)-\mu_{\log,\val}^y\|_F.
\label{eq:radial-coordinate}
\end{equation}

Utilizing Eq.~\ref{eq:radial-coordinate}, we map the full manifold-valued distribution to a one-dimensional law to summarize the mass as a function of distance from the held-out real center. This is a push-forward construction, $\nu^y_{\val}=(r_y)_{\#}P_{\val}^y, \ \nu^y_m=(r_y)_{\#}Q_m^y$ where $(r_y)_\#\mu$ denotes the pushforward of the distribution $\mu$ under the map $r_y$. This allows us to sample a matrix from the distribution, compute its radial distance $r_y(C)$, and study the resulting distribution on $\mathbb{R}_{\ge 0}$. The radial curves in Fig.~\ref{fig:radial-calibration} highlight these pushforward laws. To account for radial mismatch, we use the 1-Wasserstein distance~\cite{kantorovich1939mathematical}:
\begin{align}
W^{\mathrm{rad}}_{m,y}=W_1\bigl(\nu_y^m, \nu_y^{\mathrm{val}}\bigr),
\label{eq:radial-w1}
\end{align}
and the log-quantile calibration error
\begin{align}
E^{\mathrm{rad}}_{m,y} = \frac{1}{|Q|}\sum_{q\in Q}\left|\log \frac{Q_{\nu_y^m}(q) + \varepsilon}{Q_{\nu_y^{\mathrm{val}}}(q) + \varepsilon}\right|,
\label{eq:radial-logerr}
\end{align}
where $Q\subset(0,1)$ is a fixed quantile grid and $Q_{\nu}(q)$ denotes the
quantile function of a distribution $\nu$.

\begin{proposition}
\label{prop:le_frechet}
$\bar C_y^{\mathrm{LE},\mathrm{val}} = \exp(\mu_y^{\log,\mathrm{val}})$ minimizes the expected squared log-Euclidean distance to the held-out real distribution: 
\begin{align*}
\bar C_y^{\mathrm{LE},\mathrm{val}} \in \arg\min_{C\in\mathcal{C}_d} \mathbb{E}_{X\sim P_y^{\mathrm{val}}} \bigl[d_{\mathrm{LE}}(X,C)^2\bigr]. 
\end{align*}
\end{proposition}

\begin{proof}
For any fixed $C\in\mathcal{C}_d$, by Eq.~\ref{eq:le-distance}, $d_{\mathrm{LE}}(X,C)^2 = \|\log X-\log C\|_F^2$. Therefore the optimization problem is
\begin{align*}
\min_{C\in\mathcal{C}_d}
\mathbb{E}\bigl[\|\log X-\log C\|_F^2\bigr].
\end{align*}

Now let $M = \log C$. Because every $C\in\mathcal{C}_d$ is SPD, the matrix logarithm $\log C$ is well defined and maps $\mathcal{C}_d$ injectively into $\operatorname{Sym}(d)$; the image is a proper submanifold of $\operatorname{Sym}(d)$. We relax the constraint and minimize over the ambient space $\operatorname{Sym}(d)$, so the unconstrained minimizer $\bar{C}_y^{\mathrm{LE,val}} = \exp(\mu_y^{\log,\mathrm{val}})$ serves as a reference point in the broader SPD manifold rather than a minimizer strictly within $\mathcal{C}_d$. Hence the problem becomes: $\min_{M\in\operatorname{Sym}(d)} \mathbb{E}\bigl[\|\log X-M\|_F^2\bigr].$

Eq.~\ref{eq:radial-w1} measures the minimal transport cost
required to transform the generated radial space into the held-out real radial
space. Eq.~\ref{eq:radial-logerr} provides a scale-sensitive comparison of
the two distributions across their entire support.

Let $Z = \log X$. Then $Z$ is a random element of the Euclidean vector space $\operatorname{Sym}(d)$, and the objective becomes $f(M) = \mathbb{E}\bigl[\|Z-M\|_F^2\bigr]$. We now expand this expression. Recall that for any matrices $A,B$, $\|A-B\|_F^2 = \langle A-B,A-B\rangle_F$, where the Frobenius inner product is $\langle A,B\rangle_F = \operatorname{tr}(A^\top B)$. Since all matrices here are symmetric, this is just the standard Euclidean inner product on $\operatorname{Sym}(d)$. Expanding the square gives $\|Z-M\|_F^2 = \|Z\|_F^2 - 2\langle Z,M\rangle_F + \|M\|_F^2$. Taking expectation on both sides, $f(M) = \mathbb{E}[\|Z\|_F^2] -2\langle \mathbb{E}[Z],M\rangle_F +\|M\|_F^2$. The first term, $\mathbb{E}[\|Z\|_F^2]$, does not depend on $M$, so it is a constant for the optimization. Thus minimizing $f(M)$ is equivalent to minimizing
\begin{align*}
g(M)=\|M\|_F^2 - 2\langle \mathbb{E}[Z],M\rangle_F.
\end{align*}

Now we explain why this objective is convex, and in fact \emph{strictly convex}. The term $\|M\|_F^2$ is a strictly convex quadratic function of $M$, because for any distinct $M_1,M_2\in\operatorname{Sym}(d)$ and any $t\in(0,1)$,
\begin{align*}
\|tM_1+(1-t)M_2\|_F^2<t\|M_1\|_F^2 + (1-t)\|M_2\|_F^2.
\end{align*}
The linear term $-2\langle \mathbb{E}[Z],M\rangle_F$ is both convex and concave, since linear functions do not affect strict convexity. Therefore $g(M)$, and hence $f(M)$, is strictly convex. A strictly convex function has at most one minimizer.

To find the minimizer explicitly, we complete the square:
\begin{align*}
\|M\|_F^2 - 2\langle \mathbb{E}[Z],M\rangle_F=\|M-\mathbb{E}[Z]\|_F^2 - \|\mathbb{E}[Z]\|_F^2.
\end{align*}

Therefore $f(M) = \mathbb{E}[\|Z\|_F^2] - \|\mathbb{E}[Z]\|_F^2 + \|M-\mathbb{E}[Z]\|_F^2$.
Again, the first two terms are constants with respect to $M$. So $f(M)$ is minimized exactly when $\|M-\mathbb{E}[Z]\|_F^2 = 0$, that is, $M = \mathbb{E}[Z]$. Recalling that $Z=\log X$, we obtain $M = \mathbb{E}[\log X] = \mu_y^{\log,\mathrm{val}}$. Since $M=\log C$, this gives $C = \exp(\mu_y^{\log,\mathrm{val}}) = \bar C_y^{\mathrm{LE},\mathrm{val}}$.
By strict convexity, this minimizer is unique.
\end{proof}

\begin{figure*}
    \centering
    \includegraphics[width=\linewidth]{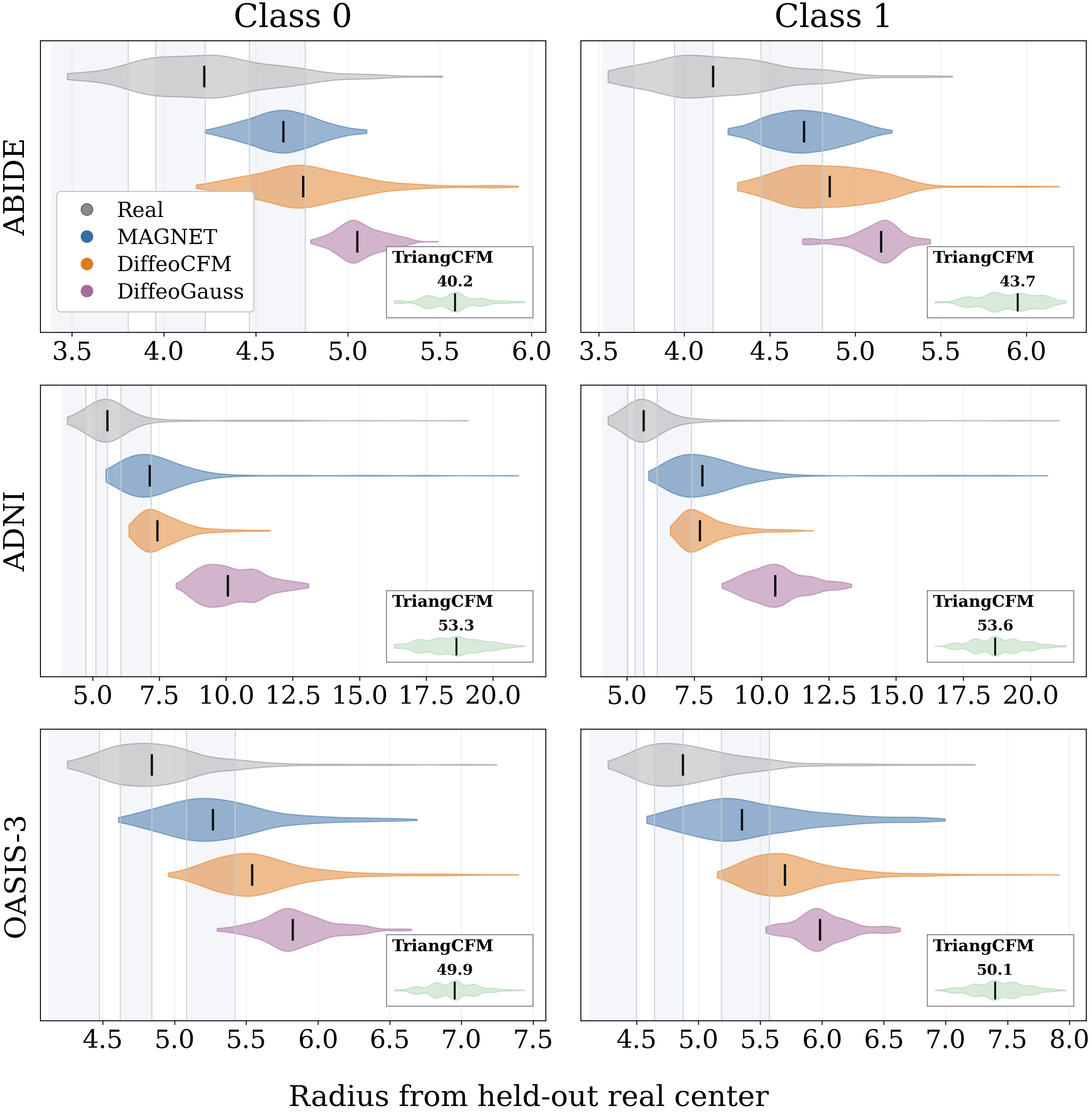}
    \caption{\footnotesize
\textbf{Global radial calibration in log-Euclidean geometry.} We plot the distribution of distances from the held-out real data center for each class and dataset. A systematic left shift relative to Real \textit{(grey)} indicates manifold contraction, while a right shift indicates overdispersion. MAGNET \textit{(blue)} consistently demonstrates the closest alignment to the real distribution across all datasets, while baselines like TriangCFM \textit{(insets)} exhibit significant overdispersion.}
\vspace{-1.5em}
\label{fig:radial-calibration}
\end{figure*}

\begin{proposition}
\label{prop:radial_consistency}
If the generated class-conditional distribution matches the held-out real distribution, i.e., $Q_y^m = P_y^{\mathrm{val}}$, then the induced radial laws are identical: $\nu_y^m = \nu_y^{\mathrm{val}}$.
Consequently, $W^{\mathrm{rad}}_{m,y} = 0 \quad \text{and} \quad E^{\mathrm{rad}}_{m,y} = 0$.
\end{proposition}

\begin{proof}
Recall Eq.~\ref{eq:le-distance}, the matrix logarithm is continuous on $\mathcal{C}_d$, subtraction by a fixed matrix is continuous, and the Frobenius norm is also continuous. Therefore $r_y$ is continuous, hence measurable. By definition of pushforward, for any measurable set $A \subseteq \mathbb{R}_{\ge 0}$, $\nu_y^m(A) = Q_y^m\bigl(r_y^{-1}(A)\bigr), \quad \nu_y^{\mathrm{val}}(A) = P_y^{\mathrm{val}}\bigl(r_y^{-1}(A)\bigr)$. If $Q_y^m = P_y^{\mathrm{val}}$ as probability measures on $\mathcal{C}_d$, then for every measurable $A$, $Q_y^m\bigl(r_y^{-1}(A)\bigr) = P_y^{\mathrm{val}}\bigl(r_y^{-1}(A)\bigr)$. Hence $\nu_y^m(A) = \nu_y^{\mathrm{val}}(A)$, which implies $\nu_y^m = \nu_y^{\mathrm{val}}$.

If two probability distributions are identical, then any proper distance or discrepancy between them must vanish. The 1-Wasserstein distance $W_1(\nu_y^m, \nu_y^{\mathrm{val}})$ is zero if and only if the two measures are identical. The quantile functions also coincide, i.e., $Q_{\nu_y^m}(q) = Q_{\nu_y^{\mathrm{val}}}(q) \quad \text{for all } q \in (0,1)$, which implies every term in Eq.~\ref{eq:radial-logerr} is zero. Therefore, $W^{\mathrm{rad}}_{m,y} = 0$ and $E^{\mathrm{rad}}_{m,y} = 0$.
\end{proof}
\textbf{\textit{Remark.}} Proposition~\ref{prop:radial_consistency} establishes a one-directional necessity: matching radial laws is required, but fundamentally insufficient, for full distribution agreement. Intuitively, the radial mapping $r_y$ acts as an isotropic aggregator, collapsing the $\frac{d(d-1)}{2}$-dimensional tangent space of $\operatorname{Sym}(d)$ into a single non-negative scalar. Because this operation strips away all angular information, two distributions can easily share identical radial profiles while occupying completely disjoint directional regions. Geometrically, this is analogous to comparing a uniform spherical shell to a narrow, dense equatorial ring; both share the same radial distance from the center, yet their structures differ entirely. In the context of connectomes, this implies that a generator might perfectly replicate the correct \textit{magnitude} of overall network deviation from the mean class state, but allocate those topological changes to entirely incorrect functional edges or subnetworks. Therefore, while vanishing radial discrepancy ($W^{\mathrm{rad}}_{m,y} = 0$) provides a rigorous, necessary confirmation that global variance is calibrated, it must be paired with directional or topological evaluations to guarantee true conditional support recovery.

\textbf{\textit{Empirical Takeaway.}} Figure~\ref{fig:radial-calibration} visualizes the full radial pushforward distributions and their Wasserstein distances to the real held-out center. Table~\ref{tab:radial-median} reports the corresponding \emph{median} log-Euclidean radial distances for MAGNET, DiffeoCFM, and DiffeoGauss. TriangCFM is excluded from the table because its severe overdispersion (median radial distances exceeding 40 units, visible as inset annotations in Fig.~\ref{fig:radial-calibration}) places it far outside the axis range of the remaining methods, making a shared table misleading. Across all three datasets, MAGNET achieves the smallest deviation from the real center, confirming tighter global geometric alignment. If a generator matches the true distribution, then this radial law must match exactly, and both Wasserstein and quantile discrepancies vanish. However, radial agreement alone does not guarantee full geometric fidelity.

\subsection{Local Support Occupancy via Tangent Charts}\label{subsec:tangent_occupancy}
While radial analysis verifies that the generated samples exhibit the correct global scale around the real class center (Appendix~\ref{app:radialanalysis}), it does not reveal whether the generated probability mass correctly aligns with the true support of the data distribution. To investigate this local occupancy, we construct a low-dimensional tangent chart using held-out real data and evaluate how the generated samples project onto this localized coordinate space. 

\noindent
\textbf{Log-domain centering and vectorization.} For each held-out real connectome $C_i^{\mathrm{val},y}$, we define
\begin{align}
z_i^{\mathrm{val},y} = \operatorname{vech}_0 \bigl(\log C_i^{\mathrm{val},y} - \mu_y^{\log,\mathrm{val}}\bigr) \in \mathbb{R}^p, \quad p = \frac{d(d-1)}{2}
\end{align}
Here $\operatorname{vech}_0(M)$ denotes vectorization of the nonredundant off-diagonal entries of $M$. This step converts each connectome into a vector in a Euclidean space, centered at the held-out real mean.

\noindent
\textbf{A real-defined tangent chart.} Let $W_y \in \mathbb{R}^{p\times 2}$ denote the first two principal directions (PCA) of the held-out real vectors $\{z_i^{\mathrm{val},y}\}$. We define the
two-dimensional chart coordinate of any connectome $C$ by
\begin{align}
u_y(C) = W_y^\top \operatorname{vec}_{\triangle} \bigl(\log C - \mu_y^{\log,\mathrm{val}}\bigr)
\label{eq:tangent-chart}
\end{align}

Importantly, the basis $W_y$ is computed using \emph{held-out real data only}. Generated samples are projected into a chart they did not define.

\noindent
\textbf{Real support region in the chart.} Let $\widehat{\rho}_y^{\mathrm{val}}$ be a kernel density estimate of the distribution of $\{u_y(C_i^{\mathrm{val},y})\}$. For a chosen threshold $\tau_y$, the support region:
\begin{align}
S_y = \{\, u \in \mathbb{R}^2 : \widehat{\rho}_y^{\mathrm{val}}(u) \ge \tau_y \,\}
\end{align}
The KDE uses a Gaussian kernel with Scott's bandwidth~\cite{scott2011multivariate} $h_y = n_y^{-1/6}$ and empirical covariance $\widehat{\Sigma}_y$ of the real chart points~\cite{devroye2013probabilistic}. The threshold $\tau_y = \operatorname{Quantile}_{0.10}\!\left(\{\widehat{\rho}_y^{\mathrm{val}}(u_y(C_i^{\mathrm{val},y}))\}_i\right)$ is the 10th percentile of real KDE scores, so $S_y$ approximates a central $90\%$ real-support region; all parameters are derived from held-out real data only. Intuitively, $S_y$ captures the region where held-out real data is concentrated in the chart.
\begin{wraptable}{r}{0.47\textwidth}
    \centering
    \scriptsize
    \setlength{\tabcolsep}{3pt}
    \newcommand{\rd}[2]{#1 {\scriptsize\color{red}(+#2)}}
    \caption{\footnotesize\textbf{Geometric Fidelity Analysis.} Median log-Euclidean
    radial distances to the held-out real class center (Eq.~\ref{eq:radial-coordinate}). Red deltas ($\Delta$) indicate the shift relative to the ground-truth median. TriangCFM is excluded due to severe overdispersion (see Fig.~\ref{fig:radial-calibration} insets). D-CFM is DiffeoCFM; D-Gauss is DiffeoGauss. CI indicates the class index.}
    
    \begin{tabular}{ll cccc}
      \toprule
      \textbf{Data} & \textbf{CI} & \textbf{Real} & \textbf{MAGNET} & \textbf{D-CFM} & \textbf{D-Gauss} \\
      \midrule
      \multirow{2}{*}{ABIDE} & C0 & 4.22 & \textbf{\rd{4.65}{.43}} & \rd{4.76}{.53} & \rd{5.05}{.83} \\
                             & C1 & 4.17 & \textbf{\rd{4.70}{.54}} & \rd{4.85}{.68} & \rd{5.15}{.98} \\
      \midrule
      \multirow{2}{*}{ADNI}  & C0 & 5.54 & \textbf{\rd{7.15}{1.6}} & \rd{7.43}{1.9} & \rd{10.1}{4.5} \\
                             & C1 & 5.62 & \rd{7.80}{2.2} &  \textbf{\rd{7.71}{2.1}} & \rd{10.5}{4.9} \\
      \midrule
      \multirow{2}{*}{OASIS} & C0 & 4.84 & \textbf{\rd{5.27}{.43}} & \rd{5.54}{.70} & \rd{5.83}{.98} \\
                             & C1 & 4.88 & \textbf{\rd{5.35}{.48}} & \rd{5.70}{.82} & \rd{5.98}{1.1} \\
      \bottomrule
    \end{tabular}
    
    \label{tab:radial-median}
    \vspace{-2em}
\end{wraptable}

\paragraph{Inside-support rate.} We measure how much generated mass falls inside this real-defined region via
\begin{align}
\mathrm{ISR}_{m,y}=\mathbb{P}_{C \sim Q_y^m}\bigl(u_y(C) \in S_y\bigr)
\label{eq:isr}
\end{align}
This is a simple and interpretable quantity: the fraction of generated samples that lie inside the held-out real support in the tangent chart.

\begin{proposition}
\label{prop:isr_consistency}
Suppose the chart distributions agree, i.e., $(u_y)_\# Q_y^m = (u_y)_\# P_y^{\mathrm{val}}$, and assume that the boundary of $S_y$ has zero probability under the held-out real chart distribution. Then $\mathrm{ISR}_{m,y} = \mathbb{P}_{C \sim P_y^{\mathrm{val}}}\bigl(u_y(C) \in S_y\bigr)$.
\end{proposition}

\begin{proof}
The map $u_y(C) = W_y^\top \operatorname{vec}_{\triangle}(\log C - \mu_y^{\log,\mathrm{val}})$ is a composition of continuous operations (logarithm, subtraction, linear projection), hence it is continuous and therefore measurable. The set $S_y$ is defined through a density threshold and is therefore measurable. By assumption, $(u_y)_\# Q_y^m = (u_y)_\# P_y^{\mathrm{val}}$. This means that for any measurable set $A \subseteq \mathbb{R}^2$,
\begin{align*}
\mathbb{P}_{C \sim Q_y^m}(u_y(C) \in A)=\mathbb{P}_{C \sim P_y^{\mathrm{val}}}(u_y(C) \in A).
\end{align*}

Taking $A = S_y$, we obtain
\begin{align*}
\mathbb{P}_{C \sim Q_y^m}(u_y(C) \in S_y)=\mathbb{P}_{C \sim P_y^{\mathrm{val}}}(u_y(C) \in S_y),
\end{align*}
provided that $S_y$ is a continuity set of the distribution. This holds if the boundary of $S_y$ has zero probability, which is assumed. Thus, $\mathrm{ISR}_{m,y} = \mathbb{P}_{C \sim P_y^{\mathrm{val}}}(u_y(C) \in S_y)$.
\end{proof}

\begin{figure*}
    \centering
    \includegraphics[width=\linewidth]{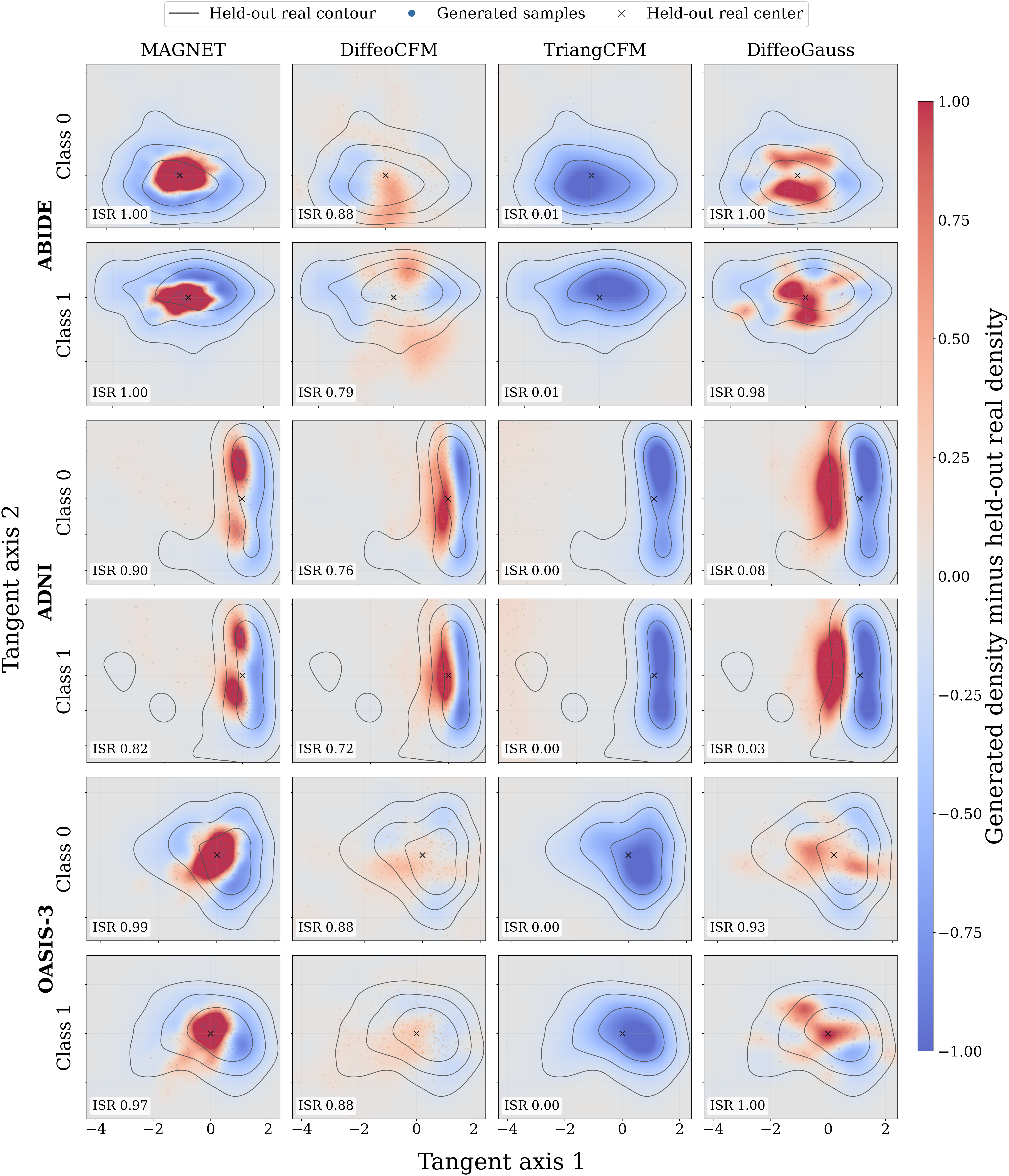}
    \caption{\footnotesize \textbf{Held-out real tangent-manifold occupancy atlas.} We project generated samples into a 2D tangent chart defined strictly by held-out real data to evaluate local support occupancy. \textbf{Visual guide:} Gray contours represent the true data distribution (real support, $S_y$). The colormap indicates the density difference (generated minus real), where red highlights areas of over-generation (excess mass) and blue indicates under-generation (missing mass). \textbf{Takeaway:} MAGNET tightly bounds its generated samples within the valid real contours, consistently achieving the highest Inside-Support Rate (ISR, Eq.~\ref{eq:isr}). Conversely, DiffeoCFM leaks mass beyond the true boundaries, TriangCFM exhibits severe mode collapse, and DiffeoGauss hallucinates out-of-distribution samples under the high heterogeneity of the ADNI cohort.}
    \label{fig:manifold-atlas}
\end{figure*}

\textit{\textbf{Remark.}} According to proposition \ref{prop:isr_consistency}, samples generated by a generator that perfectly aligns with the distribution will also align with the validation support data. Thus, the empirical results of Inside-Support Rate, defined as $\mathrm{ISR}_{m,y}$ supported by Eq.~\ref{eq:isr} directly helps us to identify the two geometric failures: \textit{over-dispersion} where the generator decides to place distribution mass outside the high-density real regions denoted by $S_y$ and \textit{under-coverage} where the generator fails to cover the distribution of the true support set.

\textit{\textbf{Empirical Takeaway.}} Figure \ref{fig:manifold-atlas} and Table \ref{tab:manifold-atlas-isr} quantify the spatial alignment of the generated density using the inside-support rate ($\mathrm{ISR}_{m,y}$) defined in Equation \ref{eq:isr}. It can be observed that the atlas (see Figure~\ref{fig:manifold-atlas}) shows distinct behavioral structures across the baselines. TriangCFM (\textit{T-CFM in Table~\ref{tab:manifold-atlas-isr}}) shows severe mode collapse and fails to populate the interior real support contour ($S_y$), generating $\mathrm{ISR}_{m,y}$ near zero across all clinical cohorts. 

\begin{wraptable}{r}{0.45\textwidth} 
  \centering
  \scriptsize
  \setlength{\tabcolsep}{1pt} 
  
  \caption{\footnotesize
  \textbf{Quantitative tangent-manifold support recovery.} Inside-support rate ($\mathrm{ISR}_{m,y} \uparrow$) measures the fraction of generated samples falling within the held-out real support region (Eq.~\ref{eq:isr}).This quantifies whether a model accurately allocates probability mass to validation-supported regions of the tangent-manifold atlas (Fig.~\ref{fig:manifold-atlas}) without relying on regions it defined itself.}
  
  \begin{tabular}{llcccc}
  \toprule
  \textbf{Dataset} & \textbf{Class} & \textbf{MAGNET} & \textbf{D-CFM} & \textbf{T-CFM} & \textbf{D-Gauss} \\
  \midrule
  \multirow{3}{*}{ABIDE} 
        & Class 0 & \textbf{1.000} & 0.879 & 0.008 & 0.996 \\
        & Class 1 & \textbf{1.000} & 0.787 & 0.008 & 0.983 \\
  \rowcolor{LightGray}
  \cellcolor{white} & Avg. & \textbf{1.000} & 0.833 & 0.008 & 0.990 \\
  \addlinespace
  
  \multirow{3}{*}{ADNI}  
        & Class 0 & \textbf{0.896} & 0.762 & 0.000 & 0.083 \\
        & Class 1 & \textbf{0.821} & 0.721 & 0.000 & 0.029 \\
  \rowcolor{LightGray}
  \cellcolor{white} & Avg. & \textbf{0.859} & 0.742 & 0.000 & 0.056 \\
  \addlinespace
  
  \multirow{3}{*}{OASIS-3} 
          & Class 0 & \textbf{0.988} & 0.879 & 0.000 & 0.925 \\
          & Class 1 & 0.967 & 0.875 & 0.000 & \textbf{0.996} \\
  \rowcolor{LightGray}
  \cellcolor{white} & Avg. & \textbf{0.978} & 0.877 & 0.000 & 0.961 \\
  \addlinespace
  \midrule
  
  \rowcolor{LightBlue}
  \multicolumn{2}{l}{\textbf{Overall avg.}} & \textbf{0.945} & 0.817 & 0.003 & 0.669 \\
  \bottomrule
  \end{tabular}
  \label{tab:manifold-atlas-isr}
  \vspace{-5em} 
  
\end{wraptable}

DiffeoGauss (\textit{D-Gauss in Table~\ref{tab:manifold-atlas-isr}}) populates the real support contour very well on ABIDE and OASIS-3, but completely hallucinates on ADNI due to the severe heterogeneity present inside the dataset and ends up projecting beyond the true contours. On the other hand, DiffeoCFM (\textit{D-CFM in Table~\ref{tab:manifold-atlas-isr}}) tracks the correct macroscopic structure but leaks the probability mass just outside the real boundaries. In contrast, our proposed method, MAGNET, localizes the generated density inside the dominant contour of held-out real support which allows us to achieve the highest ISR across all the datasets.


\section{Reverse-time refinement}\label{app:reversetime}
In the preceding analysis, we established that generated samples from the different baselines showcase distinct behaviors. The radial calibration (Appendix \ref{app:radialanalysis}) helped to identify the global scale mismatch and the tangent occupancy (Appendix \ref{subsec:tangent_occupancy}) revealed the local issues of \textit{over-dispersion} and \textit{under-coverage} with respect to true support $S_y$. However, we can only assess the terminal samples ($t=0$) of the generative process using those analysis. We also need to understand \textit{how} the probability mass navigates towards the held-out real manifold during the denoising. Hence, we now move beyond the static endpoints and inspect the dynamics of the reverse-time path $C_{m,t}^y$ as $t \to 0$.

\subsection{Reverse-time tangent-manifold trajectory}
Let $C_{m,t}^y$ denote the connectome generated by method $m$ at reverse-time checkpoint $t\in\{5,4,3,2,1,0\}$ for class $y$. We project each checkpoint through the same held-out real chart (Eq.~\ref{eq:tangent-chart}): $u_{m,t}^y = u_y(C_{m,t}^y)$. The reverse-time trajectory plots in Fig.~\ref{fig:trajectory-manifold-0} and \ref{fig:trajectory-manifold-1} therefore answer a sharply defined question: \textit{as reverse time progresses, does the generated cloud move into the same held-out real support region identified in the final-sample occupancy atlas?}

The associated checkpoint-wise inside-support rate is
\begin{equation}
\ISR_{m,y}(t)=\Prob\!\left(u_{m,t}^y\in S_y\right).
\label{eq:isr-time}
\end{equation}

\textbf{\textit{Remark.}} Note, we do not assume that $t\mapsto\ISR_{m,y}(t)$ is monotone. The reverse-time quality metrics in diffusion and flow-based generative trajectories need not to evolve monotonically with time~\cite{wu2025optimal, li2026kinetic}. The regime the path follows is much more important. 

\textbf{\textit{Empirical Takeaway.}} Trajectory manifolds are shown for split 0 using all held-out validation subjects of the selected class: ABIDE (n$_0$=52, n$_1$=33), ADNI (n$_0$=103, n$_1$=145), and OASIS-3 (n$_0$=109, n$_1$=65). DiffeoCFM enters the vicinity of the real contour support earlier but MAGNET finishes with tighter support alignment in the second half of the reverse process to ensure high alignment to the real support contour. On OASIS-3, we see a modest gap between the two methods with MAGNET remaining near the real support contour and DiffeoCFM decaying mildly from $1.00$ to $0.91$. However, ABIDE highlights different dynamics. Both the methods remain inside the support contour throughout the denoising trajectory, indicating easier adaptability to the data as compared to the previous two datasets.

\begin{figure*}[!t]
    \centering
    \includegraphics[width=\linewidth]{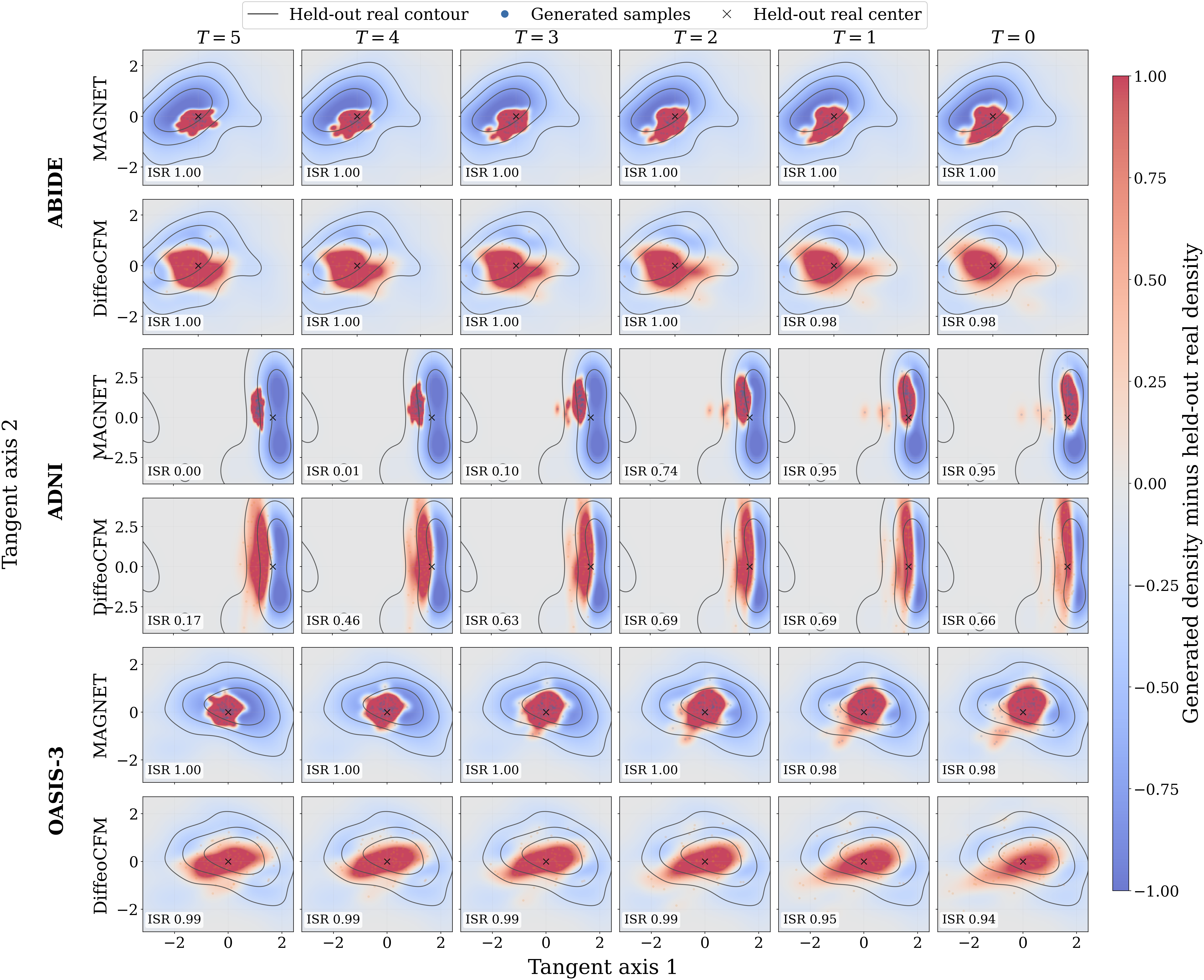}
    \caption{\footnotesize
    \textbf{Dynamic support alignment for healthy controls (Class 0).} We track the generated probability mass over the reverse-time denoising trajectory ($T=5 \rightarrow 0$). Each checkpoint is projected into the fixed, real-data-defined tangent chart established in our occupancy atlas. \textbf{Visual Guide:} The solid gray contour demarcates the true held-out data support ($S_y$); the colormap tracks localized density error. \textbf{Takeaway:} MAGNET strictly bounds its trajectory within the interior of the true held-out data contour, achieving near-perfect inside-support rates without the boundary leakage typical of transport baselines.
    }
    \label{fig:trajectory-manifold-0}
\end{figure*}

\begin{figure*}[!t]
    \centering
    \includegraphics[width=\linewidth]{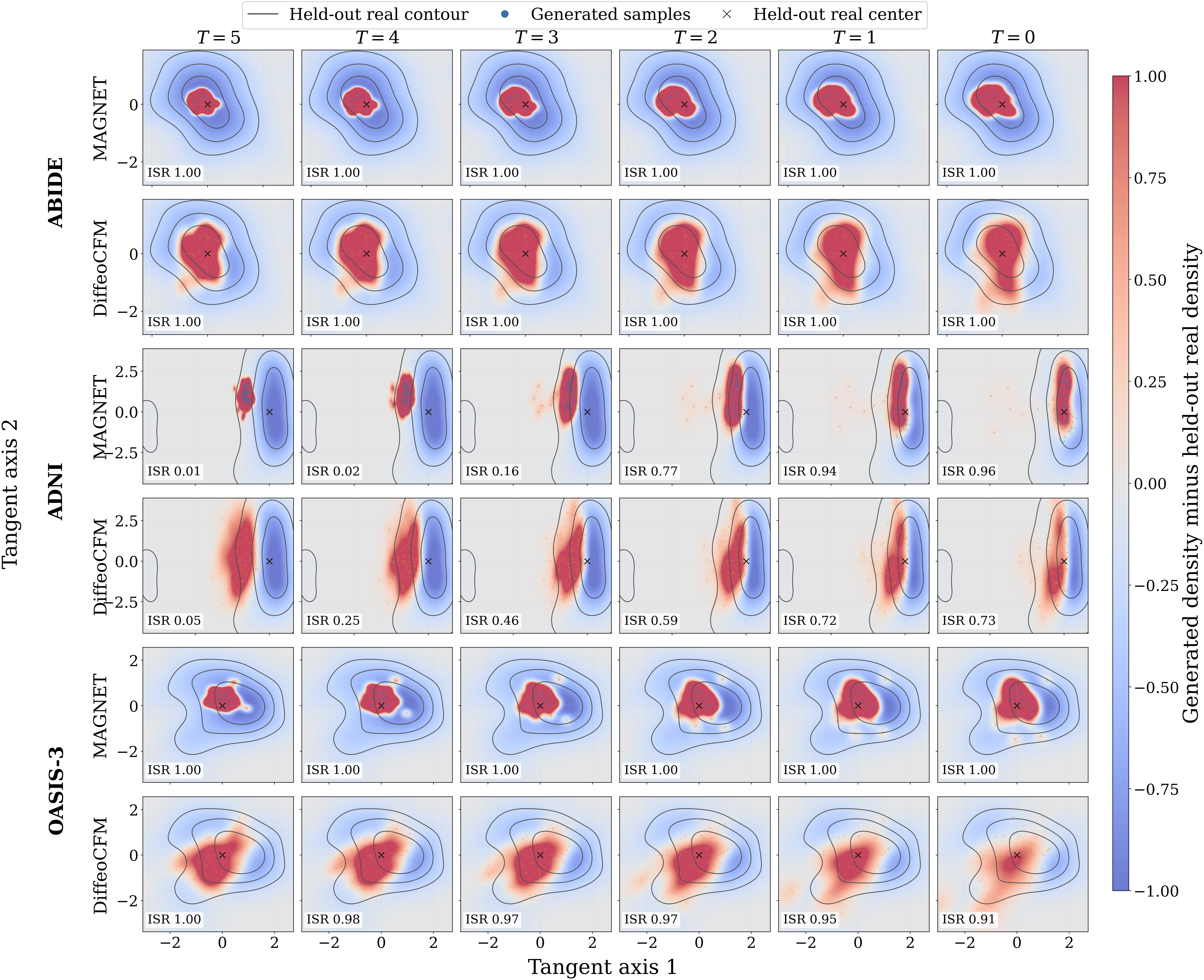}
    \caption{\footnotesize
    \textbf{Dynamic support alignment for heterogeneous disease cohorts (Class 1).} Tracking the reverse-time trajectory ($T=5 \rightarrow 0$) against the fixed real-data tangent chart. \textbf{Visual Guide:} The solid gray contour demarcates the true held-out data support ($S_y$); the colormap tracks localized density error. \textbf{Takeaway:} The high structural heterogeneity of diseased subjects causes baselines like DiffeoCFM to stagnate early and bleed mass outside valid clinical regions. Conversely, MAGNET actively tightens its alignment during the critical final denoising steps, successfully navigating complex, multi-modal topologies to terminate securely within the true clinical support.
    }
    \label{fig:trajectory-manifold-1}
\end{figure*}

\subsection{Stepwise residual refinement on a fixed held-out real edge mask}
The reverse trajectory is geometric but it is not able to show edge-specific generation by the methods. We examine the residual mean connectomes on a fixed real-defined edge mask to understand how the first order structure is refined by each method over the denoising timesteps.

For class $y$, let $\bar C_{\val}^y = \E_{C\sim P_{\val}^y}[C]$ be the held-out real class mean. Let $M_y$ be the binary mask containing the top held-out real edges according to $|\bar C_{\val}^y(a,b)|$. For method $m$ at checkpoint $t$, we define the generated class mean $\bar C_{m,t}^y = \E[C_{m,t}^y]$. We compare the fixed-mask residual
\begin{equation}
R_{m,t}^y=\bigl(\bar C_{m,t}^y-\bar C_{\val}^y\bigr)\odot M_y.
\label{eq:residual-strip}
\end{equation}

First column of Fig.~\ref{fig:connectome_class0} and \ref{fig:connectome_class1} shows the real held-out mean on the mask itself. Each subsequent column shows the ``generated - real'' residual using the same fixed mask. Since we keep the mask and the color scale constant across columns, we can directly compare across denoising checkpoints, avoiding any threshold drift that can arise if the support definition changes. We report two scalar statistic on the masked edges:
\begin{equation}
\MAE_{m,t}^y=\frac{1}{|M_y|}\sum_{(a,b)\in M_y}\left|\bar C_{m,t}^y(a,b)-\bar C_{\val}^y(a,b)\right|,
\label{eq:masked-mae}
\end{equation}
and
\begin{equation}
\Corr_{m,t}^y=\operatorname{Corr}\left(\{\bar C_{m,t}^y(a,b)\}_{(a,b)\in M_y},\{\bar C_{\val}^y(a,b)\}_{(a,b)\in M_y}\right).
\label{eq:masked-corr}
\end{equation}

\begin{figure}[!t]
    \centering
    \includegraphics[width=\linewidth]{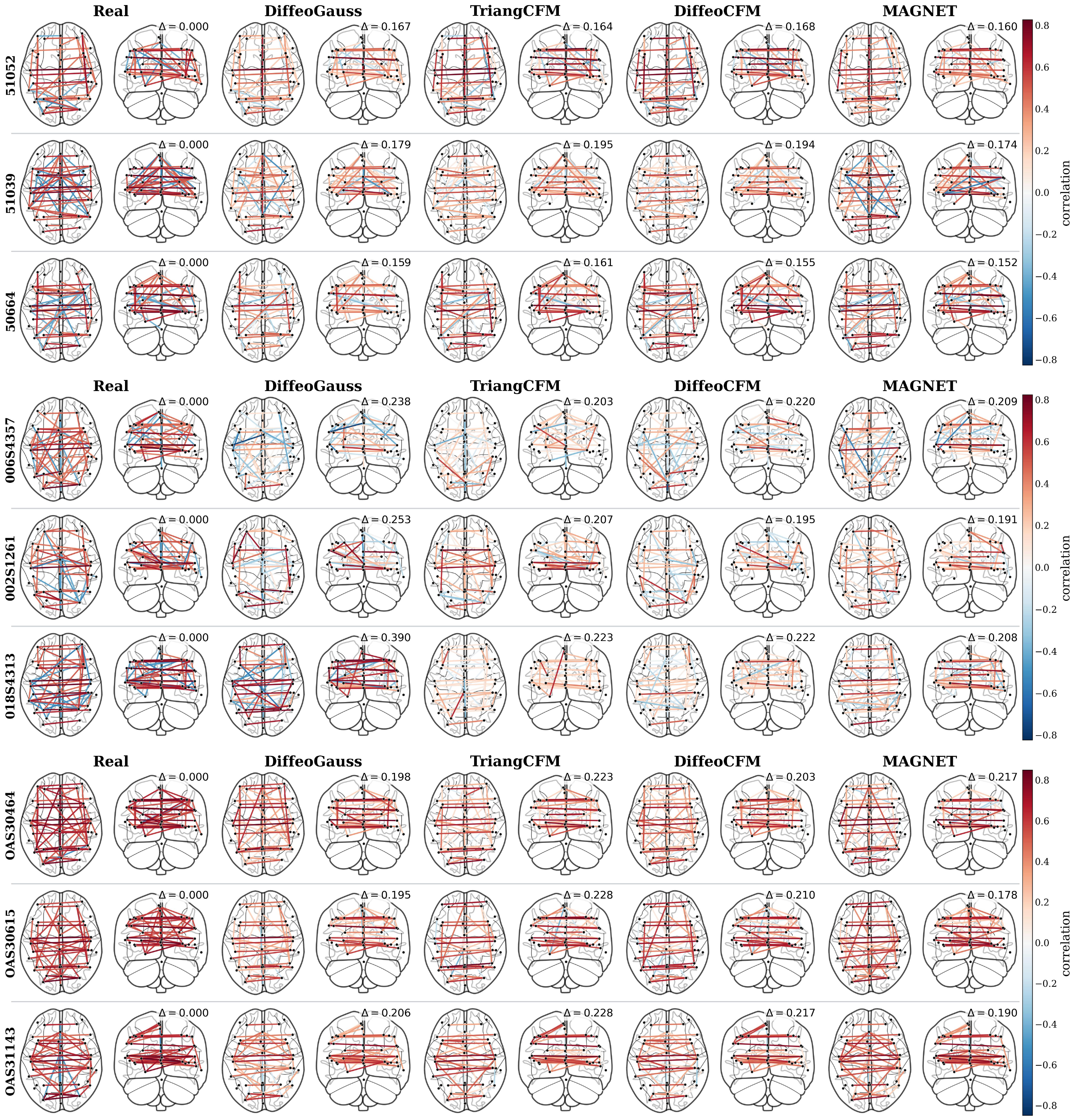}
    \caption{\footnotesize
    \textbf{Held-out connectome comparisons for class $0$.}
    Top to bottom: ABIDE, ADNI, and OASIS-3. For each dataset, we show three held-out scans, comparing the real connectome against DiffeoGauss, TriangCFM, DiffeoCFM, and MAGNET. Each row uses the top-$45$ edges of the corresponding real scan and reuses the same real-defined color scale across methods. $\Delta$ denotes the upper-triangular MAE to the real correlation matrix, with smaller values indicating closer agreement.
    }
    \label{fig:connectome_class0}
\end{figure}

\begin{figure}[!t]
    \centering
    \includegraphics[width=\linewidth]{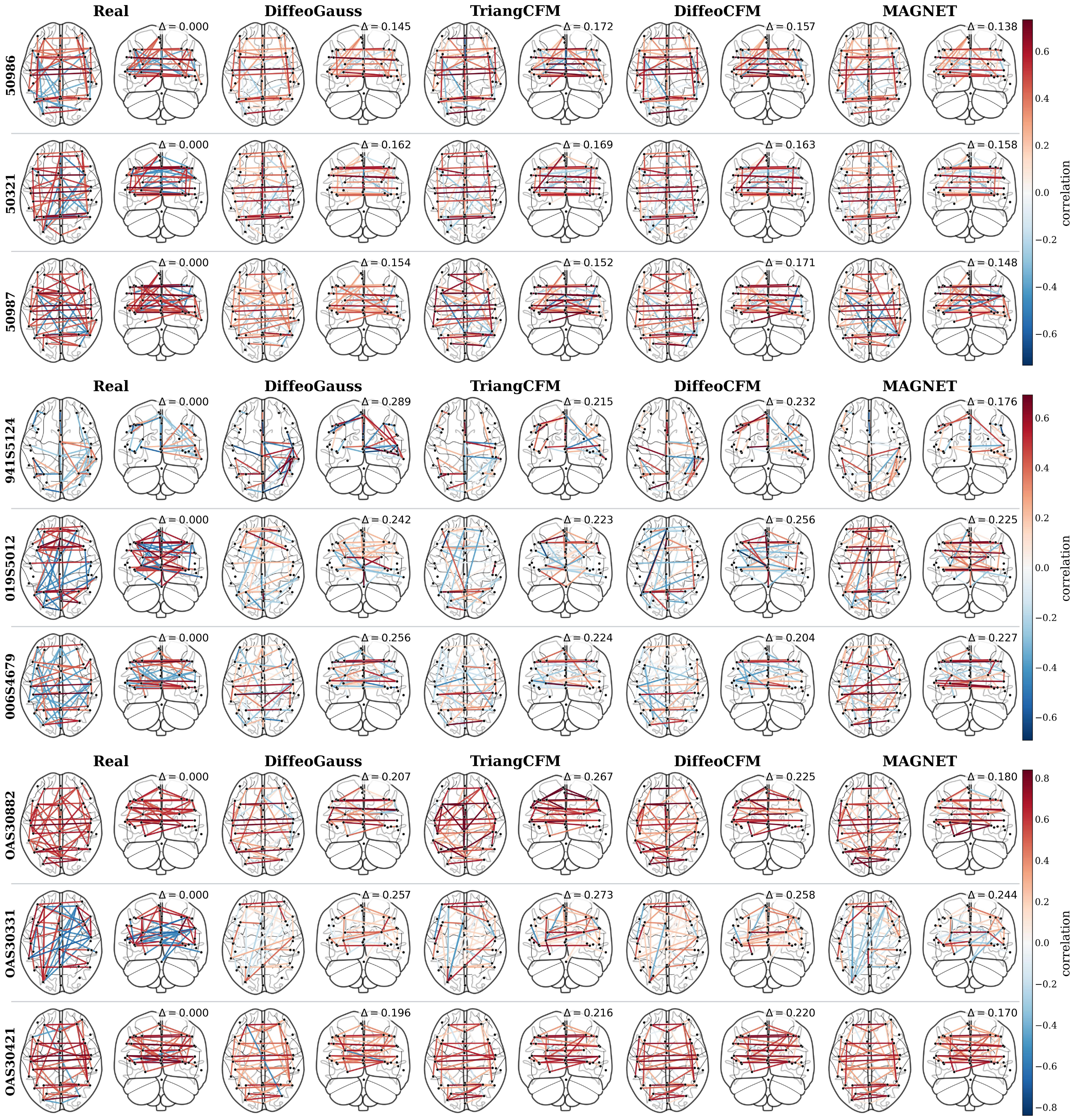}
    \caption{\footnotesize
    \textbf{Held-out connectome comparisons for class $1$.}
    Top to bottom: ABIDE, ADNI, and OASIS-3. For each dataset, we show three held-out scans, comparing the real connectome against DiffeoGauss, TriangCFM, DiffeoCFM, and MAGNET. Each row uses the top-$45$ edges of the corresponding real scan and reuses the same real-defined color scale across methods. $\Delta$ denotes the upper-triangular MAE to the real correlation matrix, with smaller values indicating closer agreement.
    }
    \label{fig:connectome_class1}
\end{figure}

\begin{proposition}
\label{prop:residual_equivalence}
For fixed $m$, $t$, and $y$, the following statements are equivalent: \ding{182} $R_{m,t}^y = 0$, \ding{183} $\mathrm{MAE}_{m,t}^y = 0$, \ding{184} $\bar C_{m,t}^y(a,b) = \bar C_{\mathrm{val}}^y(a,b)$ for all $(a,b)\in M_y$.
\end{proposition}

\begin{proof}
We reutilize the definitions discussed in Eq.~\ref{eq:residual-strip} and \ref{eq:masked-mae}.

\textit{\ding{182} $\Rightarrow$ \ding{184}} If $R_{m,t}^y = 0$, then every masked entry of the difference is zero (Eq.~\ref{eq:residual-strip}). Hence $\bar C_{m,t}^y(a,b) = \bar C_{\mathrm{val}}^y(a,b)
\quad \text{for all } (a,b)\in M_y$.

\textit{\ding{184} $\Rightarrow$ \ding{183}} If all masked entries match exactly, then every absolute difference in the MAE (Eq.~\ref{eq:masked-mae}) is zero, so $\mathrm{MAE}_{m,t}^y = 0$.

\textit{\ding{183} $\Rightarrow$ \ding{182}} The MAE is an average of nonnegative terms. If the average is zero, then every term must be zero, implying that all masked residuals vanish. Therefore $R_{m,t}^y = 0$.

Thus all three statements are equivalent.
\end{proof}

\textit{\textbf{Remark.}} Using proposition~\ref{prop:residual_equivalence}, it is easily understandable what the connectome strip measures and what it does not. The connectome strips are first order structural refinement view on a held-out real edge set. \textit{Note, this is not a substitute for full-distribution fidelity.} By anchoring the evaluation to a fixed, real-defined mask $M_y$, we guarantee that $\mathrm{MAE}_{m,t}^y$ serves as a faithful, unconfounded proxy for the residual heatmaps $R_{m,t}^y$.

\textit{\textbf{Empirical Takeaway.}}
Figures~\ref{fig:connectome_class0} and~\ref{fig:connectome_class1} extend the main-text qualitative evidence from one scan per class to three additional held-out scans per class. The same pattern persists. ADNI remains the clearest separation: DiffeoGauss oversmooths, TriangCFM often collapses toward a narrow core, and DiffeoCFM preserves some dominant edges but more frequently distorts their global arrangement. MAGNET is more consistent at preserving the leading cross-hemispheric and frontal-posterior structure while maintaining lower or competitive $\Delta$ values. ABIDE is visually easier, with all valid baselines closer to the target, but MAGNET remains competitive under the same held-out mask. OASIS-3 is more mixed, yet MAGNET still preserves the most stable global edge geometry across scans. These additional examples reduce the risk that the main-text qualitative result is driven by a single curated scan and support the same structural-recovery conclusion.

\section{Broader Impacts}\label{appsec:broaderimpacts}


Conditional connectome synthesis lies at the intersection of two broader challenges: generating valid samples on constrained manifolds and recovering class-conditional structure under coarse supervision. MAGNET addresses both, with implications beyond neuroimaging.

\textit{On the positive side}, generating geometrically valid, class-conditional connectomes supports data-scarce clinical research by augmenting underrepresented groups, enabling multi-site studies, and facilitating privacy-preserving data sharing. By modeling within-class heterogeneity via subtype conditioning, MAGNET avoids a key failure of binary-supervised generators, producing unrealistic averaged samples, thereby improving the reliability of synthetic augmentation. More broadly, the work highlights a general problem: \emph{class-conditional manifold awareness}. In domains where outputs must lie on nonlinear manifolds and labels collapse diverse subpopulations (e.g., molecular generation, materials science, medical imaging), validity alone is insufficient as models may remain manifold-valid while producing implausible samples due to support averaging. MAGNET provides a transferable framework for such settings. In particular, relevance-preserving corruption ensures class-relevant structure persists during generation, with implications for diffusion-based distillation and continual learning, and potential extensions to mitigating hallucinations and poor-quality synthesis.

\textit{Limitations and risks} remain. Synthetic connectomes may amplify demographic biases present in source datasets, and discovered subtypes are statistical rather than clinically validated. Additionally, highly realistic synthetic data closely mimicking real data samples (memorization)  could be misused (e.g., fabricating cohorts), necessitating careful validation before deployment in sensitive applications concerning privacy.

\clearpage
\newpage
\input{checklist.tex}

\end{document}

%% file: checklist.tex
\section*{NeurIPS Paper Checklist}

\begin{enumerate}

\item {\bf Claims}
    \item[] Question: Do the main claims made in the abstract and introduction accurately reflect the paper's contributions and scope?
    \item[] Answer: \answerYes{}
    \item[] Justification: The abstract and main text state the central claims on validity-preserving conditional support recovery, few-step sampling, and preservation of fine-grained clinical structure. The main benchmark is supported by three cohorts, while Table~\ref{tab:rebuttal_high_res} additionally evaluates robustness to higher-resolution anatomical and functional parcellations.

\item {\bf Limitations}
    \item[] Question: Does the paper discuss the limitations of the work performed by the authors?
    \item[] Answer: \answerYes{} 
    \item[] Justification: The limitations of the proposed MAGNET framework is present in the Conclusion (sec.~\ref{sec:conclusion}).

\item {\bf Theory assumptions and proofs}
    \item[] Question: For each theoretical result, does the paper provide the full set of assumptions and a complete (and correct) proof?
    \item[] Answer: \answerYes{}
    \item[] Justification: The paper formally states propositions in Section~\ref{sec:methodology} and provides complete mathematical proofs alongside structural assumptions in Appendices~\ref{app:proofs}, \ref{app:geometricanalysis}, and \ref{app:frameworkanalysis}.

    \item {\bf Experimental result reproducibility}
    \item[] Question: Does the paper fully disclose all the information needed to reproduce the main experimental results of the paper to the extent that it affects the main claims and/or conclusions of the paper (regardless of whether the code and data are provided or not)?
    \item[] Answer: \answerYes{}
    \item[] Justification: The paper thoroughly details the experimental setup, dataset splits, preprocessing steps (MSDL atlas, z-scoring), sampling budgets, and network hyperparameters in Section~\ref{subsec:exp_setup} and Appendix~\ref{app:archdetails}.

\item {\bf Open access to data and code}
    \item[] Question: Does the paper provide open access to the data and code, with sufficient instructions to faithfully reproduce the main experimental results, as described in supplemental material?
    \item[] Answer: \answerYes{} 
    \item[] Justification: Section~\ref{subsec:exp_setup} and Appendix~\ref{app:archdetails} specify the data splits, preprocessing, architecture, training and sampling procedures, evaluation protocol, and principal hyperparameters. The accompanying code repository provides the implementation and scripts used for the reported benchmarks and ablations.

\item {\bf Experimental setting/details}
    \item[] Question: Does the paper specify all the training and test details (e.g., data splits, hyperparameters, how they were chosen, type of optimizer) necessary to understand the results?
    \item[] Answer: \answerYes{}
    \item[] Justification: Experimental details, including the 10 disjoint train/test splits, balanced class sampling protocols, SNR loss reweighting, and precise neural architecture dimensions are documented in Section~\ref{subsec:exp_setup} and Appendix~\ref{app:archdetails}.

\item {\bf Experiment statistical significance}
    \item[] Question: Does the paper report error bars suitably and correctly defined or other appropriate information about the statistical significance of the experiments?
    \item[] Answer: \answerYes{}
    \item[] Justification: All empirical tables (e.g., Tables~\ref{tab:main_brain_results}, \ref{tab:zero_shot_metrics}, \ref{tab:bridge_ablations}, etc.) report standard deviations computed across the 10 disjoint train/validation splits for every metric.
    \item[] Guidelines:

\item {\bf Experiments compute resources}
    \item[] Question: For each experiment, does the paper provide sufficient information on the computer resources (type of compute workers, memory, time of execution) needed to reproduce the experiments?
    \item[] Answer: \answerYes{}
    \item[] Justification: Wall-clock sampling times are reported in Table~\ref{tab:main_brain_results}. 
    Hardware details (single NVIDIA A6000 GPU, 48GB VRAM) are specified in 
    Section~\ref{subsec:exp_setup}.
    
\item {\bf Code of ethics}
    \item[] Question: Does the research conducted in the paper conform, in every respect, with the NeurIPS Code of Ethics \url{https://neurips.cc/public/EthicsGuidelines}?
    \item[] Answer: \answerYes{}
    \item[] Justification: The research relies strictly on standard, pre-anonymized public datasets (ABIDE, ADNI, OASIS-3) and complies fully with ethical guidelines regarding human data and research conduct.

\item {\bf Broader impacts}
    \item[] Question: Does the paper discuss both potential positive societal impacts and negative societal impacts of the work performed?
    \item[] Answer: \answerYes{}
    \item[] Justification: Broader impacts including potential negative societal risks (demographic 
    bias amplification, misuse of synthetic cohorts in clinical settings) and mitigation strategies 
    are discussed in the Broader Impacts in Appendix~\ref{appsec:broaderimpacts}.

\item {\bf Safeguards}
    \item[] Question: Does the paper describe safeguards that have been put in place for responsible release of data or models that have a high risk for misuse (e.g., pre-trained language models, image generators, or scraped datasets)?
    \item[] Answer: \answerNA{}
    \item[] Justification: The models generated in this research (synthetic functional connectomes) are highly domain-specific and do not pose the generalized high-risk for malicious misuse associated with large language models or deepfake image generators.

\item {\bf Licenses for existing assets}
    \item[] Question: Are the creators or original owners of assets (e.g., code, data, models), used in the paper, properly credited and are the license and terms of use explicitly mentioned and properly respected?
    \item[] Answer: \answerYes{}
    \item[] Justification: The creators of the utilized datasets (ABIDE, ADNI, OASIS-3) and software libraries (e.g., Nilearn) are properly credited with citations throughout the text.

\item {\bf New assets}
    \item[] Question: Are new assets introduced in the paper well documented and is the documentation provided alongside the assets?
    \item[] Answer: \answerNA{}
    \item[] Justification: The research utilizes existing datasets and does not introduce new primary data assets.

\item {\bf Crowdsourcing and research with human subjects}
    \item[] Question: For crowdsourcing experiments and research with human subjects, does the paper include the full text of instructions given to participants and screenshots, if applicable, as well as details about compensation (if any)? 
    \item[] Answer: \answerNA{}
    \item[] Justification: This research relies entirely on retrospective analysis of existing, de-identified databases and does not involve primary crowdsourcing or new human subject interaction.

\item {\bf Institutional review board (IRB) approvals or equivalent for research with human subjects}
    \item[] Question: Does the paper describe potential risks incurred by study participants, whether such risks were disclosed to the subjects, and whether Institutional Review Board (IRB) approvals (or an equivalent approval/review based on the requirements of your country or institution) were obtained?
    \item[] Answer: \answerNA{}
    \item[] Justification: Because the study strictly utilizes pre-collected, fully anonymized public neuroimaging datasets, new Institutional Review Board (IRB) approval is not required.

\item {\bf Declaration of LLM usage}
    \item[] Question: Does the paper describe the usage of LLMs if it is an important, original, or non-standard component of the core methods in this research? Note that if the LLM is used only for writing, editing, or formatting purposes and does \emph{not} impact the core methodology, scientific rigor, or originality of the research, declaration is not required.
    \item[] Answer: \answerNA{}
    \item[] Justification: Large Language Models were not utilized as a component of the core generative methodology or scientific evaluation in this research.

\end{enumerate}

%% file: neurips_2026.bib
@inproceedings{alaa2022faithful,
  title={How {F}aithful is your {S}ynthetic {D}ata? {S}ample-level {M}etrics for {E}valuating and {A}uditing {G}enerative {M}odels},
  author={Alaa, Ahmed and Van Breugel, Boris and Saveliev, Evgeny S and Van Der Schaar, Mihaela},
  booktitle={International conference on machine learning},
  pages={290--306},
  year={2022},
  organization={PMLR}
}

@article{weiner2017alzheimer,
  title={The {A}lzheimer's {D}isease {N}euroimaging {I}nitiative 3: {C}ontinued innovation for clinical trial improvement},
  author={Weiner, Michael W and Veitch, Dallas P and Aisen, Paul S and Beckett, Laurel A and Cairns, Nigel J and Green, Robert C and Harvey, Danielle and Jack Jr, Clifford R and Jagust, William and Morris, John C and others},
  journal={Alzheimer's \& Dementia},
  volume={13},
  number={5},
  pages={561--571},
  year={2017},
  publisher={Elsevier}
}

@article{di2014autism,
  title={The autism brain imaging data exchange: towards a large-scale evaluation of the intrinsic brain architecture in autism},
  author={Di Martino, Adriana and Yan, Chao-Gan and Li, Qingyang and Denio, Erin and Castellanos, Francisco X and Alaerts, Kaat and Anderson, Jeffrey S and Assaf, Michal and Bookheimer, Susan Y and Dapretto, Mirella and others},
  journal={Molecular psychiatry},
  volume={19},
  number={6},
  pages={659--667},
  year={2014},
  publisher={Nature Publishing Group}
}

@article{kantorovich1939mathematical,
  title={The mathematical method of production planning and organization},
  author={Kantorovich, Leonid V},
  journal={Management Science},
  volume={6},
  number={4},
  pages={363--422},
  year={1939}
}

@article{he2023manifold,
  title={Manifold {P}reserving {G}uided {D}iffusion},
  author={He, Yutong and Murata, Naoki and Lai, Chieh-Hsin and Takida, Yuhta and Uesaka, Toshimitsu and Kim, Dongjun and Liao, Wei-Hsiang and Mitsufuji, Yuki and Kolter, J Zico and Salakhutdinov, Ruslan and others},
  journal={arXiv preprint arXiv:2311.16424},
  year={2023}
}

@article{schwarz1978estimating,
  title={Estimating the dimension of a model},
  author={Schwarz, Gideon},
  journal={The annals of statistics},
  pages={461--464},
  year={1978},
  publisher={JSTOR}
}

@article{mclachlan2019finite,
  title={Finite {M}ixture {M}odels},
  author={McLachlan, Geoffrey J and Lee, Sharon X and Rathnayake, Suren I},
  journal={Annual review of statistics and its application},
  volume={6},
  number={1},
  pages={355--378},
  year={2019},
  publisher={Annual Reviews}
}

@article{archakov2021new,
  title={A new parametrization of correlation matrices},
  author={Archakov, Ilya and Hansen, Peter Reinhard},
  journal={Econometrica},
  volume={89},
  number={4},
  pages={1699--1715},
  year={2021},
  publisher={Wiley Online Library}
}

@article{pinkney2024short,
  title={A short note on a flexible Cholesky parameterization of correlation matrices},
  author={Pinkney, Sean},
  journal={arXiv preprint arXiv:2405.07286},
  year={2024}
}

@article{battaglia2018relational,
  title={Relational {I}nductive {B}iases, {D}eep {L}earning, and {G}raph {N}etworks},
  author={Battaglia, Peter W and Hamrick, Jessica B and Bapst, Victor and Sanchez-Gonzalez, Alvaro and Zambaldi, Vinicius and Malinowski, Mateusz and Tacchetti, Andrea and Raposo, David and Santoro, Adam and Faulkner, Ryan and others},
  journal={arXiv preprint arXiv:1806.01261},
  year={2018}
}

@article{bronstein2017geometric,
  title={Geometric {D}eep {L}earning: {G}oing beyond {E}uclidean data},
  author={Bronstein, Michael M and Bruna, Joan and LeCun, Yann and Szlam, Arthur and Vandergheynst, Pierre},
  journal={IEEE Signal Processing Magazine},
  volume={34},
  number={4},
  pages={18--42},
  year={2017},
  publisher={IEEE}
}

@article{rubinov2010complex,
  title={Complex network measures of brain connectivity: uses and interpretations},
  author={Rubinov, Mikail and Sporns, Olaf},
  journal={Neuroimage},
  volume={52},
  number={3},
  pages={1059--1069},
  year={2010},
  publisher={Elsevier}
}

@article{van2013network,
  title={Network hubs in the human brain},
  author={Van den Heuvel, Martijn P and Sporns, Olaf},
  journal={Trends in cognitive sciences},
  volume={17},
  number={12},
  pages={683--696},
  year={2013},
  publisher={Elsevier}
}

@inproceedings{zhao2021improved,
  title={Improved consistency regularization for gans},
  author={Zhao, Zhengli and Singh, Sameer and Lee, Honglak and Zhang, Zizhao and Odena, Augustus and Zhang, Han},
  booktitle={Proceedings of the AAAI conference on artificial intelligence},
  volume={35},
  pages={11033--11041},
  year={2021}
}

@article{sporns2005human,
  title={The {H}uman {C}onnectome: a structural description of the human brain},
  author={Sporns, Olaf and Tononi, Giulio and K{\"o}tter, Rolf},
  journal={PLoS computational biology},
  volume={1},
  number={4},
  pages={e42},
  year={2005},
  publisher={Public Library of Science San Francisco, USA}
}

@incollection{scott2011multivariate,
  title={Multivariate density estimation and visualization},
  author={Scott, David W},
  booktitle={Handbook of computational statistics: Concepts and methods},
  pages={549--569},
  year={2011},
  publisher={Springer}
}

@article{bronstein2021geometric,
  title={Geometric deep learning: Grids, groups, graphs, geodesics, and gauges},
  author={Michael M. Bronstein and
                  Joan Bruna and
                  Taco Cohen and
                  Petar Velickovic},
  journal={arXiv preprint arXiv:2104.13478},
  year={2021}
}

@article{he2025position,
  title={Position: {B}eyond {E}uclidean--{F}oundation {M}odels {S}hould {E}mbrace {N}on-{E}uclidean {G}eometries},
  author={He, Neil and Liu, Jiahong and Zhang, Buze and Bui, Ngoc and Maatouk, Ali and Yang, Menglin and King, Irwin and Weber, Melanie and Ying, Rex},
  journal={arXiv preprint arXiv:2504.08896},
  year={2025}
}

@article{huang2022riemannian,
  title={Riemannian {D}iffusion {M}odels},
  author={Huang, Chin-Wei and Aghajohari, Milad and Bose, Joey and Panangaden, Prakash and Courville, Aaron C},
  journal={Advances in Neural Information Processing Systems},
  volume={35},
  pages={2750--2761},
  year={2022}
}

@article{bertolini2025diffusion,
  title={Diffusion Generative Modeling on Lie Group Representations},
  author={Bertolini, Marco and Le, Tuan and Clevert, Djork-Arn{\'e}},
  journal={arXiv preprint arXiv:2502.02513},
  year={2025}
}

@article{chung2024cfg++,
  title={{CFG}++: Manifold-constrained classifier free guidance for diffusion models},
  author={Chung, Hyungjin and Kim, Jeongsol and Park, Geon Yeong and Nam, Hyelin and Ye, Jong Chul},
  journal={arXiv preprint arXiv:2406.08070},
  year={2024}
}

@article{gao2025toward,
  title={Toward a Unified Geometry Understanding: Riemannian Diffusion Framework for Graph Generation and Prediction},
  author={Gao, Yisen and Fu, Xingcheng and Sun, Qingyun and Li, Jianxin and Li, Xianxian},
  journal={arXiv preprint arXiv:2510.04522},
  year={2025}
}

@inproceedings{
mehta2023efficient,
title={Efficient Discrete Multi Marginal Optimal Transport Regularization},
author={Ronak Mehta and Jeffery Kline and Vishnu Suresh Lokhande and Glenn Fung and Vikas Singh},
booktitle={The Eleventh International Conference on Learning Representations },
year={2023},
url={https://openreview.net/forum?id=R98ZfMt-jE}
}

@inproceedings{nazarovs2021graph,
  title={Graph reparameterizations for enabling 1000+ Monte Carlo iterations in Bayesian deep neural networks},
  author={Nazarovs, Jurijs and Mehta, Ronak R and Lokhande, Vishnu Suresh and Singh, Vikas},
  booktitle={Uncertainty in Artificial Intelligence},
  pages={118--128},
  year={2021},
  organization={PMLR}
}

@inproceedings{chytas2024pooling,
  title={Pooling image datasets with multiple covariate shift and imbalance},
  author={Chytas, Sotirios Panagiotis and Lokhande, Vishnu and Singh, Vikas},
  booktitle={International Conference on Learning Representations},
  volume={2024},
  pages={57282--57304},
  year={2024}
}

@article{khan2026isokinetic,
  title={Isokinetic Flow Matching for Pathwise Straightening of Generative Flows},
  author={Khan, Tauhid},
  journal={arXiv preprint arXiv:2604.04491},
  year={2026}
}

@article{kan2022brain,
  title={Brain {N}etwork {T}ransformer},
  author={Kan, Xuan and Dai, Wei and Cui, Hejie and Zhang, Zilong and Guo, Ying and Yang, Carl},
  journal={Advances in Neural Information Processing Systems},
  volume={35},
  pages={25586--25599},
  year={2022}
}

@inproceedings{cho2024neurodegenerative,
  title={Neurodegenerative brain network classification via adaptive diffusion with temporal regularization},
  author={Cho, Hyuna and Sim, Jaeyoon and Wu, Guorong and Kim, Won Hwa},
  booktitle={Forty-first International Conference On Machine Learning},
  year={2024}
}

@inproceedings{ahn2025spectral,
  title={Spectral Graph Autoregressive Modeling for Conditional Brain Network Augmentation},
  author={Ahn, Hayoung and Lee, Seungjoo and Sim, Jaeyoon and Hwang, Yechan and Cho, Hyuna and Wu, Guorong and Kim, Won Hwa},
  booktitle={International Workshop on Reconstruction and Imaging Motion Estimation},
  pages={209--218},
  year={2025},
  organization={Springer}
}

@inproceedings{cho2025conditional,
  title={Conditional diffusion with ordinal regression: Longitudinal data generation for neurodegenerative disease studies},
  author={Cho, Hyuna and Wei, Ziquan and Lee, Seungjoo and Dan, Tingting and Wu, Guorong and Kim, Won Hwa},
  booktitle={international conference on learning representations},
  volume={2025},
  pages={33329--33350},
  year={2025}
}

@book{devroye2013probabilistic,
  title={A probabilistic theory of pattern recognition},
  author={Devroye, Luc and Gy{\"o}rfi, L{\'a}szl{\'o} and Lugosi, G{\'a}bor},
  volume={31},
  year={2013},
  publisher={Springer Science \& Business Media}
}

@article{cover1967nearest,
  title={Nearest neighbor pattern classification},
  author={Cover, Thomas and Hart, Peter},
  journal={IEEE transactions on information theory},
  volume={13},
  number={1},
  pages={21--27},
  year={1967},
  publisher={IEEE}
}

@article{lamontagne2019oasis,
  title={{OASIS}-3: longitudinal neuroimaging, clinical, and cognitive dataset for normal aging and Alzheimer disease},
  author={LaMontagne, Pamela J and Benzinger, Tammie LS and Morris, John C and Keefe, Sarah and Hornbeck, Russ and Xiong, Chengjie and Grant, Elizabeth and Hassenstab, Jason and Moulder, Krista and Vlassenko, Andrei G and others},
  journal={medrxiv},
  pages={12},
  year={2019},
  publisher={Cold Spring Harbor Laboratory Press}
}

@article{smith2013functional,
  title={Functional connectomics from resting-state fMRI},
  author={Smith, Stephen M and Vidaurre, Diego and Beckmann, Christian F and Glasser, Matthew F and Jenkinson, Mark and Miller, Karla L and Nichols, Thomas E and Robinson, Emma C and Salimi-Khorshidi, Gholamreza and Woolrich, Mark W and others},
  journal={Trends in cognitive sciences},
  volume={17},
  number={12},
  pages={666--682},
  year={2013},
  publisher={Elsevier}
}

@article{castellanos2013clinical,
  title={Clinical applications of the functional connectome},
  author={Castellanos, F Xavier and Di Martino, Adriana and Craddock, R Cameron and Mehta, Ashesh D and Milham, Michael P},
  journal={Neuroimage},
  volume={80},
  pages={527--540},
  year={2013},
  publisher={Elsevier}
}

@article{dadi2019benchmarking,
  title={Benchmarking functional connectome-based predictive models for resting-state fMRI},
  author={Dadi, Kamalaker and Rahim, Mehdi and Abraham, Alexandre and Chyzhyk, Darya and Milham, Michael and Thirion, Bertrand and Varoquaux, Ga{\"e}l and Alzheimer's Disease Neuroimaging Initiative and others},
  journal={NeuroImage},
  volume={192},
  pages={115--134},
  year={2019},
  publisher={Elsevier}
}

@inproceedings{
collas2025riemannian,
title={Riemannian Flow Matching for Brain Connectivity Matrices via Pullback Geometry},
author={Antoine Collas and Ce Ju and Nicolas Salvy and Bertrand Thirion},
booktitle={The Thirty-ninth Annual Conference on Neural Information Processing Systems},
year={2025},
url={https://openreview.net/forum?id=NY3LzmUXl7}
}

@article{ho2020denoising,
  title={Denoising {D}iffusion {P}robabilistic {M}odels},
  author={Ho, Jonathan and Jain, Ajay and Abbeel, Pieter},
  journal={Advances in neural information processing systems},
  volume={33},
  pages={6840--6851},
  year={2020}
}

@article{lipman2022flow,
  title={Flow {M}atching for {G}enerative {M}odeling},
  author={Lipman, Yaron and Chen, Ricky TQ and Ben-Hamu, Heli and Nickel, Maximilian and Le, Matt},
  journal={arXiv preprint arXiv:2210.02747},
  year={2022}
}

@article{de2022riemannian,
  title={Riemannian {S}core-based {G}enerative {M}odelling},
  author={De Bortoli, Valentin and Mathieu, Emile and Hutchinson, Michael and Thornton, James and Teh, Yee Whye and Doucet, Arnaud},
  journal={Advances in neural information processing systems},
  volume={35},
  pages={2406--2422},
  year={2022}
}

@article{mohammadi2024graph,
  title={Graph neural networks in brain connectivity studies: Methods, challenges, and future directions},
  author={Mohammadi, Hamed and Karwowski, Waldemar},
  journal={Brain Sciences},
  volume={15},
  number={1},
  pages={17},
  year={2024},
  publisher={MDPI}
}

@article{cui2022braingb,
  title={{BrainGB}: a benchmark for brain network analysis with graph neural networks},
  author={Cui, Hejie and Dai, Wei and Zhu, Yanqiao and Kan, Xuan and Gu, Antonio Aodong Chen and Lukemire, Joshua and Zhan, Liang and He, Lifang and Guo, Ying and Yang, Carl},
  journal={IEEE transactions on medical imaging},
  volume={42},
  number={2},
  pages={493--506},
  year={2022},
  publisher={IEEE}
}

@article{wen2024dimensional,
  title={Dimensional neuroimaging endophenotypes: neurobiological representations of disease heterogeneity through machine learning},
  author={Wen, Junhao and Antoniades, Mathilde and Yang, Zhijian and Hwang, Gyujoon and Skampardoni, Ioanna and Wang, Rongguang and Davatzikos, Christos},
  journal={Biological psychiatry},
  volume={96},
  number={7},
  pages={564--584},
  year={2024},
  publisher={Elsevier}
}

@article{arsigny2006log,
  title={Log-Euclidean metrics for fast and simple calculus on diffusion tensors},
  author={Arsigny, Vincent and Fillard, Pierre and Pennec, Xavier and Ayache, Nicholas},
  journal={Magnetic Resonance in Medicine: An Official Journal of the International Society for Magnetic Resonance in Medicine},
  volume={56},
  number={2},
  pages={411--421},
  year={2006},
  publisher={Wiley Online Library}
}

@article{wu2025optimal,
  title={Optimal stopping in latent diffusion models},
  author={Yu-Han Wu and Quentin Berthet and Gérard Biau and Claire Boyer and Romuald Elie and Pierre Marion},
  journal={arXiv preprint arXiv:2510.08409},
  year={2025}
}

@article{li2026kinetic,
  title={A Kinetic-Energy Perspective of Flow Matching},
  author={Li, Ziyun and Hu, Huancheng and Lim, Soon Hoe and Li, Xuyu and Gao, Fei and Diao, Enmao and Ding, Zezhen and Vazirgiannis, Michalis and Bostrom, Henrik},
  journal={arXiv preprint arXiv:2602.07928},
  year={2026}
}

@article{chen2022analog,
  title={Analog bits: Generating discrete data using diffusion models with self-conditioning},
  author={Chen, Ting and Zhang, Ruixiang and Hinton, Geoffrey},
  journal={arXiv preprint arXiv:2208.04202},
  year={2022}
}

@article{bessadok2022graph,
  title={Graph neural networks in network neuroscience},
  author={Bessadok, Alaa and Mahjoub, Mohamed Ali and Rekik, Islem},
  journal={IEEE Transactions on Pattern Analysis and Machine Intelligence},
  volume={45},
  number={5},
  pages={5833--5848},
  year={2022},
  publisher={IEEE}
}

@article{arsigny2007geometric,
  title={Geometric means in a novel vector space structure on symmetric positive-definite matrices},
  author={Arsigny, Vincent and Fillard, Pierre and Pennec, Xavier and Ayache, Nicholas},
  journal={SIAM journal on matrix analysis and applications},
  volume={29},
  number={1},
  pages={328--347},
  year={2007},
  publisher={SIAM}
}

@article{varol2017hydra,
  title={{HYDRA}: Revealing heterogeneity of imaging and genetic patterns through a multiple max-margin discriminative analysis framework},
  author={Varol, Erdem and Sotiras, Aristeidis and Davatzikos, Christos and Alzheimer's Disease Neuroimaging Initiative and others},
  journal={Neuroimage},
  volume={145},
  pages={346--364},
  year={2017},
  publisher={Elsevier}
}

@article{benny2021evaluation,
  title={Evaluation metrics for conditional image generation},
  author={Benny, Yaniv and Galanti, Tomer and Benaim, Sagie and Wolf, Lior},
  journal={International Journal of Computer Vision},
  volume={129},
  number={5},
  pages={1712--1731},
  year={2021},
  publisher={Springer}
}

@article{regenwetter2023beyond,
  title={Beyond statistical similarity: Rethinking metrics for deep generative models in engineering design},
  author={Regenwetter, Lyle and Srivastava, Akash and Gutfreund, Dan and Ahmed, Faez},
  journal={Computer-Aided Design},
  volume={165},
  pages={103609},
  year={2023},
  publisher={Elsevier}
}

@software{Nilearn,
    author = {Nilearn contributors},
    license = {BSD-4-Clause},
    title = {{nilearn}},
    url = {https://github.com/nilearn/nilearn},
    doi = {https://doi.org/10.5281/zenodo.8397156}
}

@article{buda2018systematic,
  title={A systematic study of the class imbalance problem in convolutional neural networks},
  author={Buda, Mateusz and Maki, Atsuto and Mazurowski, Maciej A},
  journal={Neural networks},
  volume={106},
  pages={249--259},
  year={2018},
  publisher={Elsevier}
}

@article{kang2019decoupling,
  title={Decoupling representation and classifier for long-tailed recognition},
  author={Kang, Bingyi and Xie, Saining and Rohrbach, Marcus and Yan, Zhicheng and Gordo, Albert and Feng, Jiashi and Kalantidis, Yannis},
  journal={arXiv preprint arXiv:1910.09217},
  year={2019}
}

@article{maron2018invariant,
  title={Invariant and equivariant graph networks},
  author={Maron, Haggai and Ben-Hamu, Heli and Shamir, Nadav and Lipman, Yaron},
  journal={arXiv preprint arXiv:1812.09902},
  year={2018}
}

@article{zaheer2017deep,
  title={Deep sets},
  author={Zaheer, Manzil and Kottur, Satwik and Ravanbakhsh, Siamak and Poczos, Barnabas and Salakhutdinov, Russ R and Smola, Alexander J},
  journal={Advances in neural information processing systems},
  volume={30},
  year={2017}
}

@inproceedings{byrd2019effect,
  title={What is the effect of importance weighting in deep learning?},
  author={Byrd, Jonathon and Lipton, Zachary},
  booktitle={International conference on machine learning},
  pages={872--881},
  year={2019},
  organization={PMLR}
}

@article{japkowicz2002class,
  title={The class imbalance problem: A systematic study},
  author={Japkowicz, Nathalie and Stephen, Shaju},
  journal={Intelligent data analysis},
  volume={6},
  number={5},
  pages={429--449},
  year={2002},
  publisher={SAGE Publications Sage UK: London, England}
}

@inproceedings{nichol2021improved,
  title={Improved denoising diffusion probabilistic models},
  author={Nichol, Alexander Quinn and Dhariwal, Prafulla},
  booktitle={International conference on machine learning},
  pages={8162--8171},
  year={2021},
  organization={PMLR}
}

@inproceedings{peebles2023scalable,
  title={Scalable diffusion models with transformers},
  author={Peebles, William and Xie, Saining},
  booktitle={Proceedings of the IEEE/CVF international conference on computer vision},
  pages={4195--4205},
  year={2023}
}

@inproceedings{perez2018film,
  title={{FiLM}: Visual reasoning with a general conditioning layer},
  author={Perez, Ethan and Strub, Florian and De Vries, Harm and Dumoulin, Vincent and Courville, Aaron},
  booktitle={Proceedings of the AAAI conference on artificial intelligence},
  volume={32},
  year={2018}
}

@inproceedings{varoquaux2011multi,
  title={Multi-subject dictionary learning to segment an atlas of brain spontaneous activity},
  author={Varoquaux, Ga{\"e}l and Gramfort, Alexandre and Pedregosa, Fabian and Michel, Vincent and Thirion, Bertrand},
  booktitle={Biennial International Conference on information processing in medical imaging},
  pages={562--573},
  year={2011},
  organization={Springer}
}

@article{dhariwal2021diffusion,
  title={Diffusion models beat gans on image synthesis},
  author={Dhariwal, Prafulla and Nichol, Alexander},
  journal={Advances in neural information processing systems},
  volume={34},
  pages={8780--8794},
  year={2021}
}

@article{wolfers2018mapping,
  title={Mapping the heterogeneous phenotype of schizophrenia and bipolar disorder using normative models},
  author={Wolfers, Thomas and Doan, Nhat Trung and Kaufmann, Tobias and Aln{\ae}s, Dag and Moberget, Torgeir and Agartz, Ingrid and Buitelaar, Jan K and Ueland, Torill and Melle, Ingrid and Franke, Barbara and others},
  journal={JAMA psychiatry},
  volume={75},
  number={11},
  pages={1146--1155},
  year={2018}
}

@article{vignac2022digress,
  title={{DiGress}: Discrete denoising diffusion for graph generation},
  author={Vignac, Clement and Krawczuk, Igor and Siraudin, Antoine and Wang, Bohan and Cevher, Volkan and Frossard, Pascal},
  journal={arXiv preprint arXiv:2209.14734},
  year={2022}
}

@inproceedings{jo2022score,
  title={Score-based generative modeling of graphs via the system of stochastic differential equations},
  author={Jo, Jaehyeong and Lee, Seul and Hwang, Sung Ju},
  booktitle={International conference on machine learning},
  pages={10362--10383},
  year={2022},
  organization={PMLR}
}

@article{song2020denoising,
  title={Denoising diffusion implicit models},
  author={Song, Jiaming and Meng, Chenlin and Ermon, Stefano},
  journal={arXiv preprint arXiv:2010.02502},
  year={2020}
}

@article{ho2022classifier,
  title={Classifier-free diffusion guidance},
  author={Ho, Jonathan and Salimans, Tim},
  journal={arXiv preprint arXiv:2207.12598},
  year={2022}
}

@article{pennec2006riemannian,
  title={A Riemannian framework for tensor computing},
  author={Pennec, Xavier and Fillard, Pierre and Ayache, Nicholas},
  journal={International Journal of computer vision},
  volume={66},
  number={1},
  pages={41--66},
  year={2006},
  publisher={Springer}
}

@article{lee2025local,
  title={Local manifold approximation and projection for manifold-aware diffusion planning},
  author={Lee, Kyowoon and Choi, Jaesik},
  journal={arXiv preprint arXiv:2506.00867},
  year={2025}
}

@article{bartlett2017spectrally,
  title={Spectrally-normalized margin bounds for neural networks},
  author={Bartlett, Peter L and Foster, Dylan J and Telgarsky, Matus J},
  journal={Advances in neural information processing systems},
  volume={30},
  year={2017}
}

@article{miyato2018spectral,
  title={Spectral normalization for generative adversarial networks},
  author={Miyato, Takeru and Kataoka, Toshiki and Koyama, Masanori and Yoshida, Yuichi},
  journal={arXiv preprint arXiv:1802.05957},
  year={2018}
}

@article{hinton2015distilling,
  title={Distilling the knowledge in a neural network},
  author={Hinton, Geoffrey and Vinyals, Oriol and Dean, Jeff},
  journal={arXiv preprint arXiv:1503.02531},
  year={2015}
}

@article{efron2011tweedie,
  title={Tweedie’s formula and selection bias},
  author={Efron, Bradley},
  journal={Journal of the American Statistical Association},
  volume={106},
  number={496},
  pages={1602--1614},
  year={2011},
  publisher={Taylor \& Francis}
}

@article{miyasawa1961empirical,
  title={An empirical Bayes estimator of the mean of a normal population},
  author={Miyasawa, Koichi and others},
  journal={Bull. Inst. Internat. Statist},
  volume={38},
  number={181-188},
  pages={1--2},
  year={1961}
}

@article{kynkaanniemi2024applying,
  title={Applying guidance in a limited interval improves sample and distribution quality in diffusion models},
  author={Kynk{\"a}{\"a}nniemi, Tuomas and Aittala, Miika and Karras, Tero and Laine, Samuli and Aila, Timo and Lehtinen, Jaakko},
  journal={Advances in Neural Information Processing Systems},
  volume={37},
  pages={122458--122483},
  year={2024}
}

@article{shazeer2020glu,
  title={{GLU} variants improve transformer},
  author={Shazeer, Noam},
  journal={arXiv preprint arXiv:2002.05202},
  year={2020}
}

@article{fraley2002model,
  title={Model-based clustering, discriminant analysis, and density estimation},
  author={Fraley, Chris and Raftery, Adrian E},
  journal={Journal of the American statistical Association},
  volume={97},
  number={458},
  pages={611--631},
  year={2002},
  publisher={Taylor \& Francis}
}

@article{dempster1977maximum,
  title={Maximum likelihood from incomplete data via the EM algorithm},
  author={Dempster, Arthur P and Laird, Nan M and Rubin, Donald B},
  journal={Journal of the royal statistical society: series B (methodological)},
  volume={39},
  number={1},
  pages={1--22},
  year={1977},
  publisher={Wiley Online Library}
}

@inproceedings{henry2020query,
  title={Query-key normalization for transformers},
  author={Henry, Alex and Dachapally, Prudhvi Raj and Pawar, Shubham Shantaram and Chen, Yuxuan},
  booktitle={Findings of the Association for Computational Linguistics: EMNLP 2020},
  pages={4246--4253},
  year={2020}
}

@inproceedings{xiong2020layer,
  title={On layer normalization in the transformer architecture},
  author={Xiong, Ruibin and Yang, Yunchang and He, Di and Zheng, Kai and Zheng, Shuxin and Xing, Chen and Zhang, Huishuai and Lan, Yanyan and Wang, Liwei and Liu, Tieyan},
  booktitle={International conference on machine learning},
  pages={10524--10533},
  year={2020},
  organization={PMLR}
}

@article{vaswani2017attention,
  title={Attention is all you need},
  author={Vaswani, Ashish and Shazeer, Noam and Parmar, Niki and Uszkoreit, Jakob and Jones, Llion and Gomez, Aidan N and Kaiser, {\L}ukasz and Polosukhin, Illia},
  journal={Advances in neural information processing systems},
  volume={30},
  year={2017}
}

@article{goodfellow2016nips,
  title={NIPS 2016 tutorial: Generative adversarial networks},
  author={Goodfellow, Ian},
  journal={arXiv preprint arXiv:1701.00160},
  year={2016}
}

@article{navidi2025adaptive,
  title={Adaptive resampling for improved machine learning in imbalanced single-cell datasets},
  author={Navidi, Zeinab and Thoutam, Akshaya and Hughes, Madeline and Raghavan, Srivatsan and Winter, Peter S and Crawford, Lorin and Amini, Ava P},
  journal={bioRxiv},
  pages={11},
  year={2025},
  publisher={Cold Spring Harbor Laboratory}
}

@article{park2025temporal,
  title={Temporal Alignment Guidance: On-Manifold Sampling in Diffusion Models},
  author={Park, Youngrok and Jung, Hojung and Bae, Sangmin and Yun, Se-Young},
  journal={arXiv preprint arXiv:2510.11057},
  year={2025}
}

@article{sadat2023cads,
  title={{CADS}: Unleashing the diversity of diffusion models through condition-annealed sampling},
  author={Sadat, Seyedmorteza and Buhmann, Jakob and Bradley, Derek and Hilliges, Otmar and Weber, Romann M},
  journal={arXiv preprint arXiv:2310.17347},
  year={2023}
}

@article{su2025diffusion,
  title={Diffusion Models as Dataset Distillation Priors},
  author={Su, Duo and Wu, Huyu and Chen, Huanran and Shi, Yiming and Wang, Yuzhu and Ye, Xi and Zhu, Jun},
  journal={arXiv preprint arXiv:2510.17421},
  year={2025}
}

@article{roy2026manifoldgd,
  title={{ManifoldGD}: Training-Free Hierarchical Manifold Guidance for Diffusion-Based Dataset Distillation},
  author={Roy, Ayush and Lee, Wei-Yang Alex and Chakraborty, Rudrasis and Lokhande, Vishnu Suresh},
  journal={arXiv preprint arXiv:2602.23295},
  year={2026}
}
